\documentclass[12pt]{article}
\usepackage[utf8]{inputenc}
\usepackage{amsmath, bm}
\usepackage{amssymb}
\usepackage{graphicx}
\usepackage{amsfonts}
\usepackage{amsthm}
\usepackage{mathrsfs}
\usepackage{hyperref}
\usepackage{multirow,multicol}
\usepackage[table]{xcolor}
\usepackage[authoryear]{natbib}
\usepackage{algorithm}
\usepackage{algorithmic}
\usepackage{url}
\usepackage{comment}
\usepackage{setspace}
\usepackage{authblk}
\usepackage{subcaption}
\usepackage{dsfont}

\usepackage[a4paper,top=1in,bottom=1in,left=1in,right=1in,marginparwidth=1in]{geometry}

\def\bbf{\boldsymbol{b}}
\def\sbf{\boldsymbol{s}}
\def\wbf{\boldsymbol{w}}
\def\xbf{\boldsymbol{x}}
\def\zbf{\boldsymbol{z}}
\def\0bf{\boldsymbol{0}}
\def\alphabf{\boldsymbol{\alpha}}
\def\etabf{\boldsymbol{\eta}}
\def\mubf{\boldsymbol{\mu}}
\def\omegabf{\boldsymbol{\omega}}
\def\Upsilonbf{\boldsymbol{\Upsilon}}
\def\xibf{\boldsymbol{\xi}}
\def\zetabf{\boldsymbol{\zeta}}
\def\T{{\mathrm{\scriptscriptstyle T}}}
\newcommand{\tabincell}[2]{\begin{tabular}{@{}#1@{}}#2\end{tabular}}

\theoremstyle{plain}
\newtheorem{theorem}{Theorem}[section]
\newtheorem{corollary}{Corollary}[section]
\newtheorem{assumption}{Assumption}[]
\newtheorem{lemma}{Lemma}[]
\newtheorem{proposition}{Proposition}[]

\title{Implicit Generative Prior for Bayesian Neural Networks}

\author[1]{Yijia Liu}
\author[2]{Xiao Wang}
\affil[1,2]{Department of Statistics, Purdue University}

\date{}

\begin{document}

\maketitle
\begin{abstract}
    Predictive uncertainty quantification is crucial for reliable decision-making in various applied domains. Bayesian neural networks offer a powerful framework for this task. However, defining meaningful priors and ensuring computational efficiency remain significant challenges, especially for complex real-world applications. This paper addresses these challenges by proposing a novel neural adaptive empirical Bayes (NA-EB) framework. NA-EB leverages a class of implicit generative priors derived from low-dimensional distributions. This allows for efficient handling of complex data structures and effective capture of underlying relationships in real-world datasets. The proposed NA-EB framework combines variational inference with a gradient ascent algorithm. This enables simultaneous hyperparameter selection and approximation of the posterior distribution, leading to improved computational efficiency. We establish the theoretical foundation of the framework through posterior and classification consistency.
    We demonstrate the practical applications of our framework through extensive evaluations on a variety of tasks, including the two-spiral problem, regression, 10 UCI datasets, and image classification tasks on both MNIST and CIFAR-10 datasets.
    The results of our experiments highlight the superiority of our proposed framework over existing methods, such as sparse variational Bayesian and generative models, in terms of prediction accuracy and uncertainty quantification. 
    
    {\it Key words}: Deep neural networks; empirical Bayes; latent variable model; stochastic gradient method; variational inference
\end{abstract}

\newpage

\section{Introduction}

Despite the remarkable accomplishments of deep neural networks (DNNs) in the field of artificial intelligence, they encounter numerous challenges. 
%
%
When utilized in the context of supervised learning, DNN models frequently struggle to accurately gauge uncertainty within training data and only provide a point estimate regarding class or prediction.
The consequences of this limitation are profound, particularly when these models are entrusted with life-or-death decisions. In medical domains, for instance, experts may find it challenging to determine whether they should rely on automated diagnostic systems, and passengers in self-driving vehicles may not receive alerts to take control when the vehicle encounters situations it does not comprehend.

To illustrate the importance of predictive uncertainty, we present two real-world classification examples. First, in Figure \ref{fig:cifar10}, we compare the predicted probabilities of the ResNet-18 \citep{he2016deep} classifier with our 95\% Bayesian credible intervals on four test images from the CIFAR-10 dataset (\url{https://www.kaggle.com/c/cifar-10/}). This dataset comprises 60,000 32$\times$32 color images across 10 classes. We observe that the standard DNN method 
%
%
often assigns high probabilities to incorrect classes. In contrast, our proposed approach, \textit{Neural Adaptive Empirical Bayes} (NA-EB), offers prediction intervals for the likelihood of each class label. 
The widths of the prediction intervals reflect how sure or unsure NA-EB is about the correctness of its predictions while the point estimate of the likelihood generated by the standard DNN method does not convey uncertainty information. In the case of the deer image, our method produces similarly wide and overlapping prediction intervals for the two potential classes, respectively, implying the uncertainty in its predictions. In contrast, the standard DNN method assigns a high probability to the wrong class without accounting for uncertainty.
Furthermore, in the case of the frog image, our method yields a relatively narrow prediction interval for the correct class, whereas the DNN method assigns a high probability estimate to the wrong class. This indicates that NA-EB not only quantifies predictive uncertainty but also enhances classification accuracy by implicitly specifying a correct prior for classifier weights.
Furthermore, we present comparisons in a scenario of weaker data signals in Figure \ref{fig:noisy_mnist}. Specifically, we employ a fully connected feedforward neural network as the base classifier on an artificially noisy dataset \citep{basu2017learning} created by introducing motion blur into the MNIST dataset (\url{http://yann.lecun.com/exdb/mnist/}). NA-EB generates narrow prediction intervals with high probability values for certain digits but wide overlapping intervals for others, indicating uncertainty in the model predictions. In contrast,  for digits such as "two" and "four", the DNN method assigns high probabilities to incorrect classes without providing any information about the uncertainty regarding its belief. 
In summary, NA-EB offers a robust classifier capable of expressing its uncertainty through a full posterior distribution rather than a single point estimate. This suggests potential applications of NA-EB in the fields of medical diagnostics, such as the automated classification of diabetic retinopathy from retinal images. Uncertainty estimation is particularly critical in the medical domain, ensuring confident model predictions for screening automation while flagging uncertain cases for manual review by a medical expert. Bayesian deep learning has already demonstrated its significance in medical diagnostics \citep{worrall2016automated,leibig2017leveraging,kamnitsas2017efficient,ching2018opportunities}, underscoring the potential relevance of NA-EB in such applications due to its enhanced performance in uncertainty quantification and prediction accuracy.

\begin{figure}[t!]
    \centering
\includegraphics[scale=0.40]{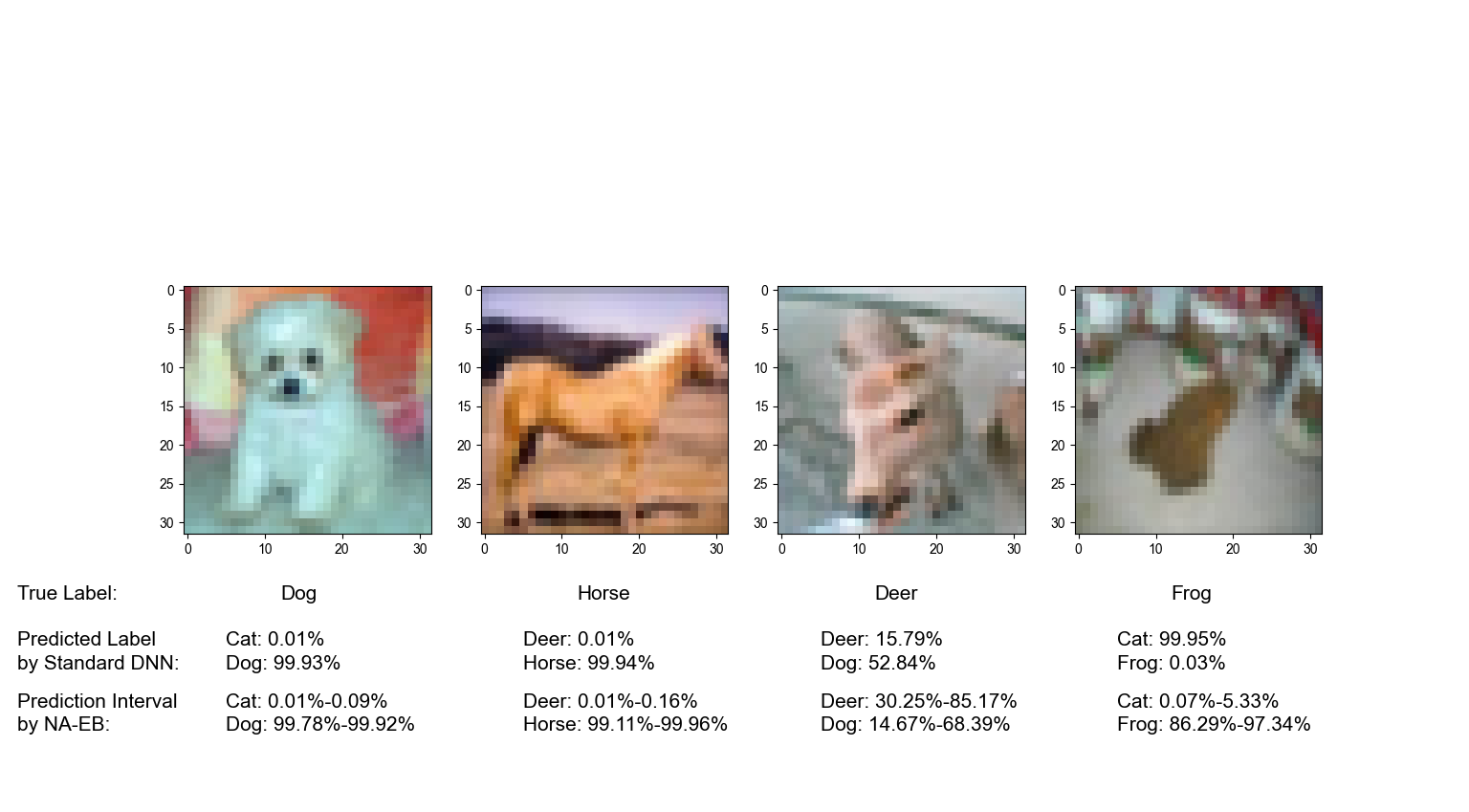}
    \caption{A comparison of classification results between the standard DNN method and NA-EB on four test images from the CIFAR-10 dataset.}
    \label{fig:cifar10}
\end{figure}

\begin{figure}[t!]
    \centering
    \includegraphics[scale=0.40]{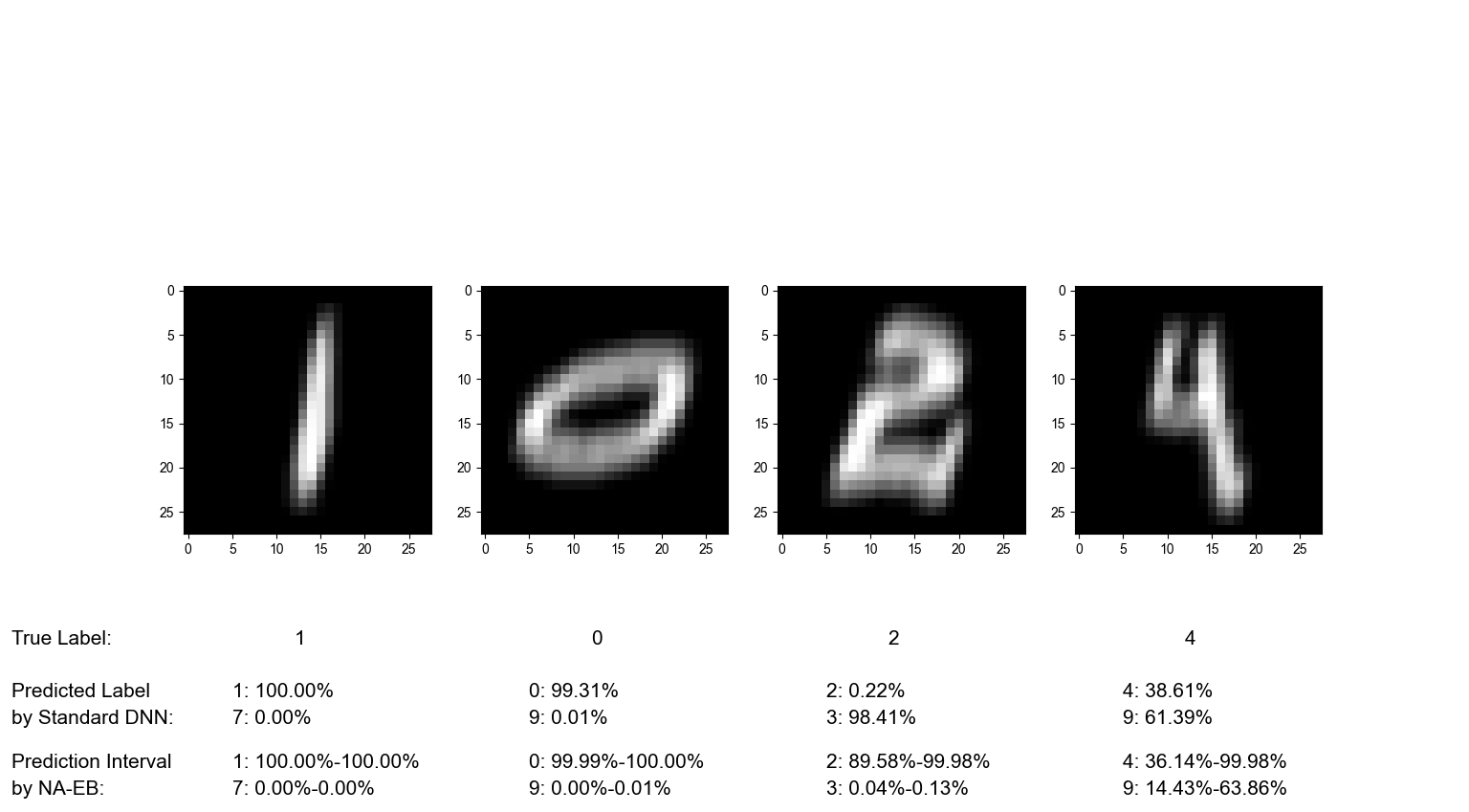}
    \caption{A comparison of classification results between the standard DNN method and NA-EB on four test images from the noisy MNIST with motion blur dataset.}
    \label{fig:noisy_mnist}
\end{figure}

Under the Bayesian framework, given the prior of the weights of the classifier or the regression model depending on the specific supervised learning task, the $95\%$ Bayesian credible interval can be easily constructed based on the shortest interval that contains $95\%$ of the predictive probability. However, when prior knowledge of the weights is not available, choosing a priori is a challenging task. Furthermore, when the weights are believed to reside in a low-dimensional manifold in higher-dimensional space, specifying the prior distribution of the weights is not straightforward since it may not have a tractable formula. 
%
%
For example, if we model the MNIST data, which contain 60,000 $28\times 28$ handwritten digit images, using a simple 2-hidden layer neural network with 784 input nodes, 800 nodes in each of two hidden layers and 10 nodes in the output layers, the model in total contains about 1.3 million parameters. It is not trivial to specify a prior distribution of 1.3 million dimensions. It is also reasonable to assume that these 1.3 million parameters reside in a low-dimensional manifold. 


Bayesian DNNs (BNNs) have been extensively studied by \citep{neal1992bayesian, mackay1995probable, neal96, bishop1997bayesian, lampinen2001bayesian, bernardo2009bayesian}. More recent developments with advanced algorithms can be found in \citep{sun2017learning, mullachery2018bayesian, hubin2018deep, javid2020compromise, wilson2020bayesian, izmailov2021bayesian}. BNNs promise improved predictions and yield richer representations from cheap model averaging. 
%
%
The most widely used priors for BNNs are isotropic Gaussian \citep{neal96, hern15, Louizos17, Dusenberry20, immer2021scalable}.  Gaussian priors have recently been shown to cause a cold posterior effect in BNNs. That is, the tempered posterior performs better than the true Bayesian posterior, suggesting prior misspecification \citep{wenzel2020good}. Another well-studied prior is placing sparsity-induced priors on network weights \citep{blundell2015weight, molchanov2017variational, ghosh2019model, bai2020efficient}, which can be used to aid compression of network weights. But sparsity-induced priors do not benefit prediction. In addition, \citet{quinonero2005evaluating} demonstrated that many priors of convenience can lead to unreasonable predictive uncertainties. Although there are many methods available for Bayesian computing, such as MCMC \citep{neal96, graves2011practical}, Langevin dynamics \citep{welling11}, and Hamiltonian methods \citep{spring16}, computing the posterior distribution is generally intractable, making it difficult to make inferences about the predictive distribution.

In this article, we propose a new class prior distributions called {\it implicit generative prior} to facilitate Bayesian modeling and inference. This idea is motivated by deep generative models in the machine learning literature \citep{goodfellow2014generative, kingma2013auto}. {\it Prescribed explicit priors} are those that provide an explicit distribution of the classifier's or the regression model's weights. Instead, the implicit generative prior represents the prior distribution through a function transformation of a known distribution to define a stochastic procedure that directly generates the parameter. We use a low-dimensional latent variable with a known prior distribution and transform it using a deterministic function with the hyperparameters. The deterministic transformation function is specified by a DNN, which takes the latent variable as the input and produces the weights as the output. This formulation is commonly seen in deep learning models such as the generative adversarial network (GAN, \citealp{goodfellow2014generative}) and the variational autoencoder (VAE, \citealp{kingma2013auto}).

In the context of Bayesian inference, we encounter two immediate challenges. First, the hyperparameter, the weights of the deterministic transformation function specified by a DNN, is high dimensional, making its selection a difficult task. Second, Bayesian computation becomes
computationally intractable in many cases, particularly when dealing with DNNs of any practical size. To address these challenges, we propose a variational approach called Neural
Adaptive Empirical Bayes (NA-EB). The key idea behind NA-EB is to leverage a variational posterior distribution with unknown hyperparameters, minimizing the Kullback--Leibler (KL) divergence from the true posterior distribution while simultaneously estimating the hyperparameters and approximating the posterior distribution. The notion of estimating the hyperparameters in the prior distribution from the data stems from empirical Bayes, a concept introduced by Herbert Robbins in the 1950s \citep{robbins1992empirical}. Empirical Bayes methods have since been extensively studied and applied in various domains \citep{efron1973stein, efron2001empirical, carlin2008bayesian, yves11, efron2012large}. Instead, minimizing the KL divergence between the variational distribution and the true posterior distribution is derived from variational inference \citep{hinton1993keeping, blei2007correlated, blundell2015weight, zhang2018noisy}. 
For training, we employ the gradient ascent algorithm along with Monte Carlo estimates to estimate both the variational posterior distribution and the unknown hyperparameters in the prior. This combination enables the NA-EB framework to be well suited for complex models and large-scale datasets. NA-EB shares some similarity with \cite{atanov2018deep} that we both use a variational approximation for the true posterior distribution of the weights and propose an implicit generative prior for the weights. On the other hand, NA-EB is distinguished from \cite{atanov2018deep} in terms of the objective function, the training procedure for the parameters of the implicit prior, the assumption about the parametric form of the implicit prior, and the establishment of theoretical guarantees.


Our contributions can be summarized as follows. 
\begin{itemize}
    \item We propose a novel approach to defining the prior distribution through a DNN transformation of a known low-dimensional distribution. This method provides a highly flexible prior that can capture complex high-dimensional distributions and incorporate intrinsic low-dimensional structure with ease. This approach represents a significant advancement in Bayesian modeling and inference, as it addresses the challenge of defining meaningful priors for neural networks, which has been a bottleneck in practical applications.
    \item Theoretically, we perform an asymptotic analysis in terms of posterior consistency \citep{bhattacharya2020variational, Bhattacharya2021StatisticalFO}, which quantifies the quality of the resulting posterior as data are collected indefinitely. In contrast to this previous work, we establish the uniform posterior consistency for a class of nonlinear transformations when defining the prior. Furthermore, we establish the classification accuracy of variational Bayes DNNs. Therefore, our framework is guaranteed to provide a numerically stable and theoretically consistent solution. 
    \item Empirically, we evaluate the finite sample performance through simulated data analysis and real data applications, including the classical two-spiral problem, synthetic regression data, 10 UCI datasets, MNIST dataset, noisy MNIST dataset, and CIFAR-10 dataset. We extensively compare our method with other methods including SGD \citep{rumelhart1986learning}, variational Bayesian \citep{blundell2015weight}, probabilistic back-propagation \citep{hernandez2015probabilistic}, dropout \citep{gal2016dropout}, ensembles Bayesian \citep{lakshminarayanan2017simple}, sparse variational Bayesian \citep{bai2020efficient}, variational Bayesian with collapsed bound \citep{tomczak2021collapsed}, conditional generative models \citep{zhou2022deep}, and diffusion models \citep{han2022card}.  As a result, the proposed method outperforms these existing methods on predictive accuracy and uncertainty quantification in both regression and classification tasks as shown in Section \ref{sec:numerical}. The source code implementations are available in the Appendix.
\end{itemize}

This paper is organized as follows. In Section \ref{sec:framework}, we give an overview of the proposed framework. In Section \ref{sec:algorithm}, we present the computational algorithm for NA-EB. In Section \ref{sec:theory}, we establish that our algorithm is theoretically guaranteed in terms of variational posterior consistency and classification accuracy. In Section \ref{sec:numerical}, we illustrate the performance of our model through simulation studies and real-life data analysis. We conclude our paper in Section \ref{sec:conclusion} with discussions. Further details about our algorithm, instructions for utilizing the code files, and the assumptions and the outline of the proofs of theorems and corollaries are given in the Appendix.

\section{Implicit Generative Prior}\label{sec:framework}

Let $\mathcal{D}_n = \{(\xbf_i, y_i)\}_{i=1}^n$ represent the training data with sample size $n$, where $y_i \in \mathcal{Y}$ is the response of interest and $\xbf_i\in \mathcal{X} \subseteq \mathbb{R}^p$ is the covariate. As in the classical setting of supervised learning, $(\xbf_i, y_i)$, for $i=1, \ldots, n$, are assumed to be i.i.d. from an unknown joint distribution $P_{\xbf, y}$ on $\mathcal{X}\times\mathcal{Y}$. Consider a probabilistic model $p(y\mid\xbf; \wbf)$, where $\wbf \in{\mathscr W}\subseteq \mathbb R^D$ denotes the vector of unknown parameter. For example, for regression, $p(y\mid\xbf; \wbf)$ can be a Gaussian distribution with an unknown mean that is modeled by a DNN with weights $\wbf$. For classification, $p(y\mid\xbf; \wbf)$ is the categorical distribution in which the success probabilities are modeled by an unknown multivariate function with parameter $\wbf$.
Let $L(\mathcal{D}_n; \wbf)$ be the joint conditional distribution of $y_i$ given $\xbf_i$, where $\wbf$ is the unknown parameter.  Bayesian inference will introduce a prior $\pi(\wbf)$ on $\wbf$, and compute the posterior distribution of the weights given the training data such as
$p(\wbf \mid \mathcal{D}_n)\propto \pi(\wbf)L(\mathcal{D}_n; {\wbf})$. The predictive distribution of a future $y^*$ given a test data $\xbf^*$ can be obtained by 
\begin{equation}\label{eq:bay_pred}
    p(y^*\mid\xbf^*, \mathcal{D}_n) = \int p(y^*\mid\xbf^*; \wbf)p(\wbf\mid\mathcal{D}_n) d\wbf.
\end{equation}
This is equivalent to an ensemble method, which uses an infinite number of models for prediction where the parameters of each model are sampled from the posterior distribution. 

We specify $\pi(\wbf)$ through a highly
flexible \emph{implicit} model defined by a two-step procedure, where firstly a latent variable $\zbf\in\mathscr{Z}\subseteq\mathbb{R}^{r}$ with $r\le  D$
is sampled from a fixed distribution $\pi_0(\zbf)$, and then $\zbf$ is mapped to $\wbf=G_{\etabf}(\zbf)$ via a \emph{deterministic} transformation $G_{\etabf}: \mathbb R^r \rightarrow \mathbb R^D$ with hyperparameter $\etabf\in\mathbb R^{n_\eta}$. 
This defines a class of priors using the push-forward measure,
\begin{equation*}
    \Pi = \left\{ G_{\etabf}\#\pi_0: G_{\etabf} \in \mathcal{G}  \right\},
\end{equation*}
where $\mathcal{G}$ is the function space for the transformation function.
When $G_{\etabf}$ is invertible and $r=D$, we recover the familiar rule of transformation of probability distributions. The prior distribution $\pi(\wbf)$ for this situation has an explicit formula by the change-of-variable technique.
We are interested in developing more general and flexible cases where $G_{\etabf}$ is a nonlinear function with $r\le D$. Under this circumstance, the explicit density of $\wbf$ is intractable, since the set $\{G_{\etabf}(\zbf) \le \wbf\}$ cannot be determined.  The advantages of the above formulation are the following: (a) $G_{\etabf}(\zbf)$ with $\zbf\sim \pi_0$ can represent a wide range of distributions. In fact, for any continuous random vector $\wbf$ in a low-dimensional manifold, there always exists a $G_{\etabf}$ such that $G_{\etabf}(\zbf)$ has the same distribution as $\wbf$ \citep{chen2022inferential}. (b) For many problems, it is reasonable to believe that the high-dimensional $\wbf$ lies in a low-dimensional manifold and the dimension $r$ can be much smaller than $D$. (c) The mapping $G_{\etabf}$ is unconstrained, considerably simplifying the functional optimization problem. 

Without loss of generality, we choose a normal distribution on each entry for $\zbf=(z_1, \ldots, z_r)^{\T}$ of the form
\begin{equation}\label{eq:prior}
    \pi_0(\zbf)=\prod_{j=1}^{r}\frac{1}{\sqrt{2\pi\zeta_j^2}}\exp\left\{-\frac{1}{2\zeta_j^2}(z_j-\mu_j)^2\right\},
\end{equation}
where $\zbf=(z_1,\ldots,z_{r})^{\T}$ and each $z_j$ requires a separate mean $\mu_j$ and variance $\zeta_j$. We will provide conditions on $\mubf = (\mu_1,\ldots,\mu_{r})^{\T}$ and $\zetabf=(\zeta_1,\ldots,\zeta_{r})^{\T}$ in Section \ref{sec:theory} to ensure the posterior consistency of both
Bayesian and variational approaches.

If $\etabf$ is known,  the posterior distribution of $\zbf$ given $\mathcal{D}_n$ is
\begin{equation}\label{eq:posteriorz}
    p_{\etabf}(\zbf\mid\mathcal{D}_n)=\frac{L(\mathcal{D}_n; {G_{\etabf}(\zbf)}){\pi_0}(\zbf)}{\int  L(\mathcal{D}_n; {G_{\etabf}(\zbf)}){\pi_0}(\zbf)d\zbf}.
\end{equation} 
Selecting the hyperparameter $\etabf$, in particular, the high-dimensional hyperparameter, is very difficult. Instead, the empirical Bayes method will estimate $\etabf$ from the data based on the marginal likelihood, and obtain
\begin{equation}\label{eq:marg_opt}
\hat\etabf = \arg\max_{\etabf} L_{\tt marg}({\etabf}; \mathcal{D}_n),
\end{equation}
where the marginal likelihood is given by
\begin{equation}\label{eq:marg_like}
L_{\tt marg}({\etabf}; \mathcal{D}_n) = \int_{\mathscr{W}} L(\mathcal{D}_n; {\wbf}) \pi(\wbf) d\wbf = \int_{\mathscr{Z}}  L(\mathcal{D}_n; {G_{\etabf}(\zbf)})\pi_0(\zbf) d\zbf.
\end{equation}
However, the main difficulty of the optimization problem \eqref{eq:marg_opt} is evaluating \eqref{eq:marg_like} and its gradient with respect to $\etabf$. The exact evaluation of $L_{\tt marg}({\etabf}; \mathcal{D}_n)$ is intractable since, in general, the integral is high-dimensional and does not have a closed form. In the next section, we introduce a novel algorithm based on the variational method to efficiently estimate ${\etabf}$ and approximate the posterior distribution.

\section{The Algorithm}\label{sec:algorithm}

The conventional MCMC implementation suffers from high computational cost, which limits its use for large scale problems. To avoid computational issues, we adopt the variational approach to estimate  $\etabf$ and derive the posterior distribution. The key difference between the current setting and the regular variational inference is that there involves
an additional unknown hyperparameter $\etabf$ in the prior. Variational inference starts from a variational family, which is used to approximate the true posterior distribution $p_{\etabf}(\zbf\mid\mathcal{D}_n)$ in \eqref{eq:posteriorz}. Given several options, we work with a simple and computationally tractable variational family, which is a mean field Gaussian
variational family of the form
\begin{equation*}
    \mathcal{Q}=\left\{q_{\alphabf}(\zbf): q_{\alphabf}(\zbf) = \prod_{j=1}^{r}\frac{1}{\sqrt{2\pi \varrho_j^2}}\exp\left\{-\frac{1}{2 \varrho_j^2}(z_j-m_j)^2\right\}\right\},
\end{equation*}
where $\alphabf = (m_1,\ldots,m_r,\varrho_1,\ldots,\varrho_{r})^{\T}\in\mathbb{R}^{2r}$  represents all unknown parameters in $q_{\alphabf}$.

Instead of maximizing the marginal log-likelihood $\log L_{\tt marg}(\etabf; \mathcal{D}_n)$ in \eqref{eq:marg_like}, we maximize the evidence lower bound (ELBO), defined by
\begin{equation}\label{eq:elbo}
\mathcal{L}(\alphabf, \etabf):=\mbox{ELBO}(\alphabf, \etabf) = \log L_{\tt marg}(\etabf; \mathcal{D}_n) - \textsc{KL}\left( q_{\alphabf}(z) ~\|~p_{\etabf}(z\mid\mathcal{D}_n) \right),
\end{equation}
where the last term in \eqref{eq:elbo} is the KL divergence between the variation posterior $q_{\alphabf}(\zbf)$ and true posterior $p_{\etabf}(\zbf\mid\mathcal{D}_n)$, which is always nonnegative. So, $\mathcal{L}(\alphabf, \etabf)$ is a uniform lower bound for $\log L_{\tt marg}(\etabf; \mathcal{D}_n)$. If $\mathcal{L}(\alphabf, \etabf)$ is taken
as the objective function to maximize, then the result corresponds to variational inference. It is straightforward to demonstrate that the ELBO can be simplified as
\begin{equation}\label{eq:loss}
\mathcal{L}(\alphabf, \etabf) = - \textsc{KL}\left(q_{\alphabf}(\zbf)~\|~\pi_0(\zbf)\right) + \mathbb{E}_{\zbf\sim q_{\alphabf}(\zbf)}\left[{\log}\{ L(\mathcal{D}_n; G_{\etabf}(\zbf))\}\right].
\end{equation}
Note that the ELBO in \eqref{eq:loss} can be evaluated efficiently. This is because the first KL term in \eqref{eq:loss} has an explicit solution, since both $q_{\alphabf}$ and $\pi_0$ are normal densities, and the second term can be unbiasedly estimated by the Monte Carlo average. Let $(\hat\alphabf, \hat\etabf)$ be the maximizer of \eqref{eq:loss}.  Then, $q^{*}:=q_{\hat\alphabf}$ is the estimated variational posterior for $\zbf$ and the push-forward measure $\pi^{*} := G_{\hat\etabf}\#q^{*}$
is the empirical Bayes variational posterior for the weights $\wbf$, which is the approximation of the true posterior $p(\wbf\mid\mathcal{D}_n)$ in \eqref{eq:bay_pred}.

\begin{algorithm}[t]
\caption{Stochastic gradient method for updating $\alphabf$ and $\etabf$} 
\label{algorithm} 
\begin{algorithmic}[1]
\REQUIRE Training data $\mathcal{D}_n=\{(\xbf_i, y_i)\}_{i=1}^n$, learning rate sequence $\{\beta^{(t)}\}$, sample size $H$\\
\ENSURE Parameter estimates $(\hat{\alphabf}$, $\hat{\etabf})$
\STATE Initialization: $\alphabf^{(0)}$, $\etabf^{(0)}$
\FOR{$t=0,\ldots,T-1$}
\STATE Simulate $H$ samples $\zbf^{(1)}, \ldots, \zbf^{(H)}$ from $q_{\alphabf^{(t)}}(.)$
\STATE Compute
\begin{equation*}
    \begin{aligned}
         \nabla_{\alphabf^{(t)}}\mathcal{L}(\alphabf^{(t)},\etabf^{(t)})&=H^{-1}\sum_{h=1}^H[\nabla_{\alphabf^{(t)}}\log \{q_{\alphabf^{(t)}}(\zbf^{(h)})\}]*(\mathrm{I})\\
         \mathrm{where}\;(\mathrm{I})&=\log \{L(\mathcal{D}_n; G_{\etabf^{(t)}}(\zbf^{(h)}))\}+\log \{\pi_0(\zbf^{(h)})/q_{\alphabf^{(t)}}(\zbf^{(h)})\}\\
        \nabla_{\etabf^{(t)}}\mathcal{L}(\alphabf^{(t)},\etabf^{(t)})&=H^{-1}\sum_{h=1}^H\nabla_{\etabf^{(t)}}\log\{ L(\mathcal{D}_n; G_{\etabf^{(t)}}(\zbf^{(h)}))\}
    \end{aligned}
\end{equation*}
\STATE update
\begin{equation*}
    \begin{aligned}
         \alphabf^{(t+1)}&=\alphabf^{(t)} + \beta^{(t)}\nabla_{\alphabf^{(t)}}\mathcal{L}(\alphabf^{(t)},\etabf^{(t)})\\
         \etabf^{(t+1)}&=\etabf^{(t)} + \beta^{(t)}\nabla_{\etabf^{(t)}}\mathcal{L}(\alphabf^{(t)},\etabf^{(t)})
    \end{aligned}
\end{equation*}
\ENDFOR
\STATE return $\hat{\alphabf}=\alphabf^{(T)}$, $\hat{\etabf}=\etabf^{(T)}$
\end{algorithmic}
\end{algorithm}

In practice, the gradient ascent will be adopted to obtain the estimates, and our algorithm computes the gradients as
\begin{equation}\label{eq:grad}
    \begin{aligned}
        &\resizebox{0.85\linewidth}{!}{$\nabla_{\alphabf}\mathcal{L}(\alphabf,\etabf)=\mathbb{E}_{\zbf\sim q_{\alphabf}(\zbf)}\left(\left[\nabla_{\alphabf}\log \{q_{\alphabf}(\zbf)\} \right]\left[{\log} \{L(\mathcal{D}_n; {G_{\etabf}(\zbf)})\}+{\log} \left\{\frac{\pi_0(\zbf)}{q_{\alphabf}(\zbf)}\right\}\right]\right)$,}\\
        &\nabla_{\etabf}\mathcal{L}(\alphabf,\etabf)=\mathbb{E}_{\zbf\sim q_{\alphabf}(\zbf)}\left[\nabla_{\etabf}{\log} \{L(\mathcal{D}_n; {G_{\etabf}(\zbf)})\}\right].
    \end{aligned}
\end{equation}
The detailed algorithm is given in Algorithm \ref{algorithm}. This is an iterative algorithm that updates the estimated $\hat{\alphabf}$ and $\hat{\etabf}$ at each iteration. In practice, for variational parameters $(\varrho_1,\ldots, \varrho_{r})$, we apply the reparameterization trick $\varrho_j=\log(1+e^{\rho_j})$, for $j=1,\ldots,r$, and update the quantities $\rho_j$ in each step instead of $\varrho_j$ as this guarantees non-negative estimates of standard deviation $\varrho_j$ \citep{ranganath2014black, blundell2015weight}. Further
details about the algorithm can be found in the Appendix.

\section{Theoretical Analysis}\label{sec:theory}

In this section, we investigate the theoretical properties of the variational posterior. We have established the uniform variational posterior consistency over a class of transformations $G_{\etabf}$ with some smoothness conditions on $G_{\etabf}$. We have also established the classification accuracy of the variational posterior.

%

\subsection{Consistency of variational posterior}

We will mainly focus on binary classification, and the conditional density of $y$ given $\xbf$ under the truth is
\begin{equation}\label{eq:binary}
    \ell_0(y, \xbf) = \exp\{yf_0(\xbf)-\log (1 + e^{f_0(\xbf)})\},
\end{equation}
where $f_0:\mathbb R^p\mapsto\mathbb R$ is an unknown continuous function of the log-odds ratio.
Let $f_{\wbf}$ be a neural network approximation of $f_0$ with network weights $\wbf$. Write 
\begin{equation}\label{eq:l_w}
    \ell_{\wbf}(y, \xbf) = \exp \{yf_{\wbf}(\xbf)-\log (1 + e^{f_{\wbf}(\xbf)})\}.
\end{equation}
Without loss of generality, assume $\xbf_i\sim Unif[0, 1]^p$, for $i=1,2,\ldots,n$, which implies $p(\xbf) = 1$ and $p(\xbf\mid\wbf) = 1$.
Define the Hellinger neighborhood of the true density function $g_0 = \ell_0$ under the true model $f_0$ as
\begin{equation}\label{eq:hellinger_w}
    \mathcal{U}_{\varepsilon}=\big\{\wbf: d_{\mathrm{H}}(\ell_0, \ell_{\wbf})<\varepsilon\big\},
\end{equation}
where the Hellinger distance $d_{\mathrm{H}}(\ell_0, \ell_{\wbf})$ is expressed by
\begin{equation*}
    d_{\mathrm{H}}(\ell_0, \ell_{\wbf}) = \left[\frac{1}{2}\int_{\xbf\in[0,1]^p}\sum_{y\in\{0,1\}}\left\{\sqrt{\ell_0(y, \xbf)}-\sqrt{\ell_{\wbf}(y, \xbf)}\right\}^2d\xbf\right]^{\frac{1}{2}}.
\end{equation*}
Similarly, 
the Kullback--Leibler neighborhood of the true density function
$\ell_0$ under the truth $f_0$ is defined as 
\begin{equation*}
    \mathcal{N}_{\varepsilon}=\big\{\wbf: d_{\mathrm{KL}}(\ell_0, \ell_{\wbf})<\varepsilon\big\},
\end{equation*}
where the KL distance $d_{\mathrm{KL}}(\ell_0, \ell_{\wbf})$ is given by
\begin{equation*}
    d_{\mathrm{KL}}(\ell_0, \ell_{\wbf}) = \int_{\xbf\in[0,1]^p}\sum_{y\in\{0,1\}}\left[\log\left\{\frac{\ell_0(y,\xbf)}{\ell_{\wbf}(y,\xbf)}\right\}\ell_0(y,\xbf)\right]d\xbf.
\end{equation*}
In the following, we use the notation $P_0$ to denote the true distribution of $\mathcal{D}_n = \{(\xbf_i, y_i)\}^n_{i=1}$ under the true density $\ell_0$. 
Regarding the asymptotic analysis of NA-EB, we assume that both $r$ and $D$, the dimensions of $\zbf$ and $\wbf$, respectively, depend on $n$ and thus rewrite $r=r_n$, $D=D_n$.

Assume $G_{\etabf}\in \mathcal{G}$ and we put some smoothness conditions on the functions in $\mathcal{G}$ through the constraints on the Jacobian of the function. Details of the conditions are given in Assumption \ref{assump:assump1a} of the Appendix. The intuition is to improve the stability of the model, which avoids the situation where infinitesimal perturbations amplify and have substantial impacts on the performance of the output of $G_{\etabf}$. In practice, a Jacobian regularization can be added to the objective function, and a computationally efficient algorithm has been implemented by \cite{hoffman2019robust}. The prior parameters in \eqref{eq:prior} satisfy $\|\mubf\|_2^2=o(n)$ and 
\begin{equation}\label{eq:var}
    \|\zetabf\|_{
        \infty}=O(n),
        \quad\|\zetabf^*\|_{
        \infty}=O(1),
    \end{equation}
where $\zetabf^*=(1/\zeta_1,\ldots,1/\zeta_{r})^{\T}$.
This assumption, which is Assumption \ref{assump:assump1b} of the Appendix, imposes restrictions on the prior parameters so that the KL distance between the variational posterior $q^*$ and the true posterior $p(\wbf\mid\mathcal{D}_n)$ is negligible.

\begin{theorem}\label{thm:thm4.1}
Suppose $r_n\sim n^a$, $D_n\sim n^u$, $0<a\leq u<1$. 
Then, under Assumptions \ref{assump:assump1a} and \ref{assump:assump1b} in the Appendix, 
\[\sup_{G_{\etabf}\in \mathcal{G}}\pi^{*}(\mathcal{U}^c_{\varepsilon})\stackrel{P_0}{\longrightarrow} 0.\]
\end{theorem}

Theorem \ref{thm:thm4.1} indicates that, under some regularity conditions, for any $G_{\etabf}\in\mathcal{G} $ and any $\nu > 0$, $\pi^{*}(\mathcal{U}^c_{\varepsilon})<\nu$ with probability tending to 1 as $n\rightarrow\infty$. Under the conditions of Theorem \ref{thm:thm4.1} with less restrictive assumptions on $G_{\etabf}$, it can be proved that the true posterior satisfies $p(\mathcal{U}^c_{\varepsilon}\mid\mathcal{D}_n)<2e^{-n\varepsilon^2/2}$ with probability tending to 1 as $n\rightarrow\infty$ as shown in Theorem \ref{thm:thm7.17} part 1 of the Appendix. This implies that the probability of the $\varepsilon$-small Hellinger neighborhood of the true function $\ell_0$ for the true posterior increases at a rate of $1 -2e^{-n\varepsilon^2/2}$
in contrast to the slow rate of $1-\nu$ for the variational posterior. On the other hand, the consistency of the variational posterior requires more conditions on the implicit model $G_{\etabf}$ than that of the true posterior, since the Bayesian posterior is hard to compute due to the intractable integrals.

Although we have established the uniform posterior variational consistency for any transformation function $G_{\etabf}$, our numerical experiences demonstrate that better predictive performance is achieved using the estimated $\etabf$ of our algorithm than a randomly or manually selected $\etabf$.

To establish the consistency of the variational posterior in a shrinking Hellinger neighborhood of $\ell_0$, we need to modify Assumptions \ref{assump:assump1a} and \ref{assump:assump1b}.
Compared to Assumptions \ref{assump:assump1a} and \ref{assump:assump1b}, the square of the Frobenius norm of the Jacobian in Assumption \ref{assump:assump2a} and the rate of growth of $L_2$ norm of the prior mean parameter in Assumption \ref{assump:assump2b} are allowed to grow slower since the consistency of the variational posterior in a shrinking Hellinger neighborhood of $\ell_0$ is more restrictive in nature. Furthermore, it requires the existence of a neural network solution that converges to the true function $\ell_0$ at a sufficiently fast rate while ensuring controlled growth of the $L_2$ norm of its coefficients. These assumptions are summarized in Assumptions \ref{assump:assump2a}, \ref{assump:assump2b}, and \ref{assump:assump2c} of the Appendix.

\begin{theorem}\label{thm:thm4.2}
Suppose $r_n\sim n^a$, $D_n\sim n^u$, $0<a\leq u<1$ and $\epsilon_n^2\sim n^{-\delta}$, $0<\delta<1-u\leq 1-a$. 
Then, under Assumptions \ref{assump:assump2a}, \ref{assump:assump2b}, and \ref{assump:assump2c} of the Appendix, 
$$\sup_{G_{\etabf}\in \mathcal{G}}\pi^{*}(\mathcal{U}^c_{\varepsilon\epsilon_n})\stackrel{P_0}{\longrightarrow} 0.$$
\end{theorem}
A restatement of Theorem \ref{thm:thm4.2} is for any $G_{\etabf}\in\mathcal{G} $ and any $\nu > 0$, $\pi^{*}(\mathcal{U}^c_{\varepsilon\epsilon_n})<\nu$ with probability tending to 1 as $n\rightarrow\infty$. Under the conditions of Theorem \ref{thm:thm4.2} with less restrictive assumptions on $G_{\etabf}$, it can be established that the true posterior satisfies $p(\mathcal{U}^c_{\varepsilon\epsilon_n}\mid\mathcal{D}_n)<2e^{-n\varepsilon^2\epsilon_n^2/2}$ with probability tending to 1 as $n\rightarrow\infty$ as shown in Theorem \ref{thm:thm7.18} part 1 of the Appendix. This implies that the probability of the $\varepsilon\epsilon_n$-small Hellinger neighborhood of the true function $\ell_0$ for the true posterior increases at the rate of $1 -2e^{-n\varepsilon^2\epsilon_n^2/2}$
in contrast to the slow rate of $1-\nu$ for the variational posterior.

\subsection{Classification accuracy}

In this subsection, we establish the classification accuracy of the variational posterior. A classifier $C$ is a Borel-measurable function
$C: \mathbb{R}^p\mapsto\{0, 1\}$, which assigns a sample $\xbf\in\mathbb{R}^p$ to the class $C(\xbf)$. The misclassification error risk of a classifier $C$ is given by
\begin{equation}\label{eq:testerror}
    R(C) = \int_{\mathbb{R}^p\times\{0, 1\}}\mathds{1}(C(\xbf)\neq y)dP_{\xbf, y}.
\end{equation}
where $\mathds{1}(\cdot)$ denotes the indicator function.
The Bayes classifier is defined as
\begin{equation}\label{eq:bayes_class}
    C^{\tt{Bayes}}(\xbf) = \left\{
    \begin{aligned}
        &1, \quad\sigma(f_0(\xbf))\geq1/2\\
        &0, \quad\mathrm{otherwise},
    \end{aligned}
    \right.
\end{equation}
where $\sigma(x) = e^x/(1 + e^x)$ is the sigmoid function.

As the predictive distribution of Algorithm \ref{algorithm} can be estimated in \eqref{eq:pred_dist} of the Appendix, the classifier based on the variational posterior is given by
\begin{equation}\label{eq:varational_class}
    \hat{C}(\xbf) = \left\{
    \begin{aligned}
        &1, \quad\sigma(\hat{f}(\xbf))\geq1/2\\
        &0, \quad\mathrm{otherwise},
    \end{aligned}
    \right.
\end{equation}
where $\hat{f}(\xbf)=\sigma^{-1}(\int\sigma(f_{G_{\hat{\etabf}}(\zbf)}(\xbf))q^*(\zbf)d\zbf)$ is the variational estimator of $f_0(\xbf)$.

Although the Bayes classifier is optimal in terms of minimizing the test error in \eqref{eq:testerror} \citep{hastie2009elements}, it is not useful in practice, as the truth $f_0$ is unknown. We compare the classification accuracy of the Bayes classifier in \eqref{eq:bayes_class} and the variational classifier in \eqref{eq:varational_class} in Corollaries \ref{cor:cor4.3} and \ref{cor:cor4.4} under different assumptions on the prior parameters and the deterministic transformation function $G_{\etabf}$.

\begin{corollary}\label{cor:cor4.3}
Under the conditions of Theorem \ref{thm:thm4.1},
\begin{equation*}
    \sup_{G_{\etabf}\in \mathcal{G}}\left|R(\hat{C})-R(C^{\tt{Bayes}})\right|\stackrel{P_{\xbf, y}}{\longrightarrow}0.
\end{equation*}
\end{corollary}

Corollary \ref{cor:cor4.3} can be rephrased as for any $G_{\etabf}\in\mathcal{G} $ and any $\nu > 0$, $|R(\hat{C})-R(C^{\tt{Bayes}})|<\nu$ with probability tending to 1 as $n\rightarrow\infty$. As part 2 of Theorem \ref{thm:thm7.17} in the Appendix shows that the true posterior gives the classification accuracy at the same consistency rate under the conditions of Theorem \ref{thm:thm4.1} with less restrictive assumptions on $G_{\etabf}$ which indicates that there is no harm using the variational posterior approximation.

\begin{corollary}\label{cor:cor4.4}
Under the conditions of Theorem \ref{thm:thm4.2}, for every $0\leq\kappa\leq\frac{2}{3}$
\begin{equation*}
    \sup_{G_{\etabf}\in \mathcal{G}}\epsilon_n^{-\kappa}\left|R(\hat{C})-R(C^{\tt{Bayes}})\right|\stackrel{P_{\xbf,y}}{\longrightarrow}0.
\end{equation*}
\end{corollary}

A restatement of Corollary \ref{cor:cor4.4} is for any $G_{\etabf}\in\mathcal{G} $ and any $\nu > 0$, $0\leq\kappa\leq\frac{2}{3}$, $|R(\hat{C})-R(C^{\tt{Bayes}})|<\nu\epsilon_n^{\kappa}$ with probability tending to 1 as $n\rightarrow\infty$. As part 2 of Theorem \ref{thm:thm7.18} in the Appendix shows, under the conditions of Theorem \ref{thm:thm4.2} with less restrictive assumptions on $G_{\etabf}$, the true posterior satisfies $\epsilon_n^{-\kappa}|R(\hat{C})-R(C^{\tt{Bayes}})|\stackrel{P_{\xbf, y}}{\longrightarrow}0$ for every $0\leq\kappa\leq 1$.

\section{Numerical Results}\label{sec:numerical}

In this section, we evaluate our NA-EB in regression and classification experiments. For these two classical types of supervised learning problems, we illustrate the performance of our proposed framework and algorithm in terms of predictive accuracy and predictive uncertainty with real data analysis and synthetic datasets, respectively. 



\subsection{Two-spiral problem}

Consider the classification problem of learning a mapping for the two-spiral dataset $\{(\xbf_i,y_i)\}_{i=1}^{194}$ in which the samples $(\xbf_i,y_i)\in\mathbb{R}^2\times\mathbb{R}$ are generated from:
\begin{equation*}
    \begin{aligned}
        & \resizebox{0.9\linewidth}{!}{$\xbf_{i, 1} = 6.5 \times (-1)^{\{i\pmod{2}\}} \times \left[1-\frac{i-\{i\pmod{2}\}}{208}\right] \times \sin\left(\frac{[i-\{i\pmod{2}\}]\pi}{32}\right)$},\\
        & \resizebox{0.9\linewidth}{!}{$\xbf_{i, 2} = 6.5 \times (-1)^{\{i\pmod{2}\}} \times \left[1-\frac{i-\{i\pmod{2}\}}{208}\right] \times \cos\left(\frac{[i-\{i\pmod{2}\}]\pi}{32}\right)$},\\
        & y_i = i\pmod{2},
    \end{aligned}
\end{equation*}
where $i\pmod{2}$ is the remainder after dividing $i$ by 2, for $i=1,\ldots,194$, and the sample points on two intertwined spirals are shown in Figure \ref{fig:spiral}.

\begin{figure}[tb]
    \centering
    \includegraphics[scale=0.62]{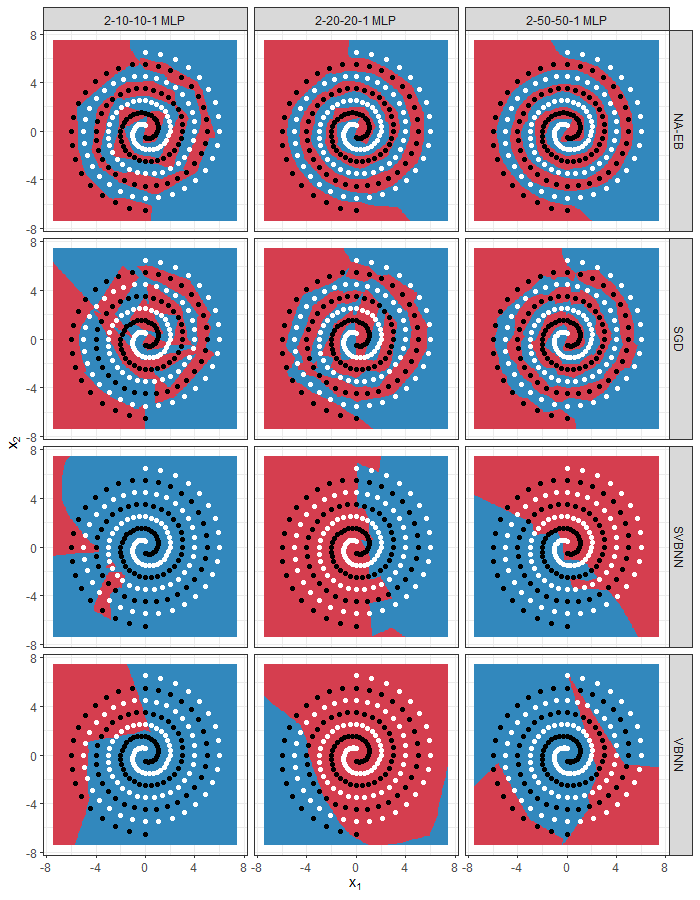}
    \caption{NA-EB, SGD, VBNN, and SVBNN classifying maps using 2-10-10-1 MLP, 2-20-20-1 MLP, and 2-50-50-1 MLP respectively. Black and white points denote training data for two spirals. Red and blue regions indicate the two classified classes.}
    \label{fig:spiral}
\end{figure}

Since no structural information is assumed for the mapping, a feedforward neural network, which is also known as the multiple layer perceptron (MLP), can be used to distinguish between points. We adopt fully connected two-hidden-layer MLPs consisting of two input units, $h$ hidden units for both layers, one output unit, and the ReLu activation function denoted by the $2$-$h$-$h$-$1$ MLP $(h=10, 20, 50)$. Besides, we employ a one-hidden-layer MLP comprising 128 rectified linear units as the deterministic transformation function $G_{\etabf}$. For comparison, the results of a standard neural network optimized by stochastic gradient descent via backpropagation (SGD) \citep{rumelhart1986learning}, a variational Bayesian algorithm (VBNN) \citep{blundell2015weight}, and a sparse variational BNN (SVBNN) \citep{bai2020efficient} are also reported in Figure \ref{fig:spiral}. The comparison indicates that NA-EB outperforms SGD, SVBNN, and VBNN in terms of predictive performance. NA-EB can find perfect solutions that distinguish between points on two intertwined spirals with smooth boundaries for different architectures of MLPs. However, SGD performs rather poorly as the number of hidden units of MLPs decreases. We also observe in this example that VBNN and SVBNN have not converged well as it requires the MLP to learn a highly nonlinear separation of the input space. 

\begin{figure}[tb]
    \centering
    \includegraphics[scale=0.3]{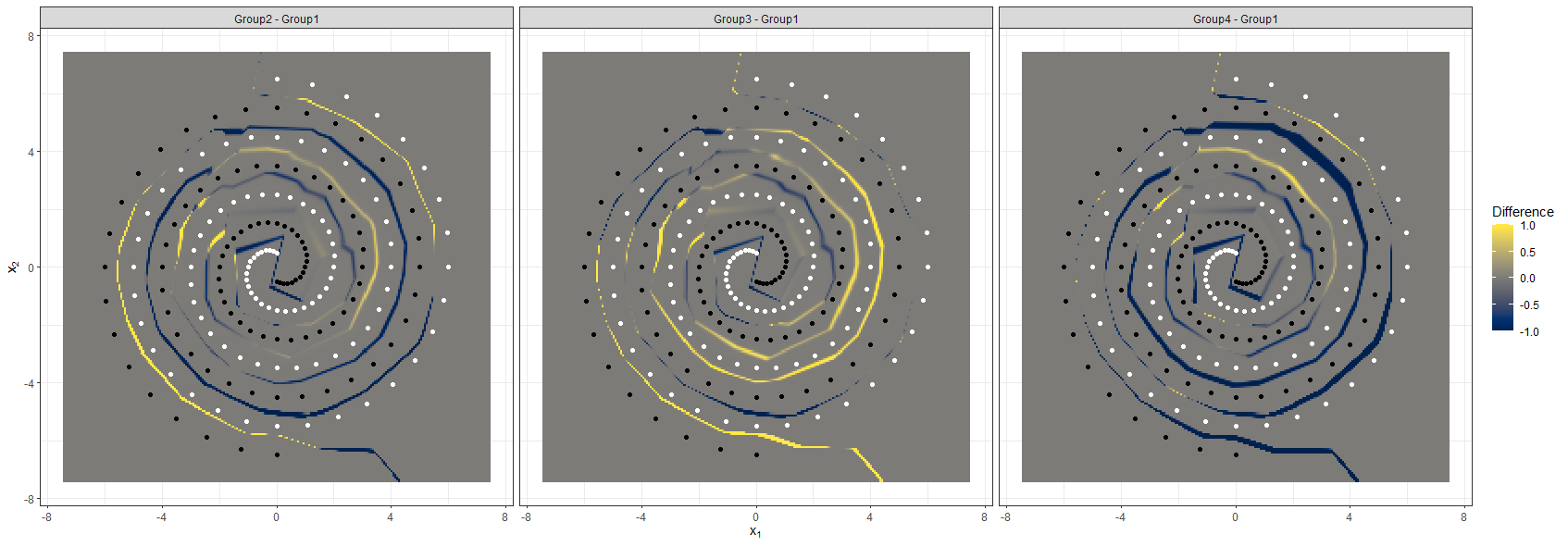}
    \caption{Difference maps of the predicted probability from four groups of weights sampled from the variational posterior of NA-EB. The yellow and dark blue areas in each map indicate the regions classified as different class by two groups of weights, respectively.}
    \label{fig:diff}
\end{figure}

Furthermore, we sample weights from the learned variational posterior distribution multiple times and then compare their classification maps to illustrate the uncertainty in the weights. In particular, we generate 4 groups of weights $\{\wbf^{(1)}, \wbf^{(2)}, \wbf^{(3)}, \wbf^{(4)}\}$ of the $2$-$20$-$20$-$1$ MLP from the variational posterior and display the maps of differences between predicted probability from these groups $\{p(f_{\wbf^{(v)}}(\xbf)=1\mid\xbf)-p(f_{\wbf^{(1)}}(\xbf)=1\mid\xbf)\}_{v=2}^4$ respectively in Figure \ref{fig:diff}. As can be seen, the absolute values of the probability differences tend to be 1 in the middle of two intertwined spirals and 0 elsewhere. This indicates that the variational posterior prefers to be uncertain in the middle of two spirals, as it is reasonable to classify this region as either of two classes.

\subsection{Synthetic 1D Experiments}

In this subsection, we consider synthetic regression problems and demonstrate the predictive distribution obtained by NA-EB in toy datasets. We generate two datasets that consist of different nonlinear functions.
We sample the inputs $x$ uniformly from the interval $[-4, 4]$ and then generate training data from the first curve:
\begin{equation*}
    y=x^3+2x+3+\epsilon,
\end{equation*}
where $\epsilon\sim N(0, 9)$. We also generate sample points from the second curve:
\begin{equation*}
    y=x-0.3\sin (2\pi x)+\epsilon,
\end{equation*}
where the inputs $x$ are uniformly sampled from the interval $[0, 0.6]$ and $\epsilon\sim N(0, 0.02^2)$.
We compare our method with the standard stochastic gradient descent via backpropagation (SGD) \citep{rumelhart1986learning}, with the sparse variational BNN (SVBNN) \citep{bai2020efficient} and with a variational inference (VI) approach \citep{graves2011practical}. The neural network architecture includes two hidden layers with 100 and 50 hidden units for each hidden layer.
Besides, we use a one-hidden-layer MLP with 128 rectified linear units as the deterministic transformation function $G_{\etabf}$. We use 300 training epochs for all methods on the training data. 

\begin{figure}[tb]
     \begin{center}
     \begin{subfigure}[b]{0.7\textwidth}
         \centering
         \includegraphics[scale=0.4]{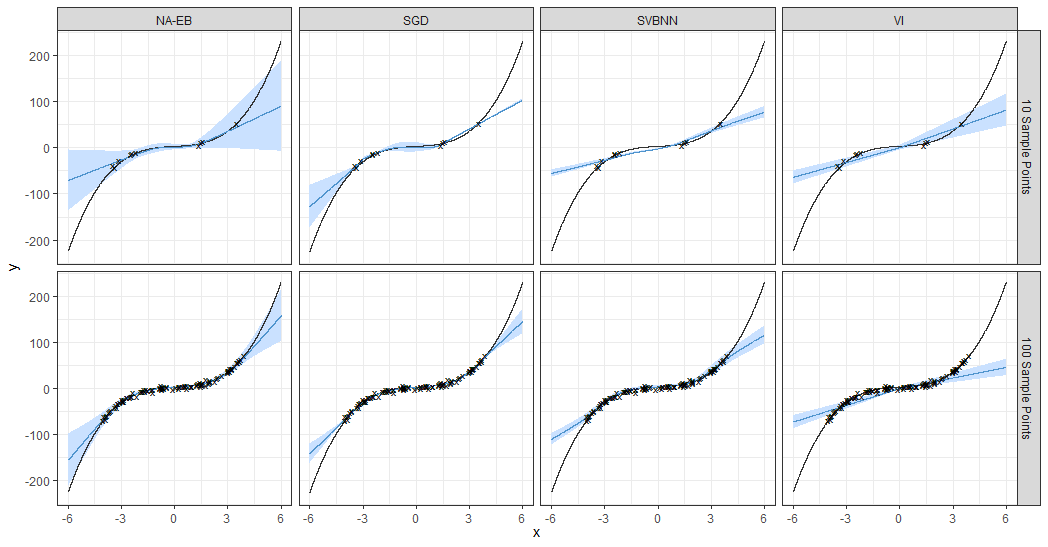}
         \caption{$y=x^3+2x+3+\epsilon$}
         \label{fig:func1}
     \end{subfigure}
     \begin{subfigure}[b]{0.7\textwidth}
         \centering
         \includegraphics[scale=0.4]{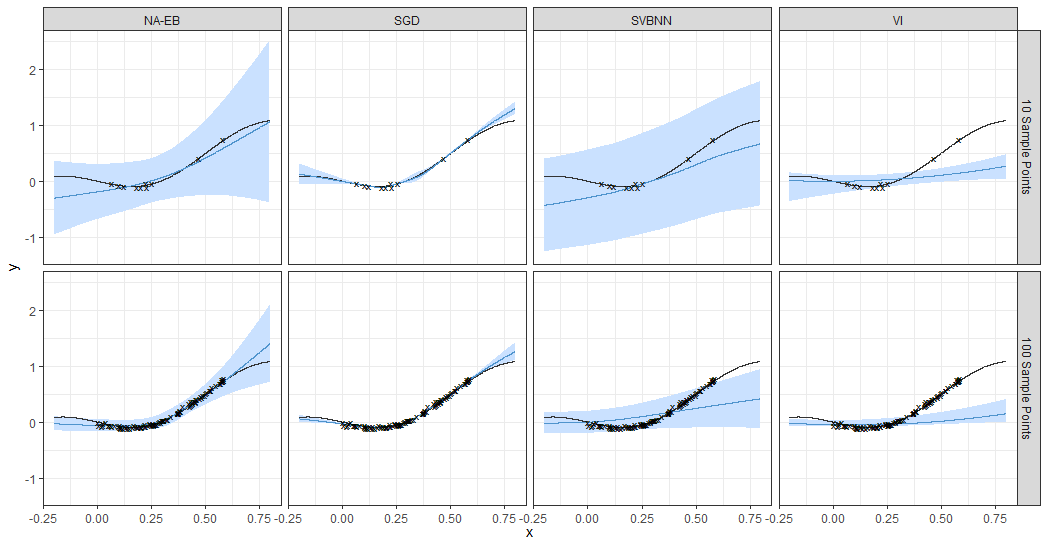}
         \caption{$y=x-0.3\sin (2\pi x)+\epsilon$}
         \label{fig:func2}
     \end{subfigure}
     \end{center}
     \caption{Regression of two toy datasets with credible intervals using 10 sample points and 100 sample points made by NA-EB, SGD, SVBNN and VI.}
     \label{fig:toy}
\end{figure}

Figures \ref{fig:func1} and \ref{fig:func2}  show the predictions and credible intervals generated by each method. Noisy training samples from two datasets are shown as black crosses, with true data-generating functions depicted by black continuous lines and mean predictions shown as dark blue lines. The light blue areas represent credible intervals at $\pm 3$ standard deviations. In the sample interval, NA-EB and SGD give mean predictions that are much closer to the true data-generating function than VI and SVBNN given the same sample size for both datasets. On the other hand, in the regions of the input space where there are no data, NA-EB generates a diverged confidence interval reflecting that there are many possible extrapolations, while SGD fits a specific curve with the variance almost reduced to zero. This indicates that NA-EB prefers to be uncertain when nearby data is unavailable, in contrast to a standard neural network that can be overly confident. Furthermore, as the number of sample points increases, the prediction accuracy of NA-EB improves, and the approximated uncertainty of NA-EB decreases.

\subsection{UCI datasets}

We further evaluate the predictive accuracy and uncertainty quantification of NA-EB on real-world datasets for regression tasks. We adopt the same set of 10 UCI regression benchmark datasets \citep{Dua:2019} as well as the experimental protocol proposed in \cite{hernandez2015probabilistic} and followed by \cite{gal2016dropout, lakshminarayanan2017simple, han2022card}. These datasets are available at \url{https://archive.ics.uci.edu/datasets}.

Each data set is randomly divided into training and test sets with 90\% and 10\% of the data, respectively. The splitting process is repeated 20 times for all datasets except that for Year dataset and Protein dataset, we do the train-test splitting only one and five times, respectively, due to their large sample sizes. Also, we normalize all datasets so that the input features and targets have zero mean and unit variance in the training set, and remove the normalization for evaluation. Furthermore, for both the Kin8nm and Naval dataset, we multiply the response variable by 100 which is the same as \cite{han2022card} to match the scale of other datasets. We compare our method with four BNN methods: probabilistic back-propagation (PBP) \citep{hernandez2015probabilistic}, the dropout uncertainty (MC Dropout) approach \citep{gal2016dropout}, the deep ensembles (Ensembles) Bayesian method \citep{lakshminarayanan2017simple}, and the sparse variational BNN (SVBNN) \citep{bai2020efficient}. We also compare our method with two deep generative model approaches: the generative conditional distribution sampler (GCDS) \citep{zhou2022deep} and classification and regression diffusion models (CARD) \citep{han2022card}.  
To be consistent with the results summarized in \cite{han2022card}, we adopt the same experimental setup: a two-hidden-layer network architecture with 100 and 50 hidden units, the ReLU activation function, the Adam optimizer, the batch size varied case by case and 500 training epochs. Besides, we specify the deterministic transformation function $G_{\etabf}$ by a two-hidden-layer MLP comprising 128 rectified linear units for each layer.

Recent BNN approaches \citep{hernandez2015probabilistic, gal2016dropout, lakshminarayanan2017simple} employ the negative log-likelihood (NLL) to quantify uncertainty. However, NLL computation assumes a Gaussian density for the conditional distributions $p(y \mid\xbf;\wbf)$ for all $\xbf$, which may not hold for real-world datasets. Therefore, we adopt the quantile interval coverage error (QICE) metric proposed by \cite{han2022card} as a measure of uncertainty for regression tasks. Simultaneously, for classification tasks that will be discussed in detail in the subsequent sections, we maintain the use of NLL to quantify uncertainty, aligning with the metrics reported in classification experiments of \cite{han2022card}. As stated in \cite{han2022card}, QICE is defined as the mean absolute error between the proportion of true data contained within each quantile interval of generated samples of size $N$ and the ideal proportion:
\begin{equation*}
    \mathrm{QICE}:=\frac{1}{M}\sum_{m=1}^M\left|r_m-\frac{1}{M}\right|,\quad\mathrm{where}\; r_m=\frac{1}{N}\sum_{n=1}^N\mathds{1}(y_n\geq \hat{y}_n^{\tt{low}_m})\mathds{1}(y_n\leq \hat{y}_n^{\tt{high}_m}),
\end{equation*}
where $\hat{y}_n^{\tt{low}_m}$ and $\hat{y}_n^{\tt{high}_m}$ represent the low and high percentiles of the $m$th quantile interval, respectively, of our choice for the
predicted $y_n$ outputs, for $n=1,\ldots,N$, given the same $\xbf$ input.
Ideally, when the learned conditional distribution perfectly matches the true one, the QICE value should be 0. QICE is an empirical metric that does not impose Gaussian restrictions or any specific parametric form on the conditional distribution. Similar to NLL, it relies on the summary statistics of samples from the learned distribution to empirically evaluate the similarity between the learned and true conditional distributions. To be consistent with the results presented in \cite{han2022card}, we use the same parameter $M=10$ quantile intervals to calculate QICE.

Tables \ref{tb:uci_rmse} and \ref{tb:uci_qice} summarize the average test root mean squared error (RMSE) and QICE with their standard deviation across all splits for each method, respectively. We observe that NA-EB obtains the best results in 8 out of 10 datasets in terms of RMSE and 4 out of 10 for QICE, and it is competitive with the best method for the remaining datasets. It should be noted that although we do not explicitly optimize our model with respect to MSE or QICE, we still outperform existing models trained with these objectives. We also list the number of unknowns for all alternative methods in Table \ref{tb:uci_rmse} for fair comparison. It is important to note that we assume that all methods employ the same base classifier, with the weights of this classifier being of dimension $D$. The key difference between NA-EB and other BNN methods is how we define the prior. In NA-EB, we define the implicit prior by $\wbf = G_{\etabf}(\zbf)$, where $\zbf$ follows a standard Gaussian, and $\etabf$ is the so-called hyperparameter in statistics. A subjective Bayesian will choose a known $\etabf$, but the performance is poor in this setting. Instead, we incorporate the concept of empirical Bayes to estimate these hyperparameters from the data. This will cause another approximately $100D$ hyperparameters in our UCI examples.

\begin{table}
\caption{Test RMSE of UCI regression tasks. Boldface indicates the method with the smallest test RMSE.}
\centering
\label{tb:uci_rmse}
\resizebox{\textwidth}{!}{
\begin{tabular}{lccccccc}
\hline\\[-2.2ex]
\multirow{2}{*}{Dataset} &
\multicolumn{7}{c}{Average Test RMSE and Standard Errors} \\
\cline{2-8}\\[-2.2ex]
& PBP & MC Dropout & Ensembles &SVBNN & GCDS & CARD & NA-EB \\
\hline\\[-2.2ex]
\# Parameters/Hyperparameters & $\sim 2D$ & $\sim D$ & $\sim 5D$ & $\sim 3D$ & $\sim 2D$ & $\sim D$ & $\sim 100D$ \\
\hline\\[-2.2ex]
Boston & 2.89 $\pm$ 0.74 & 3.06 $\pm$ 0.96 & 3.17 $\pm$ 1.05 & 3.17 $\pm$ 0.57 & 2.75 $\pm$ 0.58 & 2.61 $\pm$ 0.63 & {\bf2.18 $\pm$ 0.27}\\
Concrete & 5.55 $\pm$ 0.46 & 5.09  $\pm$ 0.60 & 4.91 $\pm$ 0.47 & 5.57 $\pm$ 0.47 & 5.39 $\pm$ 0.55 & 4.77 $\pm$ 0.46 & {\bf3.87 $\pm$ 0.59} \\
Energy & 1.58 $\pm$ 0.21 & 1.70 $\pm$ 0.22  & 2.02 $\pm$ 0.32 & 1.92 $\pm$ 0.19 & 0.64 $\pm$ 0.09 & 0.52 $\pm$ 0.07 & {\bf0.39 $\pm$ 0.05}\\
Kin8nm & 9.42 $\pm$ 0.29 & 7.10 $\pm$ 0.26 & 8.65 $\pm$ 0.47 & 9.37 $\pm$ 0.26 & 8.88 $\pm$ 0.42 & {\bf6.32 $\pm$ 0.18} & 6.74 $\pm$ 0.13\\
Naval & 0.41 $\pm$ 0.08 & 0.08 $\pm$ 0.03 & 0.09 $\pm$ 0.01 & 0.21 $\pm$ 0.05 & 0.14 $\pm$ 0.05 & 0.02 $\pm$ 0.00 & {\bf 0.02 $\pm$ 0.00}\\
Power & 4.10 $\pm$ 0.15 & 4.04 $\pm$ 0.14 & 4.02 $\pm$ 0.15 & 4.01 $\pm$ 0.18 & 4.11 $\pm$ 0.16 & 3.93 $\pm$ 0.17 & {\bf 3.52 $\pm$ 0.14}\\
Protein & 4.65 $\pm$ 0.02 & 4.16 $\pm$ 0.12 & 4.45 $\pm$ 0.02 & 4.30 $\pm$ 0.05 & 4.50 $\pm$ 0.02 & 3.73 $\pm$ 0.01 & {\bf 3.65 $\pm$ 0.04}\\
Wine & 0.64 $\pm$ 0.04 & 0.62 $\pm$ 0.04 & 0.63 $\pm$ 0.04 & 0.62 $\pm$ 0.04 & 0.66 $\pm$ 0.04 & 0.63 $\pm$ 0.04 & {\bf 0.57 $\pm$ 0.04}\\
Yacht & 0.88 $\pm$ 0.22 & 0.84 $\pm$ 0.27 & 1.19 $\pm$ 0.49 & 1.10 $\pm$ 0.27 & 0.79 $\pm$ 0.26 & 0.65 $\pm$ 0.25 & {\bf 0.23 $\pm$ 0.05}\\
Year & 8.86 $\pm$ NA & 8.77 $\pm$ NA & 8.79 $\pm$ NA & 8.87 $\pm$ NA & 9.20 $\pm$ NA & {\bf 8.70 $\pm$ NA} & 8.76 $\pm$ NA\\
\hline
\end{tabular}
}
\end{table}

\begin{table}
    \caption{Test QICE (in \%) of UCI regression tasks. Boldface indicates the method with the smallest test QICE.}
    \centering
    \label{tb:uci_qice}
    \resizebox{\textwidth}{!}{
    \begin{tabular}{lcccccccc}
    \hline\\[-2.2ex]
    \multirow{2}{*}{Dataset} &&
    \multicolumn{7}{c}{Average Test QICE and Standard Errors} \\
    \cline{3-9}\\[-2.2ex]
    && PBP & MC Dropout & Ensembles &SVBNN & GCDS & CARD & NA-EB \\
    \hline\\[-2.2ex]
    Boston && 3.50 $\pm$ 0.88 & 3.82 $\pm$ 0.82 & 3.37 $\pm$ 0.00 & 4.18 $\pm$ 1.24 & 11.73 $\pm$ 1.05 & 3.45 $\pm$ 0.83 & {\bf 3.36 $\pm$ 0.73}\\
    Concrete && 2.52 $\pm$ 0.60 & 4.17  $\pm$ 1.06 & 2.68 $\pm$ 0.64 & 3.50 $\pm$ 0.76 & 10.49 $\pm$ 1.01 & {\bf 2.30 $\pm$ 0.66} & 2.51 $\pm$ 0.66 \\
    Energy && 6.54 $\pm$ 0.90 & 5.22 $\pm$ 1.02  & 3.62 $\pm$ 0.58 & 5.49 $\pm$ 0.58 & 7.41 $\pm$ 2.19 & 4.91 $\pm$ 0.94 & {\bf 4.89 $\pm$ 0.82}\\
    Kin8nm && 1.31 $\pm$ 0.25 & 1.50 $\pm$ 0.32 & 1.17 $\pm$ 0.22 & 5.87 $\pm$ 0.45 & 7.73 $\pm$ 0.80 & {\bf0.92 $\pm$ 0.25} & 1.38 $\pm$ 0.26\\
    Naval && 4.06 $\pm$ 1.25 & 12.50 $\pm$ 1.95 & 6.64 $\pm$ 0.60 & {\bf 0.78 $\pm$ 0.28} & 5.76 $\pm$ 2.25 & 0.80 $\pm$ 0.21 &  3.90 $\pm$ 1.06\\
    Power && {\bf 0.82 $\pm$ 0.19} & 1.32 $\pm$ 0.37 & 1.09 $\pm$ 0.26 & 1.07 $\pm$ 0.28 & 1.77 $\pm$ 0.33 & 0.92 $\pm$ 0.21 & 1.00 $\pm$ 0.33\\
    Protein && 1.69 $\pm$ 0.09 & 2.82 $\pm$ 0.41 & 2.17 $\pm$ 0.16 & 1.22 $\pm$ 0.21 & 2.33 $\pm$ 0.18 & {\bf 0.71 $\pm$ 0.11} & 0.96 $\pm$ 0.19\\
    Wine && 2.22 $\pm$ 0.64 & 2.79 $\pm$ 0.56 & 2.37 $\pm$ 0.63 & 2.55 $\pm$ 0.65 & 3.13 $\pm$ 0.79 & 3.39 $\pm$ 0.69 & {\bf 2.15 $\pm$ 0.71}\\
    Yacht && 6.93 $\pm$ 1.74 & 10.33 $\pm$ 1.34 & 7.22 $\pm$ 1.41 & 8.40 $\pm$ 1.70 & 5.01 $\pm$ 1.02 & 8.03 $\pm$ 1.17 & {\bf 4.99 $\pm$ 1.42}\\
    Year && 2.96 $\pm$ NA & 2.43 $\pm$ NA & 2.56 $\pm$ NA & 1.64 $\pm$ NA & 1.61 $\pm$ NA & {\bf 0.53 $\pm$ NA} & 1.52 $\pm$ NA\\
\hline
    \end{tabular}
    }
\end{table}

\begin{table}
\caption{Test RMSE of 4 UCI regression datasets based on NA-EB using different latent dimensions. Boldface indicates the latent
dimension with the smallest test RMSE.}
\centering
\label{tb:uci_latentsize}
\resizebox{\textwidth}{!}{
\begin{tabular}{lccccccc}
\hline \\[-2.2ex]
\multirow{3}{*}{Dataset}&\multirow{3}{*}{{\tabincell{c}{Sample\\ Size}}}&\multirow{3}{*}{{\tabincell{c}{Feature\\ Dimension}}}&&\multicolumn{4}{c}{Average Test RMSE and Standard Errors} \\
\cline{5-8}\\[-2.2ex]
&&&&\multicolumn{4}{c}{Latent Dimension} \\
\cline{5-8}\\[-2.2ex]
&&&& 20 & 60 & 100 & 140 \\
\hline\\[-2.2ex]
Yacht  &308 & 6&& 0.25 $\pm$ 0.09 & {\bf 0.23 $\pm$ 0.05} & 0.23 $\pm$ 0.07 & 0.24 $\pm$ 0.08 \\
Wine  &1599 & 11 && 0.59  $\pm$ 0.04 & 0.60 $\pm$ 0.04 & {\bf 0.57 $\pm$ 0.04} & 0.58 $\pm$ 0.04 \\
Power  & 9568 & 4 && 3.56 $\pm$ 0.15  & {\bf 3.52 $\pm$ 0.14} & 3.52 $\pm$ 0.18 & 3.59 $\pm$ 0.16\\
Protein  & 45730 & 9 && 3.73 $\pm$ 0.03  & {\bf 3.65 $\pm$ 0.04} & 3.72 $\pm$ 0.03 & 3.72 $\pm$ 0.02\\ \hline
\end{tabular}
}
\end{table}

Furthermore, we conduct experiments on UCI regression datasets to demonstrate that NA-EB is robust to the dimension $r$ of the latent variable $\zbf$. We select four datasets with various sample sizes and feature dimensions and consider different latent dimensions in an appropriate range. 
As shown in Table \ref{tb:uci_latentsize}, the average test RMSE changes insignificantly with different latent dimensions. With our network architecture,
the empirical results show that NA-EB obtains the best result when the latent dimension is approximately 10 times the feature dimension and its predictive performance is robust within a suitable range of the latent dimension. 

We further investigate the impact of different parameterizations of the deterministic transformation function $G_{\etabf}$ on the model performance. As described previously, we employ a fully connected feed-forward neural network with $r$ input nodes and $D$ output nodes as $G_{\etabf}$. We explore different parameterizations of $G_{\etabf}$ by changing the number of hidden layers and the number of hidden units within a layer. The test RMSE results using different architectures of $G_{\etabf}$ obtained from 20 runs on the four UCI datasets are summarized in the first three columns of Table \ref{tb:uci_transformation}.
Our findings indicate that the $r$-$128$-$128$-$D$ MLP architecture consistently achieves the lowest test RMSE results across all select UCI datasets. This suggests that higher complexity of this architecture allows for a more versatile and flexible representation of $G_{\etabf}$, resulting in better predictive accuracy. However, it's worth noting that the test RMSE results are quite similar between the two-hidden-layer MLPs transformations, implying the robust predictive performance of NA-EB within a reasonable range of parameterization complexities of $G_{\etabf}$.

    \begin{table}
    \centering
    \caption{RMSE of 4 UCI regression datasets based on NA-EB using different parameterizations of the deterministic transformation function $G_{\etabf}$ with or without Jacobian regularization. Boldface indicates the architecture of $G_{\etabf}$ with the smallest RMSE.}
    \label{tb:uci_transformation}
    \resizebox{\textwidth}{!}{
    \begin{tabular}{lccccc}
     \hline \\[-2.2ex]
\multirow{3}{*}{Dataset}&&\multicolumn{4}{c}{Average Test RMSE and Standard Errors} \\
\cline{3-6}\\[-2.2ex]
&&\multicolumn{4}{c}{Architecture of Transformation Function $G_{\etabf}$} \\
\cline{3-6}\\[-2.2ex]
&& $r$-128-$D$ MLP & $r$-64-64-$D$ MLP & $r$-128-128-$D$ MLP & {\tabincell{c}{$r$-128-128-$D$ MLP\\Jacobian Regularization}} \\
\hline\\[-2.2ex]
Yacht && 0.45 $\pm$ 0.12 & 0.24 $\pm$ 0.09 & {\bf 0.23 $\pm$ 0.05} & 0.33 $\pm$ 0.12 \\
Wine && 0.60 $\pm$ 0.04 & 0.60 $\pm$ 0.04 & {\bf 0.57 $\pm$ 0.04} & 0.60 $\pm$ 0.05 \\
Power && 3.70 $\pm$ 0.19  & 3.62 $\pm$ 0.21 & {\bf 3.52 $\pm$ 0.14} & 3.54 $\pm$ 0.17\\
Protein && 3.75 $\pm$ 0.06  & 3.68 $\pm$ 0.06 & {\bf 3.65 $\pm$ 0.04} & 3.72 $\pm$ 0.05\\ \hline
\end{tabular}}
    \end{table}

In addition, as the asymptotic analysis of NA-EB in Section \ref{sec:theory} suggests, the square of the Frobenius norm of the input-output Jacobian $\|J(\zbf)\|_F^2=\|\partial G_{\etabf}/\partial \zbf\|^2_F$
shall be constrained. To effectively incorporate Jacobian regularization into the training process, we optimize a joint loss function that aligns with (5) from \cite{hoffman2019robust}:
    \begin{equation*}
        \mathcal{L}_{\tt{joint}}(\alphabf,\etabf)=\mathcal{L}(\alphabf,\etabf)+\frac{\lambda_{\tt{JR}}}{2}\|J(\zbf)\|_F^2,
    \end{equation*}
where $\mathcal{L}(\alphabf,\etabf)$ is the ELBO as specified in \eqref{eq:elbo} of our paper and $\lambda_{\tt{JR}}$ is a hyperparameter that controls the relative significance of the Jacobian regularizer. \cite{hoffman2019robust} provides a computationally efficient implementation of Jacobian regularization $J(\zbf)$. We follow its PyTorch implementation available at \url{https://github.com/facebookresearch/jacobian_regularizer} and set values of the hyperparameter $\lambda_{\tt{JR}}=0.1$ by default. The test RMSE incorporating Jacobian regularization on the four datasets is reported in the last column of Table \ref{tb:uci_transformation}. Interestingly, the numerical results reveal that both approaches based on the same transformation function architecture yield similar outcomes. Consequently, we opt to proceed without Jacobian regularization to reduce computational costs. 

\subsection{MNIST dataset}

\begin{table}
\caption{Test error rates (in \%) of MNIST classification tasks with different base classifiers. Boldface indicates the method with the smallest test error.}
\label{tb:mnist_testerror}
\centering
\begin{tabular}{lccccc}
\hline \\[-2.2ex]
\multirow{2}{*}{Base Classifier}&& \multicolumn{4}{c}{Test Error} \\
\cline{3-6}\\[-2.2ex]
&& SGD & VBNN & SVBNN & NA-EB \\
\hline\\[-2.2ex]
400-400 MLP && 1.83 & 1.36 & 1.40 & {\bf 1.24}\\
800-800 MLP && 1.84 & 1.34 & 1.37 & {\bf 1.22}\\
1200-1200 MLP && 1.88 & 1.32 & 1.36 & {\bf 1.21} \\
LeNet-5 && 1.14 & 31.01 & 4.08 & {\bf 0.91} \\
\hline
\end{tabular}  
\end{table}

We demonstrate our experimental results on the MNIST digits dataset, consisting of 60,000 training and 10,000 test pixel images of size 28 by 28. Similarly to \cite{lakshminarayanan2017simple}, our motivation for classification is to improve the performance of a base classifier in terms of accuracy through a generative model on the benchmark datasets instead of achieving the state-of-the-art predictive performance, as the latter is strongly related to network architecture design. Here, we shall focus on improving the performance of an ordinary feed-forward neural network of various sizes without using convolutions, distortions, etc. 
Besides, we also provide empirical results for NA-EB based on the LeNet-5 \citep{lecun1998gradient} architecture involving convolutional networks to highlight its efficacy in the robustness of the choice of base classifier architectures.
Specifically, we repeat the protocol and the network architecture of \cite{blundell2015weight} that an MLP of two hidden layers of $h$ rectified linear units and a softmax output layer with 10 units, denoted by the $h$-$h$ MLP $(h=400, 800, 1200)$, is trained with the Adam optimizer using a learning rate of $10^{-4}$ and minibatches of size 128 for 600 training epochs. On the other hand, we use the Adam optimizer, set the learning rate to $10^{-3}$, employ minibatches of size 32, and conduct training over 100 epochs in the experiment based on the LeNet-5 architecture. In both experiments utilizing different base classifiers, we use a feed-forward neural network with one hidden layer of 128 rectified linear units as the deterministic transformation $G_{\etabf}$ for NA-EB. Following \cite{blundell2015weight}, we preprocess the pixels by dividing the values by 126.

\begin{table}
    \caption{Test NLL of MNIST classification tasks with different base classifiers. Boldface indicates the method with the smallest test NLL.}
    \centering
    \label{tb:mnist_nll}
    \begin{tabular}{lcccc}
    \hline \\[-2.2ex]
\multirow{2}{*}{Base Classifier}&& \multicolumn{3}{c}{Test NLL} \\
\cline{3-5}\\[-2.2ex]
&& VBNN & SVBNN & NA-EB \\
\hline\\[-2.2ex]
400-400 MLP && 0.118 & 0.144 & {\bf 0.057}\\
LeNet-5 && 0.877 & 0.102 & {\bf 0.034}\\
\hline
\end{tabular} 
\end{table}

\begin{table}
    \caption{The CPU time (in s) of MNIST classification tasks with different base classifiers. Boldface indicates the method with the shortest CPU time.}
    \centering
    \label{tb:mnist_time}
    \begin{tabular}{lcccccc}
 \hline \\[-2.2ex]
\multirow{2}{*}{Base Classifier}&\multirow{2}{*}{Epochs}&& \multicolumn{4}{c}{CPU Time} \\
\cline{4-7}\\[-2.2ex]
& && SGD & VBNN & SVBNN & NA-EB \\
\hline\\[-2.2ex]
400-400 MLP & 600 && {\bf 8366.89} & 47201.29 & 352780.21 & 130016.78\\
LeNet-5 & 100 && {\bf 2328.70} & 10088.73 & 44615.34 & 13957.58\\
\hline
\end{tabular}  
\end{table}

As summarized in Table \ref{tb:mnist_testerror}, we compare the test error of our proposed method for different base classifier architectures with the performance of SGD \citep{simard2003best}, sparse variational BNN (SVBNN) \citep{bai2020efficient} and variational BNN (VBNN) reported in \cite{blundell2015weight}. Here, the inclusion of SGD serves as a baseline reference to assess predictive accuracy of NA-EB and other variational BNN methods. Meanwhile, the learning curves on the test set for these methods based on a network with two layers of 1200 rectified linear units and on the LeNet-5 architecture are presented in Figures \ref{fig:mnist_mlp} and \ref{fig:mnist_lenet5}, respectively. In addition, the test NLL results based on the two base classifiers for NA-EB, VBNN \citep{blundell2015weight} and SVBNN \citep{bai2020efficient} are summarized in Table \ref{tb:mnist_nll}. We further report the CPU time cost on a 2.30 GHz computer for each method in Table \ref{tb:mnist_time} to compare the computational speed of NA-EB against other variational BNN methods. All these evaluations, carried out using different base classifiers for each method, adhere to the same training setup, respectively, as previously described.  
From the results presented in Tables \ref{tb:mnist_testerror}, \ref{tb:mnist_nll}, and \ref{tb:mnist_time}, it is evident that regardless of the choice of the base classifier, NA-EB consistently outperforms other BNN methods in terms of predictive accuracy and uncertainty quantification. Additionally, NA-EB ranks the second place among the BNN methods in terms of computation time as it requires more parameters than VBNN to characterize the implicit prior distribution. However, it's noteworthy that when LeNet-5 is used as the base classifier, VBNN tends to converge to a local minimum across trials, considering all hyperparameter options for the mixing coefficient $\pi$ and variances $\sigma_1^2$, $\sigma_2^2$ listed in Section 5.1 of \cite{blundell2015weight}, as shown in Figure \ref{fig:mnist_lenet5}. On the contrary, NA-EB achieves the lowest test error with reasonable computational cost. Moreover, as can be seen from Figure \ref{fig:mnist_mlp}, NA-EB converges the fastest and eventually obtains the lowest test error after 600 epochs when the MLP is employed as the base classifier. SVBNN and VBNN converge to larger test errors at similar rates.

\begin{figure}[t]
    \centering
    \includegraphics[scale=0.5]{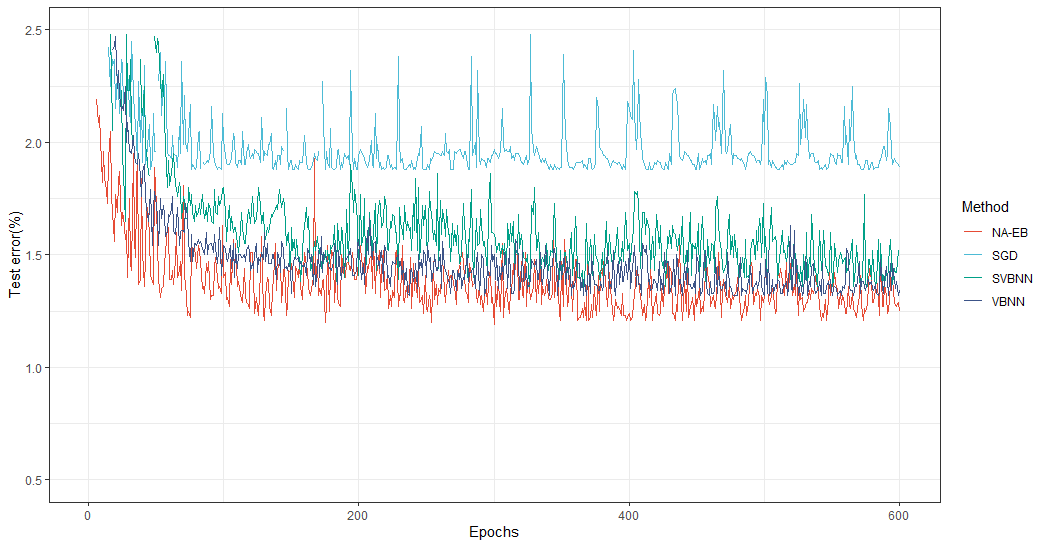}\par
    \caption{Test error on MNIST based on the 1200-1200 MLP classifier as training progresses: NA-EB, SGD, SVBNN and VBNN.}
    \label{fig:mnist_mlp}
\end{figure}

\begin{figure}[t]
    \centering
    \includegraphics[scale=0.5]{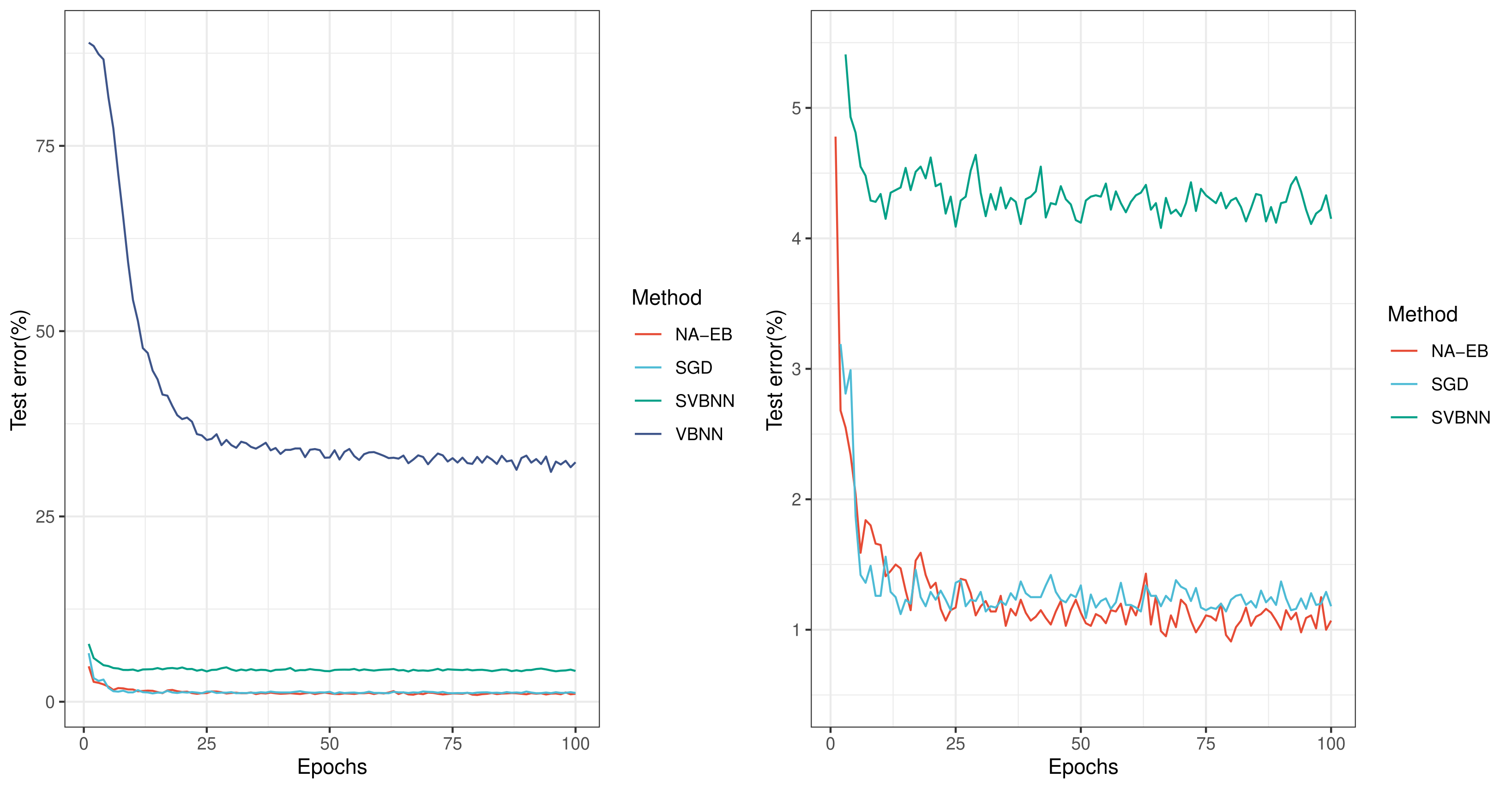}
    \caption{Left: Test error on MNIST based on the LeNet-5 classifier as training progresses: NA-EB, SGD, SVBNN and VBNN. Right: An enlarged figure of the left panel for NA-EB, SGD and SVBNN.}
    \label{fig:mnist_lenet5}
\end{figure}

We further investigate the effectiveness of our approach in scenarios with weaker data signals. Specifically, we assess the performance of our method using three artificial datasets collectively referred to as n-MNIST (noisy MNIST) \citep{basu2017learning}. These datasets are created by introducing (1) additive white Gaussian noise (AWGN), (2) motion blur, and (3) a combination of AWGN and reduced contrast to the MNIST dataset.
In the n-MNIST dataset with AWGN, we employ Gaussian noise with a signal-to-noise ratio of 9.5, simulating significant background clutter. For the n-MNIST dataset with motion blur, we apply a motion blur filter to emulate the linear motion of a camera by 5 pixels at an angle of 15 degrees counterclockwise. In the n-MNIST dataset with reduced contrast and AWGN, we scale down the contrast range to half and apply AWGN with a signal-to-noise ratio of 12. This emulates background clutter along with a significant change in lighting conditions.
We repeat the training protocol for MNIST classification tasks and summarize the test performance of our method, along with SGD \citep{simard2003best}, VBNN \citep{blundell2015weight}, and SVBNN \citep{bai2020efficient}. All these methods are based on a 400-400 MLP base classifier. The results for the three noisy MNIST datasets are presented in Table \ref{tb:mnist_noisy}. It is evident that our method improves the test accuracy and achieves the lowest test NLL regardless of types of noise. 
Additionally, we present four samples from the noisy MNIST dataset with motion blur in Figure \ref{fig:noisy_mnist} to highlight the significance of predicting uncertainty. In the left two panels, the prediction intervals for the corresponding true labels are narrow and centered around 100\%, demonstrating that NA-EB is highly confident about its predictions. Conversely, NA-EB provides relatively wide prediction intervals for the right two figures, particularly in the rightmost panel. Specifically, the top two most likely labels assigned by NA-EB have overlapping and broad prediction intervals, signifying uncertainty in its predictions. In contrast, the standard DNN method assigns an overly confident high prediction probability to an incorrect label in the right two panels.

\begin{table}
    \caption{Test error rates (in \%) and NLL of classification tasks on different noisy MNIST datasets. Boldface indicates the method with the smallest test error or the smallest NLL.}
    \centering
    \label{tb:mnist_noisy}
    \resizebox{\textwidth}{!}{
    \begin{tabular}{lcccccccc}
 \hline \\[-2.2ex]
\multirow{2}{*}{Dataset}&&SGD&\multicolumn{2}{c}{VBNN}&\multicolumn{2}{c}{SVBNN}&\multicolumn{2}{c}{NA-EB} \\
\cline{3-9}\\[-2.2ex]
&&Test Error& Test Error & Test NLL& Test Error & Test NLL& Test Error & Test NLL\\
\hline\\[-2.2ex]
n-MNIST with AWGN && 4.98 & 4.79 & 0.17 & 5.87 & 0.29 & {\bf 4.75} & {\bf 0.13}\\
n-MNIST with Motion Blur && 1.99 & 1.96& 0.11 &1.65& 0.28 & {\bf 1.47} & {\bf 0.07}\\
n-MNIST with reduced contrast and AWGN  && 8.06 & 7.23 & 0.23 & 9.84 & 1.27 & {\bf 7.15} & {\bf 0.18} \\\hline
\end{tabular}
    }
\end{table}

\subsection{CIFAR-10 dataset}

We evaluate the performance of NA-EB on the CIFAR-10 dataset. This dataset comprises 50,000 training images and 10,000 test images, each sized at 32 $\times$ 32 pixels and containing 3 color channels, evenly distributed across 10 distinct classes. We employ ResNet-18 \citep{he2016deep}, a larger CNN architecture, as the base classifier and optimize the loss function for 200 epochs using the default SGD optimizer with a learning rate of 0.01 and momentum of 0.9. We examine the predictive accuracy and uncertainty quantification of NA-EB in terms of test error and test NLL, respectively. We compare NA-EB with other BNN methods, as reported by \cite{han2022card} in Table \ref{tb:cifar10}. Specifically, CMV-MF-VI, CM-MF-VI and CV-MF-VI are variants of a recent BNN work \citep{tomczak2021collapsed} that employs a novel and tighter ELBO to conduct variational inference in BNNs. Our observations reveal that NA-EB achieves enhanced test accuracy, primarily due to its highly flexible implicit prior. Furthermore, our NLL result, while not the primary objective of NA-EB, is competitive with some of the best performing methods in this regard. 

To gain a more tangible grasp of the significance of predictive uncertainty, we juxtapose the predicted probabilities obtained from the ResNet-18 classifier with our 95\% Bayesian credible intervals for four test images sourced from the CIFAR-10 dataset. In the case of the deer image, the standard DNN method assigns a high probability to the wrong class without considering uncertainty. Conversely, NA-EB yields prediction intervals of similar width and overlapping for the two potential classes, suggesting uncertainty in its predictions. Moreover, for the frog image, NA-EB provides a relatively narrow prediction interval for the correct class, while the DNN method assigns a high probability estimate to the wrong class. This highlights that NA-EB not only quantifies predictive uncertainty but also improves classification accuracy by implicitly specifying a correct prior for classifier weights. From the examples of the noisy MNIST dataset and CIFAR-10 dataset, it is evident that the widths of the prediction intervals indicate the level of certainty or uncertainty NA-EB holds regarding the accuracy of its predictions, whereas the point estimate of likelihood produced by the standard DNN method lacks uncertainty information. 
It is reasonable to conclude that the superior performance of NA-EB, both in terms of predictive accuracy and uncertainty quantification, can be extended to other classification tasks within the medical domain, such as automated diagnostics on medical images, where uncertainty estimation is of particular importance.
    
\begin{table}
    \caption{Test error (in \%) and NLL of CIFAR-10 classification tasks. Boldface indicates the method with the smallest test error or the smallest NLL.}
    \centering
    \label{tb:cifar10}
    \resizebox{\textwidth}{!}{
    \begin{tabular}{lccccccccc}
 \hline \\[-2.2ex]
Metrics && CMV-MF-VI & CM-MF-VI & CV-MF-VI & MF-VI & MC Dropout & MAP & CARD & NA-EB \\
\hline\\[-2.2ex]
Test Error && 13.75 & 13.34 & 20.22 & 22.92 & 16.36 & 15.31 & 9.07 & {\bf 8.40} \\
Test NLL && 0.41 & {\bf 0.39} & 0.59 & 0.68 & 0.49 &  0.93 & 0.46 & 0.49\\
\hline
\end{tabular}  
    }
    \end{table}

\section{Conclusion}\label{sec:conclusion}

In this paper, we propose a general implicit generative prior for Bayesian inference. In particular, we develop the NA-EB framework to address the unsolved challenges in Bayesian DNNs. Meaningful priors for neural network weights are defined implicitly through a deep neural network transformation of a low-dimensional known distribution. The posterior for complex models with a large data volume can be efficiently computed using the proposed stochastic gradient method. We rigorously analyze the theoretical property of the algorithm and demonstrate that NA-EB outperforms many existing methods in a variety of numerical examples. Furthermore, NA-EB also takes advantage of recent developments in computational techniques. Our programming code for NA-EB is implemented using the PyTorch machine learning framework \citep{paszke2019pytorch}, which allows us to construct complex neural networks and compute the gradient of functions using the automatic differentiation technique. Additionally, the NA-EB code can be easily run on modern computational hardware, such as graphics processing units, significantly reducing the computing time for large-scale datasets.

NA-EB has adopted the empirical Bayes framework to estimate the hyperparameter $\etabf$, which brings an additional unknown quantity into the model, especially for very large models. We will explore techniques to reduce the unknown-quantity count in NA-EB while maintaining its advantages in uncertainty quantification. This could involve investigating model compression techniques or exploring alternative network architectures.

\section{Acknowledgments}

The authors would like to thank the anonymous referees, an Associate
Editor and the Editor for their constructive comments that improved the
quality of this paper. The second author was supported in part by NSF Grant SES-2316428.

\newpage 

\appendix

\begin{center}
{\Large \bf   Appendix}
\end{center}

\section{Algorithm Details}

We provide more details for Algorithm \ref{algorithm}.  First, all variational parameters $m_1, \ldots,m_r$, $\rho_1,\ldots,\rho_r$ are initialized by i.i.d. uniform $Unif[-1, 1]$. All components of $\etabf$ in $G_{\etabf}$ are initialized by a uniform distribution following the default initialization in PyTorch 1.9.0. Then, in each iteration, we compute the Monte Carlo estimates of the gradients in \eqref{eq:grad} by generating $H$ samples from the current variational posterior distribution. In terms of calculating ${\nabla_{\rho^{(t)}_{j}}\mathcal{L}}(\alphabf_t,\etabf_t)$, for $j=1,\ldots,r$, we apply the chain rule, $\nabla_{\rho_j}\log \{q_{\alphabf}(\zbf)\}=[\nabla_{\varrho_j}\log \{q_{\alphabf}(\zbf)\}|_{\varrho_j=\log (1+e^{\rho_j})}]e^{\rho_j}(1+e^{\rho_j})^{-1}$, where the term in the first curly brackets is the derivative of ${\log}[q_{\alpha}(z)]$ with respect to $\varrho_j$ evaluated at the point $\varrho_j={\log}(1+e^{\rho_j})$ and the rest is the derivative of $\varrho_j$ with respect to $\rho_j$. For the next step, we update the estimates of next iteration adopting the gradient descent method whereas the chosen learning rate sequence $\{\beta^{(t)}\}$ follows the Robbins--Monro conditions for convergence \citep{ranganath2014black}. The stopping criteria for Algorithm 1 is set to be a maximum iteration step.

We further investigate an additional improvement on the above algorithm. In the era of big data, our algorithm can also adopt minibatch optimization. The training data $\mathcal{D}_n$ is randomly split into a partition of $B$ subsets $\mathcal{D}_1, \ldots, \mathcal{D}_B$, and each gradient is averaged over one subset of the data iteratively \citep{graves2011practical}. For example, one can partition $\mathcal{D}_n$ based on $\tau=(\tau_1, \ldots, \tau_B)\in[0, 1]^B$ with $\sum_{b=1}^B\tau_b=1$. The objective function in \eqref{eq:loss} is changed to $$\mathcal{L}_b(\alphabf,\etabf) =-\tau_b\textsc{KL}(q_{\alphabf}~\|~\pi_0) +\mathbb{E}_{\zbf\sim q_{\alphabf}(\zbf)}[{\log} \{L(\mathcal{D}_b; {G_{\etabf}(\zbf)})\}].$$ This is reasonable since $\mathbb{E}_B[\sum_{b=1}^B \mathcal{L}_b(\alphabf,\etabf)] = \mathcal{L}(\alphabf,\etabf)$, where the expectation is taken over the random partitioning of minibatches. In particular, we adopt the scheme $\tau_b=2^{B-b}/(2^B-1)$ in our algorithm which puts more weights on the first few minibatches as little data are seen.

As we obtain the estimated DNN transformation function $G_{\hat{\etabf}}$ and the variational posterior distribution $q_{\hat{\alphabf}}(\zbf)$ from Algorithm 1, the predictive distribution can be obtained by multiple forward passes on the network while sampling from the weight posteriors. Given a new input $\xbf^{*}$, the predictive distribution $p(y^{*} \mid \xbf^{*},\mathcal{D}_n)$ can be estimated by 
\begin{equation}\label{eq:pred_dist}
    \hat p(y^{*} \mid \xbf^{*},\mathcal{D}_n)=\frac{1}{M}\sum_{m=1}^M p(y^{*} \mid \xbf^{*}, {G_{\hat{\etabf}}(\tilde{\zbf}_{m})})
\end{equation}
with $\tilde{\zbf}_{m}$ sampled from $ q_{\hat{\alphabf}}(\zbf)$ independently, for $m=1, \ldots, M$, and $M$ being the number of Monte Carlo samples.

\section{Instructions for utilizing code files}
We provide source code of numerical experiments in Section \ref{sec:numerical} at \url{https://github.com/yjliu7/Neural-Adaptive-Empirical-Bayes}. The implementation of NA-EB relies on the PyTorch deep learning framework. In terms of data, we use simulated data in \textit{two\_spiral.py} for the two-spiral problem and utilize the MNIST dataset provided by the \textit{torchvision} package. The ten UCI datasets can be downloaded from \url{https://archive.ics.uci.edu/datasets}. Regarding the model, we specify the deterministic transformation function $G_{\etabf}$ by a fully connected feedforward neural network in \textit{deterministic\_transformation.py}. We provide a Gaussian variational sampler for the latent variable $\zbf$ in \textit{gaussian\_variational.py}. The base classifier $f_{\wbf}$ and the loss function $\mathcal{L}(\alphabf, \etabf)$ are given in \textit{Bayes.py}. Furthermore, we include detailed explanations and comments within each Python file for reference.

\section{Assumptions for Theoretical Analysis}

To take advantage of the representation power of the DNN, we
approximate $f_0$ by a DNN $f_{\wbf}$, defined in the
following way:
\begin{align}
 & f_{\wbf}=A_{d}\circ\sigma\circ A_{d-1}\circ\cdots\circ\sigma\circ A_{1},\label{eq:dnn_formula}\\
 & A_{1}\in\mathcal{A}_{p}^{N},\ A_{2},\ldots,A_{d-1}\in\mathcal{A}_{N}^{N},\ A_{d}\in\mathcal{A}_{N}^{1},\label{eq:affine_transformation}\\
 & \sigma((x_{1},\ldots,x_{N})^{\T})=(\sigma(x_{1}),\ldots,\sigma(x_{N}))^{\T},\nonumber
\end{align}
where $N$ and $d$ are the width and depth of the DNN, respectively,
$\mathcal{A}_{n_{1}}^{n_{2}}$ stands for affine transformations $\mathbb{R}^{n_{1}}\mapsto\mathbb{R}^{n_{2}}$
of the form $\xbf_{n_{1}\times1}\mapsto\wbf_{n_{2}\times n_{1}}\xbf+\bbf_{n_{2}\times1}$,
$\sigma$ is the nonlinear activation function, and $\wbf\in {\mathscr W}\subseteq \mathbb R^D$ represents all the parameters in the DNN. In \eqref{eq:affine_transformation}
a constant width $N$ is used in each layer merely for simplicity of notations; in practice the dimensions of $A_{i}$ can vary from layer to layer. Other popular modes such as CNN \citep{krizhevsky2017imagenet} and ResNet \citep{he2016deep} are special cases of \eqref{eq:dnn_formula}. DNN models allow for great freedom in selecting the activation function. Popular choices include the rectified linear unit \citep{glorot2011deep}, $\sigma(x)=\max\{x,0\}$, and its smoothed version $\sigma(x)=\log \{1+\exp(x)\}\approx\max\{x,0\}$. 

Due to the universal approximation property of DNNs \citep{hornik89}, for any $\epsilon>0$, there exists a DNN $f_{\Upsilonbf}$ of the form \eqref{eq:dnn_formula} with weights $\Upsilonbf\in \mathscr{W}\subset \mathbb R^{D_n}$ such that $\|f_0-f_{\Upsilonbf}\|_{\infty}\leq\varepsilon/4$, where $\|\cdot\|_{\infty}$ is the $L_{\infty}$ norm of a function defined as $\|f\|_{\infty}=\sup_{\xbf\in\mathcal{X}}|f(\xbf)|$. In addition, there exists an $\sbf=(s_1,\ldots,s_{r_n})^{\T}\in\mathscr{Z}\subset \mathbb R^{r_n}$ such that $\Upsilonbf=G_{\etabf}(\sbf)$. A neighborhood of $\sbf$ is defined as 
\begin{equation}\label{eq:p_neighbour}
    \mathcal{P}_{\varepsilon}=\left\{\zbf: |z_j-s_{j}|<\frac{\sqrt{\varepsilon}}{8\sqrt{r_n n^{1-u}}\{D_n+(p+1)\|\Upsilonbf\|_1\}}, j=1,\ldots,r_n\right\},
\end{equation}
where $\|\cdot\|_1$ is the $L_1$ norm of a vector defined as $\|\Upsilonbf\|_1 = \sum_{k=1}^{D_n}|\Upsilon_k|$ and $u$ is given in Theorems \ref{thm:thm4.1} and \ref{thm:thm4.2} measuring the order of magnitude of $D_n$. Let $\{\mathcal{F}_n\}$ denote a sequence of sieves defined as \begin{equation}\label{eq:f_sieves}
    \mathcal{F}_n=\{\zbf:|z_j|\leq \Tilde{C}_n, j=1,\ldots,r_n\},
\end{equation}
where $\Tilde{C}_n$ is a parameter related to $n$. In the following, we write $G_{\etabf}=(G_{\etabf 1},G_{\etabf 2},\ldots,G_{\etabf D_n})^{\T}$ where $G_{\etabf k}: \mathscr{Z}\mapsto\mathbb{R}$ is the deterministic transformation function such that $w_k=G_{\etabf k}(\zbf)$, for $k=1,\ldots,D_n$. 

In Theorem \ref{thm:thm4.1}, we establish that under some conditions, the variational posterior concentrates in $\varepsilon$-small Hellinger neighborhoods of the true density $\ell_0$. The conditions on the deterministic transformation $G_{\etabf}$ and prior parameters are summarized below. 

\begin{assumption}\label{assump:assump1a}
The deterministic function $G_{\etabf}$ is twice differentiable at $\zbf\in\mathscr{Z}$.
For $\zbf\in\mathcal{F}_n\cup\mathcal{P}_{\varepsilon}$ with $\Tilde{C}_n = e^{n^b/r_n}$ where $b$ is a constant, the first-order derivative satisfies 
$\|\partial G_{\etabf}/\partial \zbf\|_F^2
=o(\varepsilon n^{1-u})$ where $u$ is given in Theorem \ref{thm:thm4.1} and $\|\cdot\|_F$ denotes the Frobenius norm of a matrix $\mathbf{A}$ defined as $\|\mathbf{A}\|_F=(\sum_i\sum_j|a_{ij}|^2)^{1/2}$. The second-order derivative satisfies 
$\{\sum_{k=1}^{D_n}(\partial^2 G_{\etabf k}/\partial z_j^2)^2\}^{1/2}=o(\sqrt{D_n}\sum_{k=1}^{D_n}(\partial G_{\etabf k}/\partial z_j)^2)$ at $\sbf$, for $j=1,\ldots,r_n$.
\end{assumption}

\begin{assumption}\label{assump:assump1b}
The prior parameters in \eqref{eq:prior} satisfy the assumption \eqref{eq:var} and $\|\mubf\|_2^2=o(n)$, where $\|\cdot\|_2$ is the $L_2$ norm of a vector defined as $\|\mubf\|_2 = (\sum_{j=1}^{r}\mu_j^2)^{1/2}$.
\end{assumption}

Assumption \ref{assump:assump1a} puts some smoothness constraints on the Jacobian of $G_{\etabf}$. The intuition is to improve the stability of the model, which avoids the situation where infinitesimal perturbations amplify and have substantial impacts on the performance of the output of $G_{\etabf}$. In practice, a Jacobian regularization can be added to the objective function, and a computationally efficient algorithm has been implemented by \cite{hoffman2019robust}. 
Furthermore, the second-order derivative condition in Assumption \ref{assump:assump1a} is mild, since it only requires that the square root of the mean square of the second-order derivative be insignificant compared to the sum of the squared first-order derivative at a fixed point $\sbf$. As long as a substantial fraction of $\partial G_{\etabf k}/\partial z_j$ among $k=1,\ldots,D_n$ evaluated at this fixed point is nonzero, the condition is satisfied.
Assumption \ref{assump:assump1b} imposes restrictions on the prior parameters so that the KL distance between the variational posterior $q^*$ and the true posterior $p(\wbf\mid\mathcal{D}_n)$ is negligible.

In Theorem \ref{thm:thm4.2}, we establish the consistency of variational posterior for shrinking neighborhood sizes of the true density $\ell_0$. However, since Theorem \ref{thm:thm4.2} is more restrictive in nature than Theorem \ref{thm:thm4.1}, it requires additional assumptions on the approximation of the neural network solution to the true function $f_0$. Therefore, we modify Assumptions \ref{assump:assump1a} and \ref{assump:assump1b} as follows.

\begin{assumption}\label{assump:assump2a}
    The deterministic function $G_{\etabf}$ is twice differentiable at $\zbf\in\mathscr{Z}$. For $\zbf\in\mathcal{F}_n\cup\mathcal{P}_{\varepsilon\epsilon_n^2}$ with $\Tilde{C}_n = e^{n^b\epsilon_n^2/r_n}$ where $b$ is a constant, the first-order derivative satisfies 
$\|\partial G_{\etabf}/\partial \zbf\|_F^2
=o(\varepsilon\epsilon_n^2 n^{1-u})$ where $u$ is given in Theorem \ref{thm:thm4.2}. The second-order derivative satisfies 
$\{\sum_{k=1}^{D_n}(\partial^2 G_{\etabf k}/\partial z_j^2)^2\}^{1/2}=o(\sqrt{D_n}\sum_{k=1}^{D_n}(\partial G_{\etabf k}/\partial z_j)^2)$ at $\sbf$, for $j=1,\ldots,r_n$.
\end{assumption}

\begin{assumption}\label{assump:assump2b}
    The prior parameters in \eqref{eq:prior} satisfy the assumption \eqref{eq:var} and $\|\mubf\|_2^2=o(n\epsilon_n^2)$.
\end{assumption}

\begin{assumption}\label{assump:assump2c}
    There exists a sequence of neural network functions $f_{\Upsilonbf}$ of the form \eqref{eq:dnn_formula} with $\Upsilonbf=G_{\etabf}(\sbf)$ satisfying $\|f_0-f_{\Upsilonbf}\|_{\infty}=o(\epsilon_n^2)$, $\|\Upsilonbf\|_2^2=o(n\epsilon_n^2)$, $\|\sbf\|_2^2=o(n\epsilon_n^2)$. 
\end{assumption}

Compared to Assumptions \ref{assump:assump1a} and \ref{assump:assump1b}, the square of the Frobenius norm of the Jacobian in Assumption \ref{assump:assump2a} and the rate of growth of $L_2$ norm of the prior mean parameter in Assumption \ref{assump:assump2b} are allowed to grow slower since the consistency of the variational posterior in a shrinking Hellinger neighborhood of $\ell_0$ is more restrictive in nature. Furthermore, Assumption \ref{assump:assump2c} requires the existence of a neural network solution that converges to the true function $\ell_0$ at a sufficiently fast rate while ensuring controlled growth of the $L_2$ norm of its coefficients. 

\section{Overview of The Proof}

We focus mainly on the proof for the binary classification in \eqref{eq:binary}. Without loss of generality, we consider $f_{\wbf}$ in \eqref{eq:dnn_formula} as a single layer neural network: $f_{\wbf}(\xbf) = \sum^{d}_{j=1}a_j\sigma({\omegabf_j}^{\T}\xbf+b_j)$ where $\wbf=(\omegabf_1^{\T},\ldots,\omegabf_d^{\T},b_1,\ldots,b_d,a_1,\ldots,a_d)^{\T}$, $a_j,b_j\in\mathbb{R}$, $\omegabf_j\in\mathbb{R}^p$, $j=1,\ldots,d$.
We briefly outline the main steps in the proof of Theorems \ref{thm:thm4.1} and \ref{thm:thm4.2} and Corollaries \ref{cor:cor4.3} and \ref{cor:cor4.4}. We borrow a few steps and notations from \cite{bhattacharya2020variational}.

To establish the consistency of variational posterior, we use the following inequality as the foundation of the proof for Theorems \ref{thm:thm4.1} and \ref{thm:thm4.2} and its derivations are given in detail in the proof of Theorem \ref{thm:thm4.1} in the Appendix:
\begin{equation}\label{eq:theorem4.1_2}
    -q^{*}(\mathcal{V}^c_{\varepsilon})A\leq d_{\mathrm{KL}}(q^{*}, p_{\etabf}(\cdot\mid\mathcal{D}_n))+|B|+\log 2,
\end{equation}
where $\mathcal{V}_{\varepsilon}$ is defined in \eqref{eq:hellingerz} and 
\begin{equation}\label{eq:theorem4.1_1}
    \resizebox{0.8\linewidth}{!}{$ A=\log\left\{\int_{\mathcal{V}_{\varepsilon}^c} \frac{L(\mathcal{D}_n; f_{G_{\etabf}(\zbf)})}{L_0}\pi_0(\zbf)d\zbf\right\},\quad B=-\log\left\{\int\frac{L(\mathcal{D}_n; f_{G_{\etabf}(\zbf)})}{L_0}\pi_0(\zbf)d\zbf\right\}$}.
\end{equation}
Then we get the following main steps towards the proof:
\begin{itemize}
    \item[1.] We construct the upper bound of the first term $A$ by decomposing its exponential term as 
    $$
    e^A=\int_{\mathcal{V}_{\varepsilon}^c\cap\mathcal{F}_n} \{L(\mathcal{D}_n; f_{G_{\etabf}(\zbf)})/L_0\}\pi_0(\zbf)d\zbf+\int_{\mathcal{V}_{\varepsilon}^c\cap\mathcal{F}_n^c} \{L(\mathcal{D}_n; f_{G_{\etabf}(\zbf)})/L_0\}\pi_0(\zbf)d\zbf,$$ where $\{\mathcal{F}_n\}_{n=1}^{\infty}$ is a suitably chosen sequence of sieves (see the proof of Proposition \ref{proposition:prop7.15} in the Appendix), which gives the lower bound of $-q^{*}(\mathcal{V}^c_{\varepsilon})A$.
    \item[2.] We control the second quantity $B$ by the rate at which the prior $\pi_0$ gives mass to a shrinking KL neighborhood of the true density $\ell_0$ (see step 1 (c) of the proof of Proposition \ref{proposition:prop7.16} in the Appendix for Theorem \ref{thm:thm4.1} and step 2 (c) of the proof of Proposition \ref{proposition:prop7.16} in the Appendix for Theorem \ref{thm:thm4.2}).
    \item[3.] We construct the upper bound for $d_{\mathrm{KL}}(q^*, p_{\etabf}(\cdot\mid\mathcal{D}_n))$ by bounding $d_{\mathrm{KL}}(q, p_{\etabf}(\cdot\mid\mathcal{D}_n))$ below by $$d_{\mathrm{KL}}(q,\pi_0)+\left|\int\log\{L(\mathcal{D}_n; f_{G_{\etabf}(\zbf)})/L_0\}q(\zbf)d\zbf\right|+\left|\log\left[\int \{L(\mathcal{D}_n; f_{G_{\etabf}(\zbf)})/L_0\}\pi_0(\zbf)d\zbf\right]\right|$$ for a suitable $q\in\mathcal{Q}$ (see the proof of part 1 of Proposition \ref{proposition:prop7.16} in the Appendix for Theorem \ref{thm:thm4.1} and the proof of part 2 of Proposition \ref{proposition:prop7.16} in the Appendix for Theorem \ref{thm:thm4.2}).
    \item[4.] Based on the relation \eqref{eq:theorem4.1_2} and the results from steps 1, 2 and 3, we can control $q^{*}(\mathcal{V}^c_{\varepsilon})$, which is equal to $\pi^{*}(\mathcal{U}^c_{\varepsilon})$.
\end{itemize}
In addition, in terms of the proof for Corollaries \ref{cor:cor4.3} and \ref{cor:cor4.4}, we first derive that the difference in classification accuracy $R(\hat{C})-R(C^{\tt{Bayes}})$ is bounded above by the logit links $\hat{f}(\xbf)$ and $f_0(\xbf)$. We further bound the logit links by $d_{\mathrm{H}}(\hat{\ell},\ell_0)$. Finally, we establish that $R(\hat{C})-R(C^{\tt{Bayes}})$ is bounded by a constant with probability tending to 1 as the sample size increases to infinity.

\section{Theoretical Properties of Variational Posterior}

\subsection{Proof of Theorem \ref{thm:thm4.1}}

First, we define the Hellinger neighborhood of the true function density function $g_0=\ell_0$ in terms of $\zbf$ as
\begin{equation}\label{eq:hellingerz}
    \mathcal{V}_{\varepsilon}=\{\zbf: G_{\etabf}(\zbf)\in\mathcal{U}_{\varepsilon}\}.
\end{equation}
In view of \eqref{eq:hellinger_w} and \eqref{eq:hellingerz}, we notice that
 \begin{equation*}
     \pi^{*}(\mathcal{U}^c_{\varepsilon})=q^{*}(\mathcal{V}^c_{\varepsilon}).
 \end{equation*}
Therefore, we now turn to prove the posterior consistency of $q^{*}(\zbf)$.
\\Next, we construct the main inequality \eqref{eq:theorem4.1_2} related to $q^{*}(\zbf)$ as follows:
\begin{equation*}
    \begin{aligned}
         d_{\mathrm{KL}}(q^{*}, p_{\etabf}(\cdot\mid\mathcal{D}_n))&=\int_{\mathcal{V}_{\varepsilon}}q^{*}(\zbf)\log\left\{\frac{q^{*}(\zbf)}{p_{\etabf}(\zbf\mid\mathcal{D}_n)}\right\}d\zbf+\int_{\mathcal{V}^c_{\varepsilon}}q^{*}(\zbf)\log\left\{\frac{q^{*}(\zbf)}{p_{\etabf}(\zbf\mid\mathcal{D}_n)}\right\}d\zbf\\
         &=-q^{*}(\mathcal{V}_{\varepsilon})\int_{\mathcal{V}_{\varepsilon}}\frac{q^{*}(\zbf)}{q^{*}(\mathcal{V}_{\varepsilon})}\log\left\{\frac{p_{\etabf}(\zbf\mid\mathcal{D}_n)}{q^{*}(\zbf)}\right\}d\zbf\\
         &-q^{*}(\mathcal{V}^c_{\varepsilon})\int_{\mathcal{V}^c_{\varepsilon}}\frac{q^{*}(\zbf)}{q^{*}(\mathcal{V}^c_{\varepsilon})}\log\left\{\frac{p_{\etabf}(\zbf\mid\mathcal{D}_n)}{q^{*}(\zbf)}\right\}d\zbf\\
         &\geq q^{*}(\mathcal{V}_{\varepsilon})\log\left\{\frac{q^{*}(\mathcal{V}_{\varepsilon})}{p_{\etabf}(\mathcal{V}_{\varepsilon}\mid\mathcal{D}_n)}\right\}+ q^{*}(\mathcal{V}^c_{\varepsilon})\log\left\{\frac{q^{*}(\mathcal{V}^c_{\varepsilon})}{p_{\etabf}(\mathcal{V}^c_{\varepsilon}\mid\mathcal{D}_n)}\right\}\\
         &\geq q^{*}(\mathcal{V}_{\varepsilon})\log \{q^{*}(\mathcal{V}_{\varepsilon})\}+q^{*}(\mathcal{V}^c_{\varepsilon})\log \{q^{*}(\mathcal{V}^c_{\varepsilon})\}-q^{*}(\mathcal{V}^c_{\varepsilon})\log \{p_{\etabf}(\mathcal{V}^c_{\varepsilon}\mid\mathcal{D}_n)\}\\
         &\geq -q^{*}(\mathcal{V}^c_{\varepsilon})\log \{p_{\etabf}(\mathcal{V}^c_{\varepsilon}\mid\mathcal{D}_n)\} -\log 2\\
         &=-q^{*}(\mathcal{V}^c_{\varepsilon})\log\left\{\int_{\mathcal{V}_{\varepsilon}^c} \frac{L(\mathcal{D}_n; f_{G_{\etabf}(\zbf)})}{L_0}\pi_0(\zbf)d\zbf\right\}\\
         &+q^{*}(\mathcal{V}^c_{\varepsilon})\log\left\{\int\frac{L(\mathcal{D}_n; f_{G_{\etabf}(\zbf)})}{L_0}\pi_0(\zbf)d\zbf\right\}-\log 2,
    \end{aligned}
\end{equation*}
where the third inequality holds by Jensen's inequality, the fourth step follows since $p_{\etabf}(\mathcal{V}_{\varepsilon}\mid\mathcal{D}_n)\leq 1$ and the fifth step follows since $x\log (x)+(1-x)\log (1-x)\geq -\log 2$. In the above proof we have assumed $q^{*}(\mathcal{V}_{\varepsilon})>0$, $q^{*}(\mathcal{V}^c_{\varepsilon})>0$. If $q^{*}(\mathcal{V}^c_{\varepsilon})=0$, there is nothing to prove. If $q^{*}(\mathcal{V}_{\varepsilon})=0$, then we will get $\varepsilon^2=o_{P_0}(1)$ which is a contradiction.
\\By Assumption \ref{assump:assump1b}, the prior parameters satisfy
\begin{equation*}
    \|\mubf\|_2^2=o(n),\quad\|\zetabf\|_{
        \infty}=O(n),\quad\|\zetabf^*\|_{
        \infty}=O(1),\quad\zetabf^*=1/\zetabf.
\end{equation*}
Note $r_n\sim n^a$, $0<a<1$ and $D_n\sim d_n\sim n^u$, $0<u<1$ which implies $r_n\log (n)=o(n)$, $D_n\log (n)=o(n)$.
\\By part 1 of Proposition \ref{proposition:prop7.16},
\begin{equation}\label{eq:theorem4.1_3}
    d_{\mathrm{KL}}(q^{*}, p_{\etabf}(\cdot\mid\mathcal{D}_n))=o_{P_0}(n).
\end{equation}
By step 1 (c) of the proof of Proposition \ref{proposition:prop7.16},
\begin{equation}\label{eq:theorem4.1_4}
    B=o_{P_0}(n).
\end{equation}
Since $r_n\sim n ^a$, $r_n\log (n)=o(n^b)$, $a<b<1$ and $D_n\sim n ^u$, $D_n\log (n)=o(n^v)$, $u<v<1$, using Proposition \ref{proposition:prop7.15} with $\epsilon_n=1$, we have
\begin{equation}\label{eq:theorem4.1_5}
    -q^{*}(\mathcal{V}^c_{\varepsilon})A\geq n\varepsilon^2q^{*}(\mathcal{V}^c_{\varepsilon})-\log 2+o_{P_0}(1)=n\varepsilon^2q^{*}(\mathcal{V}^c_{\varepsilon})+O_{P_0}(1).
\end{equation}
Therefore, using \eqref{eq:theorem4.1_3}, \eqref{eq:theorem4.1_4} and \eqref{eq:theorem4.1_5} in \eqref{eq:theorem4.1_2}, we obtain
\begin{equation*}
    n\varepsilon^2q^{*}(\mathcal{V}^c_{\varepsilon})+O_{P_0}(1)\leq o_{P_0}(n)+o_{P_0}(n)\Rightarrow \pi^{*}(\mathcal{U}^c_{\varepsilon})=q^{*}(\mathcal{V}^c_{\varepsilon})=o_{P_0}(1).
\end{equation*}

\subsection{Proof of Theorem \ref{thm:thm4.2}}

We assume the relation \eqref{eq:theorem4.1_2} holds with $A$ and $B$ same as in \eqref{eq:theorem4.1_1}. We also turn to prove the posterior consistency of $q^{*}(\zbf)$, as $\pi^{*}(\mathcal{U}^c_{\varepsilon\epsilon_n})=q^{*}(\mathcal{V}^c_{\varepsilon\epsilon_n})$ where $\mathcal{V}_{\varepsilon\epsilon_n}$ is defined in \eqref{eq:hellingerz}.
\\Note $r_n\sim n^a$, $0<a<1$, $D_n\sim n^u$, $0<u<1$ and $\epsilon_n^2\sim n^{-\delta}$, $0<\delta<1-u<1-a$. This implies $r_n\log (n)=o(n\epsilon_n^2)$, $D_n\log (n)=o(n\epsilon_n^2)$.
\\By Assumption \ref{assump:assump2b}, the prior parameters satisfy
\begin{equation*}
    \|\mubf\|_2^2=o(n\epsilon_n^2),\quad\|\zetabf\|_{
        \infty}=O(n),\quad\|\zetabf^*\|_{
        \infty}=O(1),\quad\zetabf^*=1/\zetabf.
\end{equation*}
In addition, by Assumption \ref{assump:assump2c},
\begin{equation*}
     \|f_0-f_{\Upsilonbf}\|_{\infty}=o(\epsilon_n^2),\quad
     \|\Upsilonbf\|_2^2=o(n\epsilon_n^2),\quad
     \|\sbf\|_2^2=o(n\epsilon_n^2).
\end{equation*}
By part 2 of Proposition \ref{proposition:prop7.16},
\begin{equation}\label{eq:theorem4.2_1}
    d_{\mathrm{KL}}(q^{*}, p_{\etabf}(\cdot\mid\mathcal{D}_n))=o_{P_0}(n\epsilon_n^2).
\end{equation}
By step 2 (c) of the proof of Proposition \ref{proposition:prop7.16},
\begin{equation}\label{eq:theorem4.2_2}
    B=o_{P_0}(n\epsilon_n^2).
\end{equation}
Since $r_n\sim n ^a$, $r_n\log (n)=o(n^b\epsilon_n^2)$, $a+\delta<b<1$ and $D_n\sim n ^u$, $D_n\log (n)=o(n^v\epsilon_n^2)$, $u+\delta<v<1$, it follows, by using Proposition \ref{proposition:prop7.15} that
\begin{equation}\label{eq:theorem4.2_3}
    -q^{*}(\mathcal{V}^c_{\varepsilon\epsilon_n})A\geq n\varepsilon^2\epsilon_n^2 q^{*}(\mathcal{V}^c_{\varepsilon\epsilon_n})-\log 2+o_{P_0}(1)=n\varepsilon^2\epsilon_n^2q^{*}(\mathcal{V}^c_{\varepsilon\epsilon_n})+O_{P_0}(1).
\end{equation}
Thus, using \eqref{eq:theorem4.2_1}, \eqref{eq:theorem4.2_2} and \eqref{eq:theorem4.2_3} in \eqref{eq:theorem4.1_2}, we get
\begin{equation*}
    n\varepsilon^2\epsilon_n^2q^{*}(\mathcal{V}^c_{\varepsilon\epsilon_n})+O_{P_0}(1)\leq o_{P_0}(n\epsilon_n^2)+o_{P_0}(n\epsilon_n^2)\Rightarrow \pi^{*}(\mathcal{U}^c_{\varepsilon\epsilon_n})=q^{*}(\mathcal{V}^c_{\varepsilon\epsilon_n})=o_{P_0}(1).
\end{equation*}

\subsection{Proof of Corollary \ref{cor:cor4.3}}

Note that
\begin{equation}\label{eq:theorem7.17_7}
    \begin{aligned}
         R(\hat{C})-R(C^{\tt{Bayes}})&=E_{\xbf}\left[E_{y\mid \xbf}\left\{\mathds{1}(\hat{C}(\xbf)\neq y)-\mathds{1}(C^{\tt{Bayes}}(\xbf)\neq y)\right\}\right]\\
         &=E_{x}\left(E_{y\mid \xbf}\left[\left\{\mathds{1}(\hat{C}(\xbf)=0)-\mathds{1}(C^{\tt{Bayes}}(\xbf)=0)\right\}\sigma(f_0(\xbf))\right]\right)\\
         &+E_{\xbf}\left(E_{y\mid \xbf}\left[\left\{\mathds{1}(\hat{C}(\xbf)=1)-\mathds{1}(C^{\tt{Bayes}}(\xbf)=1)\right\}\left\{1-\sigma(f_0(\xbf))\right\}\right]\right)\\
         &=2E_{\xbf}\left[\mathds{1}(\hat{C}(\xbf)\neq C^{\tt{Bayes}}(\xbf))\left|\sigma(f_0(\xbf))-\frac{1}{2}\right|\right]\\
         &=2E_{\xbf}\left[\mathds{1}(\sigma(\hat{f}(\xbf))\geq \frac{1}{2},\sigma(f_0(\xbf))<\frac{1}{2})\left|\sigma(f_0(\xbf))-\frac{1}{2}\right|\right]\\
         &+2E_{\xbf}\left[\mathds{1}(\sigma(\hat{f}(\xbf))<\frac{1}{2},\sigma(f_0(\xbf))\geq\frac{1}{2})\left|\sigma(f_0(\xbf))-\frac{1}{2}\right|\right]\\
         &\leq 2E_{\xbf}\left[\left|\sigma(f_0(\xbf))-\sigma(\hat{f}(\xbf))\right|\right].
    \end{aligned}
\end{equation}
Let \begin{equation}\label{eq:cor4.3_1}
    \hat{\ell}(y, \xbf)=\int\ell_{G_{\etabf}(\zbf)}(y,\xbf)q^*(\zbf)d\zbf.
\end{equation}
Then 
\begin{equation*}
    \begin{aligned}
         d_{\mathrm{H}}(\hat{\ell},\ell_0)&=d_{\mathrm{H}}\left(\int\ell_{G_{\etabf}(\zbf)}q^*(\zbf)d\zbf,\ell_0\right)\\
         &\leq \int d_{\mathrm{H}}(\ell_{G_{\etabf}(\zbf)},\ell_0)q^*(\zbf)d\zbf \quad\text{by}\; \text{Jensen's}\;\text{inequality}\\
         &=\int_{\mathcal{V}_{\varepsilon}}d_{\mathrm{H}}(\ell_{G_{\etabf}(\zbf)},\ell_0)q^*(\zbf)d\zbf+\int_{\mathcal{V}_{\varepsilon}^c}d_{\mathrm{H}}(\ell_{G_{\etabf}(\zbf)},\ell_0)q^*(\zbf)d\zbf\\
         &\leq\varepsilon+o_{P_0}(1),
    \end{aligned}
\end{equation*}
where the last inequality is a consequence of Theorem \ref{thm:thm4.1}.
\\Taking $\varepsilon\rightarrow0$, we get $d_{\mathrm{H}}(\hat{\ell},\ell_0)=o_{P_0}(1)$.
\\Note that $\hat{f}(\xbf)=\sigma^{-1}(\hat{\ell}(y,\xbf))=\log \{\hat{\ell}(0,\xbf)/\hat{\ell}(1,\xbf)\}$, then 
\begin{equation}\label{eq:theorem7.17_9}
    \begin{aligned}
         2d_{\mathrm{H}}(\hat{\ell},\ell_0)&=\int_{\xbf\in[0,1]^p}\sum_{y\in\{0,1\}}\left\{\sqrt{\hat{\ell}(y,\xbf)}-\sqrt{\ell_0(y,\xbf)}\right\}^2d\xbf\\
         &=2-2\int_{\xbf\in[0,1]^p}\sum_{y\in\{0,1\}}\sqrt{\hat{\ell}(y,\xbf)\ell_0(y,\xbf)}d\xbf\\
         &=2-2\int_{\xbf\in[0,1]^p}\sum_{y\in\{0,1\}}\exp\left[\frac{1}{2}\left\{y\hat{f}(\xbf)-\log (1+e^{\hat{f}(\xbf)})+yf_0(\xbf)-\log (1+e^{f_0(\xbf)})\right\}\right]d\xbf\\
         &=2-2\int_{\xbf\in[0,1]^p}\left[\sqrt{\sigma(\hat{f}(\xbf))\sigma(f_0(\xbf))}+\sqrt{\{1-\sigma(\hat{f}(\xbf))\}\{1-\sigma(f_0(\xbf))\}}\right]d\xbf\\
         &\geq 2-2\int_{\xbf\in[0,1]^p}\sqrt{1-\left\{\sqrt{\sigma(f_0(\xbf))}-\sqrt{\sigma(\hat{f}(\xbf))}\right\}^2}d\xbf\\
         &\geq \int_{\xbf\in[0,1]^p}\left\{\sqrt{\sigma(f_0(\xbf))}-\sqrt{\sigma(\hat{f}(\xbf))}\right\}^2d\xbf\\
         &\geq \frac{1}{4}\int_{\xbf\in[0,1]^p}\{\sigma(f_0(\xbf))-\sigma(\hat{f}(\xbf))\}^2d\xbf,
    \end{aligned}
\end{equation}
where the fifth step holds because
\begin{equation*}
    \begin{aligned}
         \left\{\sqrt{ab}+\sqrt{(1-a)(1-b)}\right\}^2&=2ab+1-a-b+2\sqrt{ab(1-a)(1-b)}\\
         &= 2\sqrt{ab}\{\sqrt{ab}+\sqrt{(1-a)(1-b)}\}+1-a-b\\
         &\leq 2\sqrt{ab}\{\frac{a+b}{2}+\frac{(1-a)+(1-b)}{2}\}+1-a-b\\
         &= 2\sqrt{ab}+1-a-b\\
         &=1-(\sqrt{a}-\sqrt{b})^2.
    \end{aligned}
\end{equation*}
The sixth and seventh steps hold because $\sqrt{1-x}<1-x/2$ and $|a-b|\leq|\sqrt{a}+\sqrt{b}||\sqrt{a}-\sqrt{b}|\leq2|\sqrt{a}-\sqrt{b}|$ respectively.
\\By the Cauchy--Schwartz inequality and \eqref{eq:theorem7.17_9},
\begin{equation}\label{eq:theorem7.17_10}
    \begin{aligned}
         \int_{\xbf\in[0,1]^p}|\sigma(f_0(\xbf))-\sigma(\hat{f}(\xbf))|d\xbf&\leq\left[ \int_{\xbf\in[0,1]^p}\{\sigma(f_0(\xbf))-\sigma(\hat{f}(\xbf))\}^2d\xbf\right]^{\frac{1}{2}}\\
         &\leq2\sqrt{2}d_{\mathrm{H}}(\hat{\ell},\ell_0)=o_{P_0}(1).
    \end{aligned}
\end{equation}
The proof of part 2 is completed by \eqref{eq:theorem7.17_7}.

\subsection{Proof of Corollary \ref{cor:cor4.4}}

Let $D_n\sim n^u$ and $\epsilon_n\sim n^{-\delta}$, $0<\delta<1-u$. This implies $D_n\log (n)=o(n\epsilon_n^2)$. Additionally, $D_n\log (n)=o(n^v\epsilon_n^2)$, $u+\delta<v<1$. This implies $D_n\log (n)=o(n^v\epsilon_n^{2\kappa})$, $0\leq\kappa\leq1$.
\\Thus, using Proposition \ref{proposition:prop7.15} with $\epsilon_n=\epsilon_n^{\kappa}$, we have
\begin{equation*}
    -q^{*}(\mathcal{V}^c_{\varepsilon\epsilon_n^{\kappa}})A\geq n\varepsilon^2\epsilon_n^{2\kappa} q^{*}(\mathcal{V}^c_{\varepsilon\epsilon_n^{\kappa}})-\log 2+o_{P_0}(1)=n\varepsilon^2\epsilon_n^{2\kappa}q^{*}(\mathcal{V}^c_{\varepsilon\epsilon_n^{\kappa}})+O_{P_0}(1).
\end{equation*}
This together with \eqref{eq:theorem4.2_1}, \eqref{eq:theorem4.2_2} and \eqref{eq:theorem4.1_2} implies
\begin{equation*}
    q^{*}(\mathcal{V}^c_{\varepsilon\epsilon_n^{\kappa}})=o_{P_0}(\epsilon_n^{2-2\kappa}).
\end{equation*}
By \eqref{eq:cor4.3_1},
\begin{equation*}
    \begin{aligned}
         d_{\mathrm{H}}(\hat{\ell},\ell_0)&=d_{\mathrm{H}}\left(\int\ell_{G_{\etabf}(\zbf)}q^*(\zbf)d\zbf,\ell_0\right)\\
         &\leq \int d_{\mathrm{H}}(\ell_{G_{\etabf}(\zbf)},\ell_0)q^*(\zbf)d\zbf \quad\text{by}\; \text{Jensen's}\;\text{inequality}\\
         &=\int_{\mathcal{V}_{\varepsilon\epsilon_n^{\kappa}}}d_{\mathrm{H}}(\ell_{G_{\etabf}(\zbf)},\ell_0)q^*(\zbf)d\zbf+\int_{\mathcal{V}_{\varepsilon\epsilon_n^{\kappa}}^c}d_{\mathrm{H}}(\ell_{G_{\etabf}(\zbf)},\ell_0)q^*(\zbf)d\zbf\\
         &\leq\varepsilon\epsilon_n^{\kappa}+o_{P_0}(\epsilon_n^{2-2\kappa}).
    \end{aligned}
\end{equation*}
Dividing by $\epsilon_n^{\kappa}$ on both sides, we have
\begin{equation*}
    \frac{1}{\epsilon_n^{\kappa}}d_{\mathrm{H}}(\hat{\ell},\ell_0)=o_{P_0}(1)+o_{P_0}(\epsilon_n^{2-3\kappa})=o_{P_0}(1),\quad0\leq\kappa\leq2/3
\end{equation*}
By \eqref{eq:theorem7.17_10}, for every $0\leq\kappa\leq2/3$,
\begin{equation*}
    \frac{1}{\epsilon_n^{\kappa}}\int_{\xbf\in[0,1]^p}|\sigma(f_0(\xbf))-\sigma(\hat{f}(\xbf))|d\xbf\leq\frac{1}{\epsilon_n^{\kappa}}2\sqrt{2}d_{\mathrm{H}}(\hat{\ell},\ell_0)=o_{P_0}(1).
\end{equation*}
The proof completes following \eqref{eq:theorem7.17_7}.

\section{Theoretical Properties of True Posterior}

\begin{theorem}\label{thm:thm7.17}
Suppose $r_n\sim n^a$, $D_n\sim n^u$, $0<a\leq u<1$. Assume that the deterministic function $G_{\etabf}$ is differentiable at $\zbf\in\mathscr{Z}$ and the first-order derivative satisfies $\|\partial G_{\etabf}/\partial \zbf\|_F^2=o(\varepsilon n^{1-u})$ for $\zbf\in\mathcal{P}_{\varepsilon}$ defined in \eqref{eq:p_neighbour}. Besides, the prior parameters in \eqref{eq:prior} satisfy Assumption \ref{assump:assump1b}. Then, as $R$ is defined in \eqref{eq:testerror},
\begin{itemize}
    \item[1.] 
    $P_0\left(p(\mathcal{U}^c_{\varepsilon}\mid\mathcal{D}_n)\leq 2e^{-n\varepsilon^2/2}\right)\rightarrow 1, \quad n \rightarrow \infty.$
    \item[2.]
    $P_0\left(\left|R(\hat{C})-R(C^{\tt{Bayes}})\right|\leq 4\sqrt{2}\varepsilon\right)\rightarrow 1, \quad n \rightarrow \infty.$
\end{itemize}
\end{theorem}

\begin{proof} The proof of this theorem is comprised of two parts.
\\\textit{Proof of part 1}. In view of \eqref{eq:hellinger_w} and \eqref{eq:hellingerz}, we note that
\begin{equation}\label{eq:theorem7.17_3}
    p(\mathcal{U}_{\varepsilon}^c\mid\mathcal{D}_n)=p_{\etabf}(\mathcal{V}_{\varepsilon}^c\mid\mathcal{D}_n).
\end{equation}
Also, note that,
\begin{equation}\label{eq:theorem7.17_4}
    \begin{aligned}
        p_{\etabf}(\mathcal{V}^c_{\varepsilon}\mid\mathcal{D}_n)&=\frac{\int_{\mathcal{V}_{\varepsilon}^c}L(\mathcal{D}_n; f_{G_{\etabf}(\zbf)})\pi_0(\zbf)d\zbf}{\int L(\mathcal{D}_n; f_{G_{\etabf}(\zbf)})\pi_0(\zbf)d\zbf}\\
        &=\frac{\int_{\mathcal{V}_{\varepsilon}^c}\{L(\mathcal{D}_n; f_{G_{\etabf}(\zbf)})/L_0\}\pi_0(\zbf)d\zbf}{\int \{L(\mathcal{D}_n; f_{G_{\etabf}(\zbf)})/L_0\}\pi_0(\zbf)d\zbf}.
    \end{aligned}
\end{equation}
By Assumption \ref{assump:assump1b}, the prior parameters satisfy
\begin{equation*}
    \|\mubf\|_2^2=o(n),\quad \|\zetabf\|_{\infty}=O(n),\quad \|\zetabf^*\|_{\infty}=O(1),\quad \zetabf^*=1/\zetabf.
\end{equation*}
Note $r_n \sim n^a$, $0<a<1$ which implies $r_n\log (n) = o(n)$. And $D_n \sim n^u$, $0<u<1$ which implies $D_n\log (n) = o(n)$. Thus, the conditions of Proposition \ref{proposition:prop7.12} hold with $\epsilon_n = 1$. 
\begin{equation}\label{eq:theorem7.17_5}
\begin{aligned}
        &P_0\left(\int \frac{L(\mathcal{D}_n; f_{G_{\etabf}(\zbf)})}{L_0}\pi_0(\zbf)d\zbf\leq e^{-nv}\right)\\
        &\leq P_0\left(\left|\log\left\{\int \frac{L(\mathcal{D}_n; f_{G_{\etabf}(\zbf)})}{L_0}\pi_0(\zbf)d\zbf\right\}\right|>nv\right)\rightarrow 0,\quad n\rightarrow\infty,
\end{aligned}
\end{equation}
where the above convergence follows from \eqref{eq:prop7.16_5} in step 1 (c) of the proof of Proposition \ref{proposition:prop7.16}.
\\Additionally, since $r_n\log (n) = o(n^b)$, $a<b<1$, the conditions of Proposition \ref{proposition:prop7.15} hold with $\epsilon_n=1$.
\begin{equation}\label{eq:theorem7.17_6}
    P_0\left(\int_{\mathcal{V}_{\varepsilon}^c} \frac{L(\mathcal{D}_n; f_{G_{\etabf}(\zbf)})}{L_0}\pi_0(\zbf)d\zbf\geq 2e^{-n\varepsilon^2}\right)\rightarrow 0,\quad n\rightarrow\infty,
\end{equation}
where the last equality follows from \eqref{eq:prop7.15_1} with $\epsilon_n = 1$ in the proof of Proposition \ref{proposition:prop7.15}.
\\Using \eqref{eq:theorem7.17_5} and \eqref{eq:theorem7.17_6} with \eqref{eq:theorem7.17_3} and \eqref{eq:theorem7.17_4}, we get
\begin{equation*}
    P_0\left(p(\mathcal{U}^c_{\varepsilon}\mid\mathcal{D}_n)\geq 2e^{-n(\varepsilon^2-\nu)/2}\right)=P_0\left(p_{\etabf}(\mathcal{V}^c_{\varepsilon}\mid\mathcal{D}_n)\geq 2e^{-n(\varepsilon^2-\nu)/2}\right)\rightarrow 0, \quad n \rightarrow \infty.
\end{equation*}
Taking $\nu=\varepsilon^2/2$, the proof is completed.
\\\textit{Proof of part 2}. Let \begin{equation}\label{eq:theorem7.17_8}
    \hat{\ell}(y, \xbf)=\int\ell_{G_{\etabf}(\zbf)}(y,\xbf)p_{\etabf}(\zbf\mid\mathcal{D}_n)d\zbf.
\end{equation}
Then 
\begin{equation*}
    \begin{aligned}
         d_{\mathrm{H}}(\hat{\ell},\ell_0)&=d_{\mathrm{H}}\left(\int\ell_{G_{\etabf}(\zbf)}p_{\etabf}(\zbf\mid\mathcal{D}_n)d\zbf,\ell_0\right)\\
         &\leq \int d_{\mathrm{H}}(\ell_{G_{\etabf}(\zbf)},\ell_0)p_{\etabf}(\zbf\mid\mathcal{D}_n)d\zbf \quad\text{by}\; \text{Jensen's}\;\text{inequality}\\
         &=\int_{\mathcal{V}_{\varepsilon}}d_{\mathrm{H}}(\ell_{G_{\etabf}(\zbf)},\ell_0)p_{\etabf}(\zbf\mid\mathcal{D}_n)d\zbf+\int_{\mathcal{V}_{\varepsilon}^c}d_{\mathrm{H}}(\ell_{G_{\etabf}(\zbf)},\ell_0)p_{\etabf}(\zbf\mid\mathcal{D}_n)d\zbf\\
         &\leq\varepsilon+2e^{-n\varepsilon^2/2}\\
         &\leq 2\varepsilon \quad\text{with}\;\text{probability}\;\text{tending}\;\text{to}\;1\;\text{as}\;n\rightarrow\infty,
    \end{aligned}
\end{equation*}
where the last inequality is a consequence of part 1 of Theorem \ref{thm:thm7.17}.
\\The proof of part 2 is complete after \eqref{eq:theorem7.17_7} and \eqref{eq:theorem7.17_10}.
\end{proof}

\begin{theorem}\label{thm:thm7.18}
Suppose $r_n\sim n^a$, $D_n\sim n^u$, $0<a\leq u<1$, $\epsilon_n^2\sim n^{-\delta}$, $0<\delta<1-u\leq 1-a$. 
Assume that the deterministic function $G_{\etabf}$ is differentiable at $\zbf\in\mathscr{Z}$ and that the first-order derivative satisfies $\|\partial G_{\etabf}/\partial \zbf\|_F^2=o(\varepsilon\epsilon_n^2 n^{1-u})$ for $\zbf\in\mathcal{P}_{\varepsilon\epsilon_n^2}$ defined in \eqref{eq:p_neighbour}. In addition, the prior parameters in \eqref{eq:prior} satisfy Assumption \ref{assump:assump2b}. Then
\begin{itemize}
    \item[1.] 
    $P_0\left(p(\mathcal{U}^c_{\varepsilon\epsilon_n}\mid\mathcal{D}_n)\leq 2e^{-n\epsilon_n^2\varepsilon^2/2}\right)\rightarrow 1,\quad n \rightarrow \infty.$
    \item[2.]
    $P_0\left(\left|R(\hat{C})-R(C^{\tt{Bayes}})\right|\leq 4\sqrt{2}\varepsilon\epsilon_n\right)\rightarrow 1,\quad n \rightarrow \infty.$
\end{itemize}
\end{theorem}

\begin{proof}
The proof of this theorem consists of two parts.
\\\textit{Proof of part 1}. In view of \eqref{eq:hellinger_w} and \eqref{eq:hellingerz}, we notice that
\begin{equation}\label{eq:theorem7.18_1}
    p(\mathcal{U}_{\varepsilon\epsilon_n}^c\mid\mathcal{D}_n)=p_{\etabf}(\mathcal{V}_{\varepsilon\epsilon_n}^c\mid\mathcal{D}_n).
\end{equation}
By Assumption \ref{assump:assump2b}, the prior parameters satisfy
\begin{equation*}
    \|\mubf\|_2^2=o(n\epsilon_n^2),\quad \|\zetabf\|_{\infty}=O(n),\quad \|\zetabf^*\|_{\infty}=O(1),\quad \zetabf^*=1/\zetabf.
\end{equation*}
Also, by Assumption \ref{assump:assump2c},
\begin{equation*}
     \|f_0-f_{\Upsilonbf}\|_{\infty}=o(\epsilon_n^2),\quad
     \|\Upsilonbf\|_2^2=o(n\epsilon_n^2),\quad
     \|\zbf\|_2^2=o(n\epsilon_n^2).
\end{equation*}
Note $r_n \sim n^a$, $0<a<1$, $D_n \sim n^u$, $0<u<1$ and $\epsilon_n^2\sim n^{-\delta}$, $0<\delta<1-u<1-a$, therefore, $r_n\log (n) = o(n\epsilon_n^2)$, $D_n\log (n) = o(n\epsilon_n^2)$. Thus, the conditions of Proposition \ref{proposition:prop7.12} hold. 
\begin{equation}\label{eq:theorem7.18_2}
\begin{aligned}
        &P_0\left(\int \frac{L(\mathcal{D}_n; f_{G_{\etabf}(\zbf)})}{L_0}\pi_0(\zbf)d\zbf\leq e^{-n\epsilon_n^2v}\right)\\
        &\leq P_0\left(\left|\log\left\{\int \frac{L(\mathcal{D}_n; f_{G_{\etabf}(\zbf)})}{L_0}\pi_0(\zbf)d\zbf\right\}\right|>n\epsilon_n^2v\right)\rightarrow 0, \quad n\rightarrow\infty,
\end{aligned}
\end{equation}
where the above convergence follows from \eqref{eq:prop7.16_8} in step 2 (c) of the proof of Proposition \ref{proposition:prop7.16}.
\\Also, since $r_n \log (n) = o(n^b\epsilon_n^2)$, $a+\delta<b<1$, $D_n \log (n) = o(n^v\epsilon_n^2)$, $u+\delta<v<1$, the conditions of Proposition \ref{proposition:prop7.15} are satisfied.
\begin{equation}\label{eq:theorem7.18_3}
    P_0\left(\int_{\mathcal{V}_{\varepsilon}^c} \frac{L(\mathcal{D}_n; f_{G_{\etabf}(\zbf)})}{L_0}\pi_0(\zbf)d\zbf\geq 2e^{-n\epsilon_n^2\varepsilon^2}\right)\rightarrow 0, \quad n\rightarrow\infty,
\end{equation}
where the last equality follows from \eqref{eq:prop7.15_1} in the proof of Proposition \ref{proposition:prop7.15}.
\\Using \eqref{eq:theorem7.18_2} and \eqref{eq:theorem7.18_3} with \eqref{eq:theorem7.18_1} and \eqref{eq:theorem7.17_4}, we get
\begin{equation*}
    P_0\left(p(\mathcal{U}^c_{\varepsilon}\mid\mathcal{D}_n)\geq 2e^{-n\epsilon_n^2(\varepsilon^2-\nu)/2}\right)=P_0\left(p_{\etabf}(\mathcal{V}^c_{\varepsilon}\mid\mathcal{D}_n)\geq 2e^{-n\epsilon_n^2(\varepsilon^2-\nu)/2}\right)\rightarrow 0, \quad n \rightarrow \infty.
\end{equation*}
Taking $\nu=\varepsilon^2/2$, we complete the proof.
\\\textit{Proof of part 2}. By \eqref{eq:theorem7.17_8}, we get
\begin{equation*}
    \begin{aligned}
         d_{\mathrm{H}}(\hat{\ell},\ell_0)&=d_{\mathrm{H}}\left(\int\ell_{G_{\etabf}(\zbf)}p_{\etabf}(\zbf\mid\mathcal{D}_n)d\zbf,\ell_0\right)\\
         &\leq \int d_{\mathrm{H}}(\ell_{G_{\etabf}(\zbf)},\ell_0)p_{\etabf}(\zbf\mid\mathcal{D}_n)d\zbf \quad\text{by}\; \text{Jensen's}\;\text{inequality}\\
         &=\int_{\mathcal{V}_{\varepsilon\epsilon_n}}d_{\mathrm{H}}(\ell_{G_{\etabf}(\zbf)},\ell_0)p_{\etabf}(\zbf\mid\mathcal{D}_n)d\zbf+\int_{\mathcal{V}_{\varepsilon\epsilon_n}^c}d_{\mathrm{H}}(\ell_{G_{\etabf}(\zbf)},\ell_0)p_{\etabf}(\zbf\mid\mathcal{D}_n)d\zbf\\
         &\leq\varepsilon\epsilon_n+2e^{-2n\epsilon_n^2\varepsilon^2}\\
         &\leq 2\varepsilon \quad\text{with}\;\text{probability}\;\text{tending}\;\text{to}\;1\;\text{as}\;n\rightarrow\infty,
    \end{aligned}
\end{equation*}
where the last inequality is a consequence of part 1 of Theorem \ref{thm:thm7.18} and $\epsilon_n\sim n^{-\delta}$. Dividing by $\epsilon_n$ on both sides, we get
\begin{equation*}
    \epsilon_n^{-1}d_{\mathrm{H}}(\hat{\ell},\ell_0)\leq2\varepsilon \quad\text{with}\;\text{probability}\;\text{tending}\;\text{to}\;1\;\text{as}\;n\rightarrow\infty.
\end{equation*}
The remainder of the proof is followed by \eqref{eq:theorem7.17_7} and \eqref{eq:theorem7.17_10}.
\end{proof}

\section{Lemmas}

\begin{lemma}\label{lemma:lemma7.9}
Assume that the deterministic function $G_{\etabf}$ is twice differentiable at $\zbf\in\mathscr{Z}$ and the second-order derivative satisfies $$\{\sum_{k=1}^{D_n}(\partial^2 G_{\etabf k}/\partial z_j^2)^2\}^{1/2}=o(\sqrt{D_n}\sum_{k=1}^{D_n}(\partial G_{\etabf k}/\partial z_j)^2),$$ for $j=1,\ldots,r_n$. For
\begin{equation*}
     \resizebox{0.99\linewidth}{!}{$h(\zbf)=\int_{\xbf\in[0,1]^p}[\sigma(f_0(\xbf))\{f_0(\xbf)-f_{G_{\etabf}(\zbf)}(\xbf)\}+\log\{1-\sigma(f_0(\xbf))\}-\log\{1-\sigma(f_{G_{\etabf}(\zbf)}(\xbf))\}]d\xbf$},
\end{equation*}
we have
\begin{equation*}
    \sum_{j=1}^{r_n}|(\nabla^2h(\zbf))_{jj}|\leq D_n (2c^2+1) \left|\left|\partial G_{\etabf}/\partial \zbf\right|\right|_{F}^2,
\end{equation*}
where $\mathbf{A}_{jj}$ denotes the $j$th diagonal entry of a matrix $\mathbf{A}$ and c is a constant.
\end{lemma}

\begin{proof}
Note that
\begin{equation}\label{eq:lemma7.9_1}
    \begin{aligned}
        \nabla^2h(\zbf)&=-\int_{\xbf\in[0,1]^p}\underbrace{\{\sigma(f_0(\xbf))+\sigma(f_{G_{\etabf}(\zbf)}(\xbf))\}}_{g_1(\xbf)}\nabla^2 f_{G_{\etabf}(\zbf)}(\xbf)d\xbf\\
        &-\int_{\xbf\in[0,1]^p}\underbrace{\{\sigma(f_{G_{\etabf}(\zbf)}(\xbf))\}\{1-\sigma(f_{G_{\etabf}(\zbf)}(\xbf))\}}_{g_2(\xbf)}\nabla f_{G_{\etabf}(\zbf)}(\xbf)\{\nabla f_{G_{\etabf}(\zbf)}(\xbf)\}^{\T} d\xbf.
    \end{aligned}
\end{equation}
First, note that
\begin{equation*}
        (\nabla f_{G_{\etabf}(\zbf)}(\xbf))_j=\frac{\partial f_{G_{\etabf}(\zbf)}}{\partial z_j}=\sum_{k=1}^{D_n}\frac{\partial f_{\wbf}}{\partial w_k}\frac{\partial w_k}{\partial z_j},
\end{equation*}
where $v_j$ denotes the $j$th element of a vector.
\begin{equation}\label{eq:lemma7.9_2}
    \begin{aligned}
        (\nabla f_{G_{\etabf}(\zbf)}(\xbf)\{\nabla f_{G_{\etabf}(\zbf)}(\xbf)\}^{\T})_{jj}&=\left(\sum_{k=1}^{D_n}\frac{\partial f_{\wbf}}{\partial w_k}\frac{\partial w_k}{\partial z_j}\right)\left(\sum_{k=1}^{D_n}\frac{\partial f_{\wbf}}{\partial w_k}\frac{\partial w_k}{\partial z_j}\right)\\
        &\leq \left\{\sum_{k=1}^{D_n}\left(\frac{\partial f_{\wbf}}{\partial w_k}\right)^2\right\}\left\{\sum_{k=1}^{D_n}\left(\frac{\partial G_{\etabf k}}{\partial z_j}\right)^2\right\}\\
        &\leq D_n \max_{t}a_t^2\sum_{k=1}^{D_n}\left(\frac{\partial G_{\etabf k}}{\partial z_j}\right)^2,
    \end{aligned}
\end{equation}
where the second inequality holds by the Cauchy--Schwarz inequality. And the last inequality follows since $(\partial f_{\wbf}/\partial w_k)^2\leq \max_{t}a_t^2$ by combining $0<\sigma^{'}(\cdot)<1$, $0\leq x_{t^{'}}\leq 1$ and
\begin{equation*}
    \frac{\partial f_{\wbf}(\xbf)}{\partial w_k} 
    =\left\{
    \begin{array}{ll}
        1, &\quad w_k=a_t \quad\text{for}\;\text{ some}\;t=1,\ldots,d\\
        a_t\sigma^{'}(\omegabf_t^{\T}\xbf+b_t)x_{t^{'}}, &\quad w_k=\omega_{tt^{'}}\quad\text{for}\;\text{some}\;t=1,\ldots,d, t^{'}=1,\ldots,p\\
        a_t\sigma^{'}(\omegabf_t^{\T}\xbf+b_t), &\quad w_k=b_{t}\quad\text{for}\;\text{ some}\;t=1,\ldots,d.
    \end{array}
    \right.
\end{equation*}
On the other hand,
\begin{equation*}
\begin{aligned}
    (\nabla^2 f_{G_{\etabf}(\zbf)}(\xbf))_{jj}&=\underbrace{\left(\frac{\partial w_1}{\partial z_j},\ldots, \frac{\partial w_{D_n}}{\partial z_j}\right)\nabla^2_{\wbf} f_{\wbf}(\xbf)\left(\frac{\partial w_1}{\partial z_j},\ldots, \frac{\partial w_{D_n}}{\partial z_j}\right)^{\T}}_{(\mathrm{I})}\\
    &+\underbrace{\sum_{k=1}^{D_n}\frac{\partial f_{\wbf}(\xbf)}{\partial w_k}\frac{\partial^2 w_k}{\partial z_j^2}}_{(\mathrm{II})}.
\end{aligned}
\end{equation*}
Note that
\begin{equation*}
        (\mathrm{I})\leq \lambda_{\tt{max}}\left|\left|\left(\frac{\partial w_1}{\partial z_j},\ldots, \frac{\partial w_{D_n}}{\partial z_j}\right)^{\T}\right|\right|^2_2\leq D_n\max_{t}|a_t|\sum_{k=1}^{D_n}\left(\frac{\partial G_{\etabf k}}{\partial z_j}\right)^2,
\end{equation*}
where $\lambda_{\tt{max}}$ is the largest singular value of $\nabla^2_{\wbf} f_{\wbf}(\xbf)$ and the first inequality is based on the spectral norm of a matrix. The last inequality follows by the Gershgorin circle theorem where $\sum_{k_2=1}^{D_n}(\nabla^2_{\wbf} f_{\wbf}(\xbf))_{k_1k_2}\leq D_n\max_{t}|a_t|$ for any $k_1=1,\ldots,D_n$, since
\begin{equation*}
\begin{aligned}
    &(\nabla^2_{\wbf} f_{\wbf}(\xbf))_{k_1k_2}\\
    &=\frac{\partial^2 f_{\wbf}(\xbf)}{\partial w_{k_1}\partial w_{k_2}}\\
    &=\left\{
    \begin{array}{ll}
        \sigma^{'}(\omegabf_t^{\T}\xbf+b_t)x_{t^{'}}, &\quad w_{k_1}=\omega_{tt^{'}}, w_{k_2}=a_{t}\quad\text{for}\;\text{some}\;t=1,\ldots,d, t^{'}=1,\ldots,p\\
        a_t\sigma^{''}(\omegabf_t^{\T}\xbf+b_t)x_{t^{'}}x_{\Tilde{t}}, &\quad w_{k_1}=\omega_{tt^{'}}, w_{k_2}=\omega_{t\Tilde{t}}\quad\text{for}\;\text{some}\;t=1,\ldots,d, t^{'},\Tilde{t}=1,\ldots,p\\
        a_t\sigma^{''}(\omegabf_t^{\T}\xbf+b_t)x_{t^{'}}, &\quad w_{k_1}=\omega_{tt^{'}}, w_{k_2}=b_t\quad\text{for}\;\text{some}\;t=1,\ldots,d, t^{'}=1,\ldots,p\\
        \sigma^{'}(\omegabf_t^{\T}\xbf+b_t), &\quad w_{k_1}=b_{t}, w_{k_2}=a_t\quad\text{for}\;\text{ some}\;t=1,\ldots,d\\
        a_t\sigma^{''}(\omegabf_t^{\T}\xbf+b_t)x_{t^{'}}, &\quad w_{k_1}=b_{t}, w_{k_2}=\omega_{tt^{'}} \quad\text{for}\;\text{ some}\;t=1,\ldots,d, t^{'}=1,\ldots,p\\
        a_t\sigma^{''}(\omegabf_t^{\T}\xbf+b_t), &\quad w_{k_1}=b_{t_1}, w_{k_2}=b_t\quad\text{for}\;\text{ some}\;t=1,\ldots,d\\
        0, &\quad\text{otherwise}
    \end{array}
    \right.\\ 
    &\leq \max_{t}|a_t|,
\end{aligned}
\end{equation*}
where the last inequality follows using $0<\sigma^{'}(\cdot)< 1$, $0<\sigma^{''}(\cdot)< 1$, $0\leq x_{t^{'}}\leq 1$.
\\Also, note that
\begin{equation*}
        (\mathrm{II})\leq \sqrt{\left\{\sum_{k=1}^{D_n}\left(\frac{\partial f_{\wbf}}{\partial w_k}\right)^2\right\}\left\{\sum_{k=1}^{D_n}\left(\frac{\partial^2 G_{\etabf k}}{\partial z_j^2}\right)^2\right\}}\leq \sqrt{D_n \max_{t}a_t^2\sum_{k=1}^{D_n}\left(\frac{\partial^2 G_{\etabf k}}{\partial z_j^2}\right)^2},
\end{equation*}
where the first inequality follows by the Cauchy--Schwarz inequality.
\\Therefore, we get
\begin{equation}\label{eq:lemma7.9_3}
    (\nabla^2 f_{G_{\etabf}(\zbf)}(\xbf))_{jj}\leq D_n\max_{t}|a_t|\sum_{k=1}^{D_n}\left(\frac{\partial G_{\etabf k}}{\partial z_j}\right)^2+\sqrt{D_n \max_{t}a_t^2\sum_{k=1}^{D_n}\left(\frac{\partial^2 G_{\etabf k}}{\partial z_j^2}\right)^2}.
\end{equation}
Note $|g_1(\xbf)|\leq 2$, $|g_2(\xbf)|\leq 1$ and combining \eqref{eq:lemma7.9_2} and \eqref{eq:lemma7.9_3} and replacing \eqref{eq:lemma7.9_1}, we get
\begin{equation*}
    \begin{aligned}
        \sum_{j=1}^{r_n}|(\nabla^2h(\zbf))_{jj}|&\leq \sum_{j=1}^{r_n}\left\{D_n (\max_{t}a_t^2+\max_{t}|a_t|)\sum_{k=1}^{D_n}\left(\frac{\partial G_{\etabf k}}{\partial z_j}\right)^2\right\}\\
        &+\sum_{j=1}^{r_n}\left\{\max_{t}|a_t|\sqrt{D_n\sum_{k=1}^{D_n}\left(\frac{\partial^2 G_{\etabf k}}{\partial z_j^2}\right)^2}\right\}\\
        &\lesssim \sum_{j=1}^{r_n}\left\{D_n (\max_{t}a_t^2+\max_{t}|a_t|)\sum_{k=1}^{D_n}\left(\frac{\partial G_{\etabf k}}{\partial z_j}\right)^2\right\}\\
        &\leq D_n (2\max_{t}a_t^2+1)\sum_{j=1}^{r_n}\sum_{k=1}^{D_n}\left(\frac{\partial G_{\etabf k}}{\partial z_j}\right)^2\\
        &\leq D_n (2c^2+1) \left|\left|\frac{\partial G_{\etabf}}{\partial \zbf}\right|\right|_{F}^2,
    \end{aligned}
\end{equation*}
where the second inequality follows by $$\{\sum_{k=1}^{D_n}(\partial^2 G_{\etabf k}/\partial z_j^2)^2\}^{1/2}=o(\sqrt{D_n}\sum_{k=1}^{D_n}(\partial G_{\etabf k}/\partial z_j)^2)$$ and the third inequality in the above step uses $|x|<x^2+1$. The last inequality follows by letting $\max_{t}|a_t|<c$.
\end{proof}
\begin{lemma}\label{lemma:lemma7.10} 
Let $\Tilde{\mathcal{F}}_n = \{\sqrt{\ell}: \ell_{G_{\etabf}(\zbf)}(y, \xbf), \zbf\in\mathcal{F}_n\}$ where $\ell_{G_{\etabf}}(y, \xbf)=\ell_{\wbf}(y, \xbf)$ is the same as in \eqref{eq:l_w} and $\mathcal{F}_n=\{\zbf: |z_j|\leq \Tilde{C}_n, j=1,\ldots, r_n\}$ is defined in \eqref{eq:f_sieves}. Assume that $w_k=G_{\etabf k}(\zbf)$ satisfies $|w_k|\leq C_n$, for $k=1,\ldots,D_n$.  
Furthermore, assume that the deterministic function $G_{\etabf}$ is differentiable at $\zbf\in\mathscr{Z}$ and the first-order derivative satisfies $\|\partial G_{\etabf}/\partial \zbf\|_F^2
=o(\varepsilon\epsilon_n^2n^{1-u})$ for $\zbf\in\mathcal{F}_n$.
Then
\begin{equation*}
    \int_{\varepsilon^2/8}^{\sqrt{2}\varepsilon}\sqrt{H(u, \Tilde{\mathcal{F}}_n, \|\cdot\|_2)}du\lesssim \varepsilon\sqrt{2r_n\{\log (r_n)+\log (D_n)+\log (C_n)+\log(\Tilde{C}_n)-\log(\varepsilon)\}},
\end{equation*}
where $H(u, \Tilde{\mathcal{F}}_n, \|\cdot\|_2)$ is the natural logarithm of the bracketing number defined in \cite{pollard1990empirical} and explained in detail in the proof below.
\end{lemma}

\begin{proof}
As stated in \cite{pollard1990empirical}, for any two functions $l$ and $u$, we define the bracket $[l, u]$ as the set of all functions $f$ such that $l\leq f\leq u$. Then an $\varepsilon$-bracket is defined as a bracket with $\|u-l\|\leq \varepsilon$ where $\|\cdot\|$ is a metric. Define the bracketing number of a set of functions $\mathcal{F}^*$ as the minimum number of $\varepsilon$-brackets needed to cover $\mathcal{F}^*$, and denote it by $N(\varepsilon, \mathcal{F}^*, \|\cdot\|)$. Finally, the bracketing entropy, denoted by $H(\varepsilon, \mathcal{F}^*, \|\cdot\|)$, is the natural logarithm of the bracketing number.
\\Note that by Lemma 4.1 in \cite{pollard1990empirical},
\begin{equation*}
    N(\varepsilon, \mathcal{F}_n, \|\cdot\|_{\infty})\leq(\frac{3\Tilde{C}_n}{\varepsilon})^{r_n}.
\end{equation*}
For $\zbf_1, \zbf_2\in\mathcal{F}_n$, let $\Tilde{\ell}(u)=\sqrt{\ell_{u\zbf_1+(1-u)\zbf_2}(y,\xbf)}$.
\\Following equation (52) in \cite{Bhattacharya2021StatisticalFO}, we get
\begin{equation}\label{eq:lemma7.10_1}
    \sqrt{\ell_{\zbf_1}(y,\xbf)}-\sqrt{\ell_{\zbf_1}(y,\xbf)}\leq r_n\sup_j\left|\frac{\partial\Tilde{\ell}}{\partial z_j}\right| \|\zbf_1-\zbf_2\|_{\infty}\leq F(\xbf,y)\|\zbf_1-\zbf_2\|_{\infty}
\end{equation}
where the upper bound $F(\xbf,y)=r_nD_nC_n/2$. This is because $|\partial\Tilde{\ell}/\partial z_j|$ is bounded as shown below:
\begin{equation*}
    \begin{aligned}
        \left|\frac{\partial\Tilde{\ell}}{\partial z_j}\right|&=\left|\frac{1}{2}\frac{\partial f_{G_{\etabf}(\zbf)}(\xbf)}{\partial z_j}\sqrt{\left\{y-\frac{e^{f_{G_{\etabf}(\zbf)}(\xbf)}}{1+e^{f_{G_{\etabf}(\zbf)}(\xbf)}}\right\}\exp\left\{yf_{G_{\etabf}(\zbf)}(\xbf)-\log(1+e^{f_{G_{\etabf}(\zbf)}(\xbf)})\right\}}\right|\\
        &=\frac{1}{2}\sqrt{\left\{y-\frac{e^{f_{G_{\etabf}(\zbf)}(\xbf)}}{1+e^{f_{G_{\etabf}(\zbf)}(\xbf)}}\right\}\exp\left\{y f_{G_{\etabf}(\zbf)}(\xbf)-\log(1+e^{f_{G_{\etabf}(\zbf)}(\xbf)})\right\}}\left|\sum_{k=1}^{D_n}\frac{\partial f_{\wbf}(\xbf)}{\partial w_k}\frac{\partial G_{\etabf k}}{\partial z_j}\right|\\
        &\leq \frac{1}{2}\sqrt{\frac{e^{f_{G_{\etabf}(\zbf)}(\xbf)}}{1+e^{f_{G_{\etabf}(\zbf)}(\xbf)}}}\sqrt{\frac{1}{1+e^{f_{G_{\etabf}(\zbf)}(\xbf)}}}\left|\sum_{k=1}^{D_n}\frac{\partial f_{\wbf}(\xbf)}{\partial w_k}\frac{\partial G_{\etabf k}}{\partial z_j}\right|\\
        &\leq \frac{1}{2}\sqrt{\frac{e^{f_{G_{\etabf}(\zbf)}(\xbf)}}{1+e^{f_{G_{\etabf}(\zbf)}(\xbf)}}}\sqrt{\frac{1}{1+e^{f_{G_{\etabf}(\zbf)}(\xbf)}}}\sqrt{\left\{\sum_{k=1}^{D_n}\left(\frac{\partial f_{\wbf}(\xbf)}{\partial w_k}\right)^2\right\}\left\{\sum_{k=1}^{D_n}\left(\frac{\partial G_{\etabf k}}{\partial z_j}\right)^2\right\}}.
    \end{aligned}
\end{equation*}
Also, note that
\begin{equation*}
    \left|\frac{\partial f_{\wbf}(\xbf)}{\partial w_k}\right| 
    \leq\left\{
    \begin{array}{ll}
        1, &\quad w_k=a_t \quad\text{for}\;\text{ some}\;t=1,\ldots,d\\
        |a_t\sigma^{'}(\omegabf_t^{\T}\xbf+b_t)x_{t^{'}}|, &\quad w_k=\omega_{tt^{'}}\quad\text{for}\;\text{some}\;t=1,\ldots,d, t^{'}=1,\ldots,p\\
        |a_t\sigma^{'}(\omegabf_t^{\T}\xbf+b_t)|, &\quad w_k=b_{t}\quad\text{for}\;\text{ some}\;t=1,\ldots,d.
    \end{array}
    \right.
\end{equation*}
Using $|\sigma^{'}(\cdot)|\leq 1$, $|x_{t^{'}}|\leq 1$, $|a_t|\leq C_n$,
\begin{equation*}
    \left|\frac{\partial f_{\wbf}(\xbf)}{\partial w_k}\right|\leq C_n.
\end{equation*}
Thus, using $e^{f_{G_{\etabf}(\zbf)}(\xbf)}/(1+e^{f_{G_{\etabf}(\zbf)}(\xbf)})\leq 1$, $1/(1+e^{f_{G_{\etabf}(\zbf)}(\xbf)})\leq 1$, $\|\partial G_{\etabf}/\partial \zbf\|_F^2
=o(\varepsilon\epsilon_n^2n^{1-u})$, we get
\begin{equation*}
    \left|\frac{\partial\Tilde{\ell}}{\partial z_j}\right|\leq \frac{1}{2}\sqrt{D_nC_n^2\varepsilon\epsilon_n^2n^{1-u}}.
\end{equation*}
By using $D_n\sim n^u$, $\epsilon_n^2\sim n^{\delta}$, $0<\delta<1-u$ and letting $\varepsilon=1$, the bound $F(\xbf,y)$ follows.
\\In view of \eqref{eq:lemma7.10_1} and Theorem 2.7.11 in \cite{vaart1996weak}, we have 
\begin{equation*}
    N(\varepsilon,\Tilde{\mathcal{F}}_n,|\|\cdot\|_2)\leq (\frac{3r_nD_nC_n\Tilde{C}_n}{2\varepsilon})^{r_n}\Rightarrow H(\varepsilon,\Tilde{\mathcal{F}}_n,\|\cdot\|_2)\leq r_n\log\left(\frac{r_nD_nC_n\Tilde{C}_n}{\varepsilon}\right).
\end{equation*}
Using Lemma 7.3 in \cite{bhattacharya2020variational} with $M_n=r_nD_nC_n\Tilde{C}_n$, we get
\begin{equation*}
    \int_{0}^{\varepsilon}\sqrt{H(u,\Tilde{\mathcal{F}}_n,\|\cdot\|_2)}du\lesssim\varepsilon\sqrt{r_n[\log(r_nD_nC_n\Tilde{C}_n)-\log(\varepsilon)]}.
\end{equation*}
Therefore,
\begin{equation*}
    \begin{aligned}
        \int_{\varepsilon^2/8}^{\sqrt{2}\varepsilon}\sqrt{H(u, \Tilde{\mathcal{F}}_n, \|\cdot\|_2)}du &\leq\int_{0}^{\sqrt{2}\varepsilon}\sqrt{H(u, \Tilde{\mathcal{F}}_n, \|\cdot\|_2)}du\\
         &\lesssim \sqrt{2}\varepsilon \sqrt{r_n\{\log(r_nD_nC_n\Tilde{C}_n)-\log(\sqrt{2}\varepsilon)\}}.
    \end{aligned}
\end{equation*}
The proof follows by noting that $\log (\sqrt{2}\varepsilon)\geq\log (\varepsilon)$.
\end{proof}

\section{Propositions}

\begin{proposition}\label{proposition:prop7.11}
Let $q(\zbf)=N(\sbf,\mathbf{I}_{r_n}/\sqrt{n})$ and $\pi_0(\zbf)=N(\mubf,\mathbf{\Sigma})$, where $\mathbf{\Sigma}=\mathrm{diag}(\zetabf)$ and $\zetabf^{*}=1/\zetabf$. Let $n\epsilon_n^2\rightarrow\infty$, $r_n\log (n)=o(n\epsilon_n^2)$, $\|\sbf\|_2^2=o(n\epsilon_n^2)$ and $\|\mubf\|_2^2=o(n\epsilon_n^2)$, then for any $\nu>0$,
\begin{equation*}
    d_{\mathrm{KL}}(q,\pi_0)\leq n\epsilon_n^2\nu,
\end{equation*}
provided $\|\zetabf\|_{\infty}=O(n)$, $\|\zetabf^{*}\|_{\infty}=O(1)$.
\end{proposition}

\begin{proof}
\begin{equation*}
    \begin{aligned}
        d_{\mathrm{KL}}(q,\pi_0)&=\sum_{j=1}^{r_n}\left\{\log (n\zeta_j)+\frac{1}{n\zeta_j^2}+\frac{(s_{j}-\mu_j)^2}{\zeta_j^2}-\frac{1}{2}\right\}\\
        &\leq \frac{r_n}{2}\{\log (n)-1\}+\sum_{j=1}^{r_n}\log (\zeta_j)+\frac{1}{n}\sum_{j=1}^{r_n}\frac{1}{\zeta_j^2}+2\sum_{j=1}^{r_n}\frac{s_{j}^2}{\zeta_j^2}+2\sum_{j=1}^{r_n}\frac{\mu_j^2}{\zeta_j^2}-\frac{r_n}{2}\\
        &\leq \frac{r_n}{2}\{\log (n)-1\}+r_n\log (\|\zetabf\|_{\infty})+\frac{r_n}{n}\|\zetabf^{*}\|_{\infty}+2\|\sbf\|_2^2\|\zetabf^{*}\|_{\infty}+2\|\mubf\|_2^2\|\zetabf^{*}\|_{\infty}\\
        &=o(n\epsilon_n^2),
    \end{aligned}
\end{equation*}
where the second last inequality uses $\zetabf^{*}=1/\zetabf$. The last equality follows since $\|\zetabf\|_{\infty}=O(n)$, $\|\zetabf^{*}\|_{\infty}=O(1)$, $r_n\log (n)=o(n\epsilon_n^2)$, $\|\sbf\|_2^2=o(n\epsilon_n^2)$ and $\|\mubf\|_2^2=o(n\epsilon_n^2)$.
\end{proof}

\begin{proposition}\label{proposition:prop7.12}
Let $\pi_0(\zbf)$ be as in \eqref{eq:prior}. Define 
\begin{equation}\label{eq:prop7.12_1}   
    \begin{aligned}
            &\mathcal{N}_{\varepsilon}=\{\wbf: d_{\mathrm{KL}}(\ell_0, \ell_{\wbf})<\varepsilon\},\\
            &\mathcal{L}_{\varepsilon}=\{\zbf: G_{\etabf}(\zbf)\in\mathcal{N}_{\varepsilon}\},
    \end{aligned}
\end{equation}
where
\begin{equation*}
    d_{\mathrm{KL}}(\ell_0, \ell_{\wbf})=\int_{\xbf\in[0,1]^p}\left[\sigma(f_0(\xbf))\{f_0(\xbf)-f_{\wbf}(\xbf)\}+\log \left\{\frac{1-\sigma(f_0(\xbf))}{1-\sigma(f_{\wbf}(\xbf))}\right\}\right]d\xbf.
\end{equation*}
Let $\|f_0-f_{\Upsilonbf}\|_{\infty}\leq\varepsilon\epsilon_n^2/4$, $n\epsilon_n^2\rightarrow\infty$. 
Assume that the deterministic function $G_{\etabf}$ is differentiable at $\zbf\in\mathscr{Z}$ and the first-order derivative satisfies $\|\partial G_{\etabf}/\partial \zbf\|_F^2=o(\varepsilon\epsilon_n^2n^{1-u})$ for $\zbf\in\mathcal{P}_{\varepsilon\epsilon_n^2}$ defined in \eqref{eq:p_neighbour}. If $r_n\log (n) = o(n\epsilon_n^2)$, $D_n\log (n) = o(n\epsilon_n^2)$, $\|\sbf\|_2^2= o(n\epsilon_n^2)$, $\|\Upsilonbf\|_2^2= o(n\epsilon_n^2)$, $\|\mubf\|_2^2 = o(n\epsilon_n^2)$, then for any $\nu>0$,
\begin{equation*}
    \int_{\zbf\in\mathcal{L}_{\varepsilon\epsilon_n^2}}\pi_0(\zbf)d\zbf\geq e^{-n\epsilon_n^2\nu},
\end{equation*}
provided $\|\zetabf\|_{\infty} = O(n)$, $\|\zetabf^*\|_{\infty}=O(1)$ where $\zetabf^* = 1/\zetabf$.
\end{proposition}

\begin{proof}
Let $f_{\Upsilonbf}(\xbf) = \sum^{d}_{j=1}a_j^{\Upsilonbf}\sigma((\omegabf_j^{\Upsilonbf})^{\T}\xbf+b_j^{\Upsilonbf})$ be the neural network such that
\begin{equation}\label{eq:prop7.12_2}
    \|f_{\Upsilonbf}-f_0\|_1\leq \frac{\varepsilon\epsilon^2_n}{4},
\end{equation}
and let $\Upsilonbf=G_{\etabf}(\sbf)$. Such a neural network exists since $\|f_{\Upsilonbf}-f_0\|_1\leq\|f_{\Upsilonbf}-f_0\|_{\infty}\leq\varepsilon\epsilon^2_n/4$.
\\Next define a neighborhood $\mathcal{M}_{\varepsilon\epsilon_n^2}$ as follows:
\begin{equation*}
    \mathcal{M}_{\varepsilon\epsilon_n^2}=\left\{\wbf: |w_k-\Upsilon_{k}|<\frac{\varepsilon\epsilon_n^2}{8\{D_n+(p+1)\|\Upsilonbf\|_1\}},\; k=1,\ldots,D_n\right\}.
\end{equation*}
For every $\wbf\in\mathcal{M}_{\varepsilon\epsilon_n^2}$, by Lemma 7.7 in \cite{bhattacharya2020variational}, we have
\begin{equation}\label{eq:prop7.12_3}
    \|f_{\wbf}-f_{\Upsilonbf}\|_1\leq \frac{\varepsilon\epsilon^2_n}{2}.
\end{equation}
Combining \eqref{eq:prop7.12_2} and \eqref{eq:prop7.12_3}, we get for $\wbf\in\mathcal{M}_{\varepsilon\epsilon_n^2}$, $\|f_{\wbf}-f_{0}\|_1\leq\varepsilon\epsilon^2_n/2$.
\\Thus, in view of Lemma 7.8 in \cite{bhattacharya2020variational}, $d_{\mathrm{KL}}(\ell_0, \ell_{\wbf})<\varepsilon\epsilon_n^2$.
\\Thus, for every $\wbf\in\mathcal{M}_{\varepsilon\epsilon_n^2}$, we have $\wbf\in\mathcal{N}_{\varepsilon\epsilon_n^2}$.
\\Now define a neighborhood $\mathcal{P}_{\varepsilon\epsilon^2_n}$ as follows, which is the same as \eqref{eq:p_neighbour}:
\begin{equation*}
    \mathcal{P}_{\varepsilon\epsilon_n^2}=\left\{\zbf: |z_j-s_{j}|<\frac{\sqrt{\varepsilon\epsilon_n^2}}{8\sqrt{r_n n^{1-u}}\{D_n+(p+1)\|\Upsilonbf\|_1\}}, j=1,\ldots,r_n\right\}.
\end{equation*}
Then, for $k=1, \ldots, D_n$,
\begin{equation*}
    \begin{aligned}
        |w_k-\Upsilon_{k}|&=|G_{\etabf k}(\zbf)-G_{\etabf k}(\sbf)|\\
        &= |\left(\nabla G_{\etabf k}(\xibf)\right)^{\T}(\zbf-\sbf)|\\
        &= \left|\sum_{j=1}^{r_n} \frac{\partial G_{\etabf k}}{\partial \xibf_j}(z_j-s_{j})\right|\\
        &\leq \sqrt{\left\{\sum_{j=1}^{r_n}\left(\frac{\partial G_{\etabf k}}{\partial \xibf_j}\right)^2\right\}\left\{\sum_{j=1}^{r_n}(z_j-s_{j})^2\right\}}\\
        &< \sqrt{\left|\left|\frac{\partial G_{\etabf}}{\partial \zbf}\right|\right|_F^2 \left[\frac{\sqrt{\varepsilon\epsilon_n^2}}{8\sqrt{n^{1-u}}\{D_n+(p+1)\|\Upsilonbf\|_1\}}\right]^2}\\
        &=\frac{\varepsilon\epsilon_n^2}{8\{D_n+(p+1)\|\Upsilonbf\|_1\}},
    \end{aligned}
\end{equation*}
where $\xibf\in\mathscr{Z}$ and the last inequality follows by the Cauchy--Schwarz inequality.
\\Thus, for every $\zbf\in\mathcal{P}_{\varepsilon\epsilon_n^2}$, $\wbf=G_{\etabf}(\zbf)\in\mathcal{M}_{\varepsilon\epsilon_n^2}\subseteq\mathcal{N}_{\varepsilon\epsilon_n^2}$. Additionally, by \eqref{eq:prop7.12_1}, for every $\zbf\in\mathcal{P}_{\varepsilon\epsilon_n^2}$, $\zbf\in\mathcal{L}_{\varepsilon\epsilon_n^2}$. Therefore,
\begin{equation*}
     \int_{\zbf\in\mathcal{L}_{\varepsilon\epsilon_n^2}}\pi_0(\zbf)d\zbf\geq  \int_{\zbf\in\mathcal{P}_{\varepsilon\epsilon_n^2}}\pi_0(\zbf)d\zbf.
\end{equation*}
Let $\delta=\sqrt{\varepsilon\epsilon_n^2}/[8\sqrt{r_n n ^{1-u}}\{D_n+(p+1)\|\Upsilonbf\|_1\}]$, then
\begin{equation}\label{eq:prop7.12_4}
    \begin{aligned}
        \int_{\zbf\in\mathcal{P}_{\varepsilon\epsilon_n^2}}\pi_0(\zbf)d\zbf&=\prod_{j=1}^{r_n}\int_{s_{j}-\delta}^{s_{j}+\delta}\frac{1}{\sqrt{2\pi\zeta_j^2}}e^{-\frac{(z_j-\mu_j)^2}{2\zeta_j^2}}dz_j\\
        &=\prod_{j=1}^{r_n}\frac{2\delta}{\sqrt{2\pi\zeta_j^2}}e^{-\frac{(\Tilde{s}_{j}-\mu_j)^2}{2\zeta_j^2}}, \quad\Tilde{s}_{j}\in[s_{j}-\delta, s_{j}+\delta]\\
        &=\prod_{j=1}^{r_n}e^{-\left\{-\frac{1}{2}\log \left(\frac{2}{\pi}\right)-\log (\delta)+\log (\zeta_j)+\frac{(\Tilde{s}_{j}-\mu_j)^2}{2\zeta_j^2}\right\}},
    \end{aligned}
\end{equation}
where the second last equality holds by the mean value theorem.
\\Note that $\Tilde{s}_{j}\in[s_{j}-\delta, s_{j}+\delta]$, since $\delta\rightarrow 0$, therefore,
\begin{equation*}
    \frac{(\Tilde{s}_{j}-\mu_j)^2}{2\zeta_j^2}\leq\frac{\max\{(s_{j}-\mu_j-1)^2,(s_{j}-\mu_j+1)^2\}}{2\zeta_j^2}\leq\frac{(s_{j}-\mu_j)^2}{\zeta_j^2}+\frac{1}{\zeta_j^2},
\end{equation*}
where the last inequality follows since $(a + b)^2 \leq 2(a^2 + b^2)$. Therefore,
\begin{equation}\label{eq:prop7.12_5}
    \begin{aligned}
        \sum_{j=1}^{r_n}\frac{(\Tilde{s}_{j}-\mu_j)^2}{2\zeta_j^2}&\leq 2\sum_{j=1}^{r_n}\frac{s_{j}^2}{\zeta_j^2}+2\sum_{j=1}^{r_n}\frac{\mu_j^2}{\zeta_j^2}+\sum_{j=1}^{r_n}\frac{1}{\zeta_j^2}\\
        &\leq 2(\|\sbf\|_2^2+\|\mubf\|_2^2+1)\|\zetabf^*\|_{\infty}\\
        &\leq n\nu\epsilon_n^2,
    \end{aligned}
\end{equation}
since $\|\sbf\|_2^2=o(n\epsilon_n^2)$, $\|\mubf\|_2^2=o(n\epsilon_n^2)$ and $\|\zetabf^*\|_{\infty}=O(1)$ and $n\epsilon_n^2\rightarrow \infty$. Furthermore,
\begin{equation*}
    \begin{aligned}
        -\log (\delta)+\log (\zeta_j)&=\log 8+\log \{D_n+(p+1)\|\Upsilonbf\|_1\}{+\log (r_n n^{1-u})}+\log (\zeta_j)-\frac{1}{2}\log (\varepsilon\epsilon_n^2)\\
        &\leq \log 8+\log \{D_n+(p+1)\sqrt{D_n}\|\Upsilonbf\|_2\}{+\log (r_n n^{1-u})}+\log (\zeta_j)\\
        &-\frac{1}{2}\log (\varepsilon)-\log (\epsilon_n)\\
        &\leq \log 8+\log (D_n)+\log (1+\|\Upsilonbf\|_2){+\log (r_n n^{1-u})}+\log (\zeta_j)\\
        &-\frac{1}{2}\log (\varepsilon)-\log (\epsilon_n),
    \end{aligned}
\end{equation*}
where the second inequality is an outcome of the Cauchy--Schwarz inequality and the third inequality follows since $p+1\leq\sqrt{D_n}$, $n\rightarrow\infty$. Therefore,
\begin{equation}\label{eq:prop7.12_6}
    \begin{aligned}
        &\sum_{j=1}^{r_n}-\frac{1}{2}\log \left(\frac{2}{\pi}\right)-\log (\delta)+\log (\zeta_j)\\
        &\leq r_n\log 8+r_n\log (D_n)+r_n\log (1+\|\Upsilonbf\|_2){+r_n\log (r_n n^{1-u})}+r_n\log (\|\zetabf\|_{\infty})\\
        &-\frac{1}{2}r_n\log (\varepsilon)-r_n\log (\epsilon_n)\\
        &\leq n\nu\epsilon_n^2,
    \end{aligned}
\end{equation}
where the last inequality follows since $D_n\log (n)=o(n\epsilon_n^2)$, $r_n\log (n)=o(n\epsilon_n^2)$, $\|\zetabf\|_{\infty}=O(n)$, $\|\Upsilonbf\|_2=o(\sqrt{n}\epsilon_n)=o(n)$ and $1/n\epsilon_n^2=o(1)$ which implies $-2\log (\epsilon_n)=o(\log (n))$.
\\Combining \eqref{eq:prop7.12_5} and \eqref{eq:prop7.12_6} and replacing \eqref{eq:prop7.12_4}, the proof follows.
\end{proof}

\begin{proposition}\label{proposition:prop7.13}
Let $q(\zbf)\sim N(\sbf,\mathbf{I}_{r_n}/\sqrt{n})$. Define
\begin{equation*}
    h(\zbf)=\int_{\xbf\in[0,1]^p}\left[\sigma(f_0(\xbf))\{f_0(\xbf)-f_{G_{\etabf}(\zbf)}(\xbf)\}+\log\left\{\frac{1-\sigma(f_0(\xbf))}{1-\sigma(f_{G_{\etabf}(\zbf)}(\xbf))}\right\}\right]d\xbf.
\end{equation*}
Assume that the deterministic function $G_{\etabf}$ is twice differentiable at $\zbf\in\mathscr{Z}$ and the first-order derivative satisfies $\|\partial G_{\etabf}/\partial \zbf\|_F^2
=o(\varepsilon\epsilon_n^2n^{1-u})$ at $\sbf$, and the second-order derivative satisfies $\{\sum_{k=1}^{D_n}(\partial^2 G_{\etabf k}/\partial z_j^2)^2\}^{1/2}=o(\sqrt{D_n}\sum_{k=1}^{D_n}(\partial G_{\etabf k}/\partial z_j)^2)$ at $\sbf$, for $j=1,\ldots,r_n$. Let $\|f_0-f_{G_{\etabf}(\sbf)}\|_{\infty}\leq\varepsilon\epsilon_n^2/4$ where $n\epsilon_n^2\rightarrow\infty$. If $r_n\log (n)=o(n\epsilon_n^2)$, $D_n\sim n^u$, $\|\Upsilonbf\|_2^2=o(n\epsilon_n^2)$, then
\begin{equation*}
    \int h(\zbf)q(\zbf)d\zbf\leq\varepsilon\epsilon_n^2,
\end{equation*}
provided $\|\zetabf\|_{\infty}=O(n)$, $\|\zetabf^{*}\|_{\infty}=O(1)$ where $\zetabf^{*}=1/\zetabf$.
\end{proposition}

\begin{proof}
Since $h(\zbf)$ is a KL-distance, $h(\zbf)>0$. We shall thus establish an upper bound. Let $\mathcal{A}=\{\zbf:\cap_{j=1}^{r_n}|z_j-s_{j}|\leq\sqrt{\varepsilon\epsilon_n^2/r_n}\}$, then 
\begin{equation}\label{eq:prop7.13_1}
    \int h(\zbf)q(\zbf)d\zbf=\int_{\mathcal{A}} h(\zbf)q(\zbf)d\zbf+\int_{\mathcal{A}^c} h(\zbf)q(\zbf)d\zbf.
\end{equation}
By the Taylor expansion, the first term is equal to
\begin{equation}\label{eq:prop7.13_2}
    \begin{aligned}
        &=\int_{\mathcal{A}}\left\{h(\sbf)+(\zbf-\sbf)^{\T}\nabla h(\sbf)+\frac{1}{2}(\zbf-\sbf)^{\T}\nabla^2h(\sbf)(\zbf-\sbf)\right\}q(\zbf)d\zbf+o(\varepsilon\epsilon_n^2)\\
        &\leq |h(\sbf)|+\frac{1}{2}\int_{\mathcal{A}}(\zbf-\sbf)^{\T}\nabla^2h(\sbf)(\zbf-\sbf)q(\zbf)d\zbf+o(\varepsilon\epsilon_n^2)\\
        &=\frac{\varepsilon\epsilon_n^2}{2}+\frac{1}{2}\int_{\mathcal{A}}(\zbf-\sbf)^{\T}\nabla^2h(\sbf)(\zbf-\sbf)q(\zbf)d\zbf+o(\varepsilon\epsilon_n^2)\\
        &=\frac{\varepsilon\epsilon_n^2}{2}+o(\varepsilon\epsilon_n^2)\\
        &\leq \frac{3\varepsilon\epsilon_n^2}{4},
    \end{aligned}
\end{equation}
where the second step holds because $q(\zbf)$ is symmetric around $\sbf$. The third step holds in view of Lemma 7.8 in \cite{bhattacharya2020variational} and the fact that $\sbf$ satisfies $\|f_{G_{\etabf}(\sbf)}-f_0\|_{\infty}\leq\varepsilon\epsilon_n^2/4$.
\\The final step is justified next. With $J=\{1,\ldots,r_n\}$, let $\nabla^2h(\sbf)=((b_{jj^{'}}))_{j\in J, j^{'}\in J}$,
\begin{equation*}
    \int_{\mathcal{A}}(\zbf-\sbf)^{\T}\nabla^2h(\sbf)(\zbf-\sbf)q(\zbf)d\zbf=\sum_{j=1}^{r_n}b_{jj}\int_{|z_j-s_{j}|\leq\sqrt{\varepsilon\epsilon_n^2/r_n}}(z_j-s_{j})^2q(z_j)dz_j,
\end{equation*}
where the cross-covariance terms disappear since $z_j$ are independent and $q(\zbf)$ is symmetric around $\sbf$. Therefore,
\begin{equation*}
    \int_{\mathcal{A}}(\zbf-\sbf)^{\T}\nabla^2h(\sbf)(\zbf-\sbf)q(\zbf)d\zbf\leq\sum_{j=1}^{r_n}b_{jj}\int(z_j-s_{j})^2q(z_j)dz_j=\frac{1}{n}\sum_{j=1}^{r_n}|b_{jj}|.
\end{equation*}
Using Lemma \ref{lemma:lemma7.9}, we get
\begin{equation*}
    \int_{\mathcal{A}}(\zbf-\sbf)^{\T}\nabla^2h(\sbf)(\zbf-\sbf)q(\zbf)d\zbf\leq\frac{1}{n}\left\{D_n (2c^2+1) \left|\left|\frac{\partial G_{\etabf}}{\partial \zbf}\right|\right|_{F}^2\right\}=o(\varepsilon\epsilon_n^2),
\end{equation*}
where the last equality holds since $D_n\sim n^u$ and $\|\partial G_{\etabf}/\partial \zbf\|_F^2
=o(\varepsilon\epsilon_n^2n^{1-u})$.
\\We next handle the second term in \eqref{eq:prop7.13_1}. Using Lemma 7.8 in \cite{bhattacharya2020variational}, note that
\begin{equation*}
    \begin{aligned}
        \int_{\mathcal{A}^c}h(\zbf)q(\zbf)d\zbf&\leq 2\int_{\mathcal{A}^c}\left(\int_{\xbf\in[0,1]^p}|f_0(\xbf)-f_{G_{\etabf}(\zbf)}(\xbf)|d\xbf\right)q(\zbf)d\zbf\\
        &\leq2\int_{\xbf\in[0,1]^p}|f_0(\xbf)|d\xbf\int_{\mathcal{A}^c}q(\zbf)d\zbf+2\int_{\mathcal{A}^c}\int_{\xbf\in[0,1]^p}|f_{G_{\etabf}(\zbf)}(\xbf)|d\xbf q(\zbf)d\zbf.
    \end{aligned}
\end{equation*}
First, note that using $|\sigma(\cdot)|\leq 1$, we get $|f_{G_{\etabf}(\zbf)}(\xbf)|\leq\sum_{t=1}^{d}|a_t^{\Upsilonbf}|$. Thus, $|f_{G_{\etabf}(\zbf)}(\xbf)|\leq\sum_{t=1}^{d}|a_t^{\Upsilonbf}|+\sum_{t=1}^{d}|a_t-a_t^{\Upsilonbf}|$ which implies
\begin{equation}\label{eq:prop7.13_3}
    \frac{1}{2}\int_{\mathcal{A}^c}h(\zbf)q(\zbf)d\zbf\leq Q(\mathcal{A}^c)\left(\int_{\xbf\in[0,1]^p}|f_0(\xbf)d\xbf|+\sum_{t=1}^{r_n}|a_t^{\Upsilonbf}|\right)+\int_{\mathcal{A}^c}\left(\sum_{t=1}^{d}|a_t-a_t^{\Upsilonbf}|\right)q(\zbf)d\zbf.
\end{equation}
First, note that $\mathcal{A}^c=\cup_{j=1}^{r_n}\mathcal{A}_j^c$ where $\mathcal{A}_j=\{|z_j-s_{j}|\leq\sqrt{\varepsilon\epsilon_n^2/r_n}\}$. Therefore,
\begin{equation}\label{eq:prop7.13_4}
    \begin{aligned}
        Q(\mathcal{A}^c)&=Q(\cup_{j=1}^{r_n}\mathcal{A}_j^c)\leq\sum_{j=1}^{r_n}Q(\mathcal{A}_j^c)=\sum_{j=1}^{r_n}\int_{|z_j-s_{j}|>\sqrt{\varepsilon\epsilon_n^2/r_n}}q(z_j)dz_j\\
        &=2r_n\left\{1-\Phi\left(\sqrt{\frac{n\varepsilon\epsilon_n^2}{r_n}}\right)\right\}.
    \end{aligned}
\end{equation}
Using \eqref{eq:prop7.13_4} in the first term of \eqref{eq:prop7.13_3}, we get
\begin{equation}\label{eq:prop7.13_5}
    \begin{aligned}
        Q(\mathcal{A}^c)\left\{\int_{\xbf\in[0,1]^p}|f_0(\xbf)d\xbf|+\sum_{t=1}^{r_n}|a_t^{\Upsilonbf}|\right\}&\lesssim2(\|f_0\|_1+\|\Upsilonbf\|_1)r_n\left\{1-\Phi\left(\sqrt{\frac{n\varepsilon\epsilon_n^2}{r_n}}\right)\right\}\\
        &\leq 2(\|f_0\|_1+\sqrt{r_n}\|\Upsilonbf\|_2)r_n\left\{1-\Phi\left(\sqrt{\frac{n\varepsilon\epsilon_n^2}{r_n}}\right)\right\}\\
        &\leq 4n\epsilon_n^2r_n\left\{1-\Phi\left(\sqrt{\frac{n\varepsilon\epsilon_n^2}{r_n}}\right)\right\}\\
        &\leq 4nr_n\left\{1-\Phi\left(\sqrt{\frac{n\varepsilon\epsilon_n^2}{r_n}}\right)\right\}\\
        &\sim 4nr_n\sqrt{\frac{r_n}{n\varepsilon\epsilon_n^2}}e^{-\frac{n\varepsilon\epsilon_n^2}{2r_n}}\quad\text{by}\;\text{Mill's}\;\text{ratio}\\
        &\leq 4nr_ne^{-\frac{n\varepsilon\epsilon_n^2}{2r_n}}=o(n\epsilon_n^2),
    \end{aligned}
\end{equation}
where the second step follows by the Cauchy--Schwartz inequality, the third step is satisfied because $\|\Upsilonbf\|_2=o(\sqrt{n\epsilon_n^2})$ and $\sqrt{r_n}=o(\sqrt{n\epsilon_n^2})$ and $\|f_0\|_1$ are fixed and the fourth step is satisfied because $\epsilon_n^2\leq 1$. The last equality in the above steps holds because
\begin{equation}\label{eq:prop7.13_6}
    -\frac{n\epsilon_n^2}{r_n}+\log (r_n)+\log (n)-\log (\varepsilon)\leq -\frac{n\epsilon_n^2}{r_n}+3\log (n)=-\log (n)\left\{\frac{n\epsilon_n^2}{r_n\log (n)}-3\right\}\rightarrow-\infty,
\end{equation}
where the first inequality holds since $r_n\leq n$.\\
For the second term in \eqref{eq:prop7.13_3},
\begin{equation}\label{eq:prop7.13_7}
    \begin{aligned}
        \int_{\mathcal{A}_c}\left(\sum_{t=1}^{d}|a_t-a_t^{\Upsilonbf}|\right)q(\zbf)d\zbf&=\sum_{t=1}^{d}\int_{\mathcal{A}^c}|a_t-a_t^{\Upsilonbf}|q(\zbf)d\zbf\\
        &\leq\sum_{k=1}^{D_n}\int_{\mathcal{A}^c}|w_k-w_k^{\Upsilonbf}|q(\zbf)d\zbf\\
        &\leq\sum_{k=1}^{D_n}\sum_{j=1}^{r_n}\int_{\mathcal{A}^c}|z_j-s_{j}|q(\zbf)d\zbf\\
        &=\sum_{k=1}^{D_n}\sum_{j=1}^{r_n}\int_{|z_j-s_{j}|>\sqrt{\varepsilon\epsilon_n^2/r_n}}\sqrt{\frac{n}{2\pi}}|z_j-s_{j}|e^{-\frac{n}{2}|z_j-s_{j}|^2}dz_j\\
        &=\frac{2}{\sqrt{n}}\sum_{k=1}^{D_n}\sum_{j=1}^{r_n}\int_{\sqrt{\varepsilon\epsilon_n^2/r_n}}^{\infty}\frac{u}{\sqrt{2\pi}}e^{-\frac{1}{2}u^2}du\\
        &\leq 2D_nr_ne^{-\frac{n\varepsilon\epsilon_n^2}{2r_n}}=o(\varepsilon\epsilon_n^2),
    \end{aligned}
\end{equation}
where the last equality is a consequence of \eqref{eq:prop7.13_6} and $D_n\sim n^u$.
\\Combining \eqref{eq:prop7.13_3}, \eqref{eq:prop7.13_5} and \eqref{eq:prop7.13_7}, we get
\begin{equation*}
    \int_{\mathcal{A}^c}h(\zbf)q(\zbf)d\zbf=o(\varepsilon\epsilon_n^2)\leq\frac{\varepsilon\epsilon_n^2}{4}.
\end{equation*}
This together with \eqref{eq:prop7.13_2} completes the proof.
\end{proof}

\begin{proposition}\label{proposition:prop7.14}
Let $n\epsilon_n^2\rightarrow\infty$. Suppose $\pi_0(\zbf)$ satisfies \eqref{eq:prior} with $\|\mubf\|_2^2=o(n\epsilon_n^2)$ and $\|\zetabf\|_{\infty}=O(n)$. Suppose that for some $0<b<1$, $r_n\log (n)=o(n^b\epsilon_n^2)$, then for $\Tilde{C}_n=e^{n^b\epsilon_n^2/r_n}$ and $\mathcal{F}_n$ as in \eqref{eq:f_sieves}, we have for any $\varepsilon>0$,
\begin{equation*}
    \int_{\zbf\in\mathcal{F}_n^c}\pi_0(\zbf)d\zbf\leq e^{-n\varepsilon\epsilon_n^2}.
\end{equation*}
\end{proposition}

\begin{proof}
Let $\mathcal{F}_{jn}=\{z_j:|z_j|\leq\Tilde{C}_n\}$.
\begin{equation*}
    \mathcal{F}_n=\cap_{j=1}^{r_n}\mathcal{F}_{jn}\Rightarrow \mathcal{F}_n^c=\cap_{j=1}^{r_n}\mathcal{F}_{jn}^c.
\end{equation*}
Note that
\begin{equation*}
    \begin{aligned}
        \int_{\zbf\in\mathcal{F}_n^c}\pi_0(\zbf)d\zbf&\leq\sum_{j=1}^{r_n}\int_{\mathcal{F}_{jn}^c}\frac{1}{\sqrt{2\pi\zeta_j^2}}e^{-\frac{(z_j-\mu_j)^2}{2\zeta_j^2}}dz_j\\
        &=\sum_{j=1}^{r_n}\int_{-\infty}^{-\Tilde{C}_n}\frac{1}{\sqrt{2\pi\zeta_j^2}}e^{-\frac{(z_j-\mu_j)^2}{2\zeta_j^2}}dz_j+\sum_{j=1}^{r_n}\int^{\infty}_{\Tilde{C}_n}\frac{1}{\sqrt{2\pi\zeta_j^2}}e^{-\frac{(z_j-\mu_j)^2}{2\zeta_j^2}}dz_j\\
        &=\sum_{j=1}^{r_n}\left\{1-\Phi\left(\frac{\Tilde{C}_n-\mu_j}{\zeta_j}\right)\right\}+\sum_{j=1}^{r_n}\left\{1-\Phi\left(\frac{\Tilde{C}_n+\mu_j}{\zeta_j}\right)\right\}.
    \end{aligned}
\end{equation*}
Since $\|\mubf\|_2^2=o(n\epsilon_n^2)$, this implies $\|\mubf\|_{\infty}=o(\sqrt{n}\epsilon_n)$. Also, $\|\zetabf\|_{\infty}=O(n)$, which implies for some $M>0$,
\begin{equation}\label{eq:prop7.14_1}
    \min\left\{\frac{|\Tilde{C}_n-\mu_j|}{\zeta_j},\frac{|\Tilde{C}_n+\mu_j|}{\zeta_j}\right\}\geq \frac{(\Tilde{C}_n-\sqrt{n})}{nM}\geq e^{\log (\Tilde{C}_n)-2\log (n)}-\frac{1}{\sqrt{n}M}\sim e^{R_n\log (n)}\rightarrow\infty,
\end{equation}
where the last asymptotic relation holds because $1/\sqrt{n}\rightarrow 0$ and $R_n=(n^b\epsilon_n^2)/\{r_n\log (n)\}-2\rightarrow\infty$ since $r_n\log (n)=o(n^b\epsilon_n^2)$.
\\Thus, using Mill's ratio, we get
\begin{equation*}
    \begin{aligned}
        \int_{\zbf\in\mathcal{F}_n^c}\pi_0(\zbf)d\zbf&\lesssim\sum_{j=1}^{r_n}\frac{\zeta_j}{\Tilde{C}_n-\mu_j}e^{-\frac{(\Tilde{C}_n-\mu_j)^2}{2\zeta_j^2}}+\sum_{j=1}^{r_n}\frac{\zeta_j}{\Tilde{C}_n+\mu_j}e^{-\frac{(\Tilde{C}_n+\mu_j)^2}{2\zeta_j^2}}\\
        &\leq 2r_n e^{-\frac{(\Tilde{C}_n-\sqrt{n})^2}{2n^2M^2}}\lesssim e^{-\varepsilon n\epsilon_n^2},
    \end{aligned}
\end{equation*}
where the last asymptotic inequality holds because
\begin{equation*}
    \frac{(\Tilde{C}_n-\sqrt{n})^2}{2n^2M^2}-\log (2r_n) \gtrsim\frac{1}{2}e^{2R_n\log (n)}-2\log (n)=n\left\{\frac{e^{R_n}}{2}-\frac{2\log (n)}{n}\right\}\geq \varepsilon n\epsilon_n^2.
\end{equation*}
In the above step, the first asymptotic inequality holds due to \eqref{eq:prop7.14_1} and $r_n\leq n$. The last inequality holds since $R_n\rightarrow\infty$ and $\log (n)/n\rightarrow 0$.
\end{proof}

\begin{proposition}\label{proposition:prop7.15}
Let $n\epsilon_n^2\rightarrow\infty$. Suppose $r_n\log (n) = o(n^b\epsilon_n^2)$ for some $0<b<1$, $D_n\log (n) = o(n^v\epsilon_n^2)$ for some $0<v<1$ and $\pi_0(\zbf)$ satisfies \eqref{eq:prior} with $\|\mubf\|_2^2=o(n\epsilon_n^2)$. Assume that $w_k=G_{\etabf k}(\zbf)$ satisfies $|w_k|\leq C_n$, for $k=1,\ldots,D_n$. Furthermore, assume that the deterministic function $G_{\etabf}$ is differentiable at $\zbf\in\mathscr{Z}$ and the first-order derivative satisfies $\|\partial G_{\etabf}/\partial \zbf\|_F^2
=o(\varepsilon\epsilon_n^2n^{1-u})$ for $\zbf\in\mathcal{F}_n$. Then, for every $\varepsilon > 0$,
\begin{equation*}
    \log \left\{\int_{\mathcal{V}^c_{\varepsilon\epsilon_n}}\frac{L(\mathcal{D}_n; f_{G_{\etabf}(\zbf)})}{L_0}\pi_0(\zbf)d\zbf\right\}\leq \log 2-\varepsilon^2 n\epsilon_n^2+o_{P_0}(1).
\end{equation*}
\end{proposition}

\begin{proof}
In this direction, we first show
\begin{equation}\label{eq:prop7.15_1}
    P_0\left(\int_{\mathcal{V}_{\varepsilon\epsilon_n}^c} \frac{L(\mathcal{D}_n; f_{G_{\etabf}(\zbf)})}{L_0}\pi_0(\zbf)d\zbf> 2e^{-\varepsilon^2 n\epsilon_n^2}\right)\rightarrow 0, \quad n\rightarrow\infty.
\end{equation}
Note that
\begin{equation*}
    \begin{aligned}
        &P_0\left(\int_{\mathcal{V}_{\varepsilon\epsilon_n}^c} \frac{L(\mathcal{D}_n; f_{G_{\etabf}(\zbf)})}{L_0}\pi_0(\zbf)d\zbf> 2e^{-\varepsilon^2 n\epsilon_n^2}\right)\\
        &\leq P_0\left(\int_{\mathcal{V}_{\varepsilon\epsilon_n}^c\cap\mathcal{F}_n} \frac{L(\mathcal{D}_n; f_{G_{\etabf}(\zbf)})}{L_0}\pi_0(\zbf)d\zbf> e^{-\varepsilon^2 n\epsilon_n^2}\right)\\
        &+P_0\left(\int_{\mathcal{F}_n^c} \frac{L(\mathcal{D}_n; f_{G_{\etabf}(\zbf)})}{L_0}\pi_0(\zbf)d\zbf> e^{-\varepsilon^2 n\epsilon_n^2}\right).
    \end{aligned}
\end{equation*}
Using Lemma \ref{lemma:lemma7.10} with $\varepsilon=\varepsilon\epsilon_n$ and $C_n=e^{n^v\epsilon_n^2/r_n}$, $\Tilde{C}_n=e^{n^b\epsilon_n^2/r_n}$,
\begin{equation*}
    \begin{aligned}
        &\int_{\varepsilon^2/8}^{\sqrt{2}\varepsilon}\sqrt{H_(u, \Tilde{\mathcal{F}}_n, \|\cdot\|_2)}du\\
        &\lesssim\varepsilon\epsilon_n\sqrt{2r_n\{\log (r_n)+\log (D_n)+\log (C_n)+\log (\Tilde{C}_n)-\log (\epsilon_n)\}}\\
        &\leq \varepsilon\epsilon_n O(\max\{\sqrt{r_n\log (r_n)}, \sqrt{r_n\log (D_n)}, \sqrt{r_n\log (C_n)}, \sqrt{r_n\log (\Tilde{C}_n}),\sqrt{-\log (\epsilon_n})\})\\
        &\leq \varepsilon\epsilon_n\max\{o(\sqrt{n\epsilon_n}), o(\sqrt{n\epsilon_n}),O(\sqrt{n^v}\epsilon_n), O(\sqrt{n^b}\epsilon_n), O(\sqrt{\log (n)})\}\\
        &\leq \varepsilon^2\epsilon_n^2\sqrt{n}.
    \end{aligned}
\end{equation*}
The first inequality in the third step follows because $D_n\leq n$, $D_n\log (n)=o(n\epsilon_n)$ and $r_n\log (n)=o(n\epsilon_n)$, $r_n\log (C_n)=r_n(n^v\epsilon_n^2/r_n)$ and $r_n\log (\Tilde{C}_n)=r_n(n^b\epsilon_n^2/r_n)$, $1/\epsilon_n^2=o(n)$, then $-\log (\epsilon_n^2)\leq\log (n)$. The second inequality in the third step follows since $n^v/n=o(1)$, $n^b/n=o(1)$ and $\log (n)=o(n\epsilon_n^2)$.
\\By Theorem 1 in \cite{wong1995probability}, for some constant $C>0$, we have
\begin{equation}\label{eq:prop7.15_2}
    \begin{aligned}
         P_0\left(\int_{\mathcal{V}_{\varepsilon\epsilon_n}^c\cap\mathcal{F}_n} \frac{L(\mathcal{D}_n; f_{G_{\etabf}(\zbf)})}{L_0}\pi_0(\zbf)d\zbf> e^{-\varepsilon^2 n\epsilon_n^2}\right)
         &\leq P_0\left(\sup_{\zbf\in\mathcal{V}_{\varepsilon\epsilon_n}^c\cap\mathcal{F}_n}\frac{L(\mathcal{D}_n; f_{G_{\etabf}(\zbf)})}{L_0}>e^{-\varepsilon^2n\epsilon_n^2}\right)\\
         &\leq 4\exp(-C\varepsilon^2n\epsilon_n^2)\rightarrow 0.
    \end{aligned}
\end{equation}
Using Proposition \ref{proposition:prop7.14} with $\varepsilon=2\varepsilon$, we have 
\begin{equation*}
    \int_{\zbf\in\mathcal{F}_n^c}\pi_0(\zbf)d\zbf\leq e^{-2n\varepsilon^2\epsilon_n^2}.
\end{equation*}
Therefore, using Lemma 7.6 in \cite{bhattacharya2020variational} with $\varepsilon=2\varepsilon^2\epsilon_n^2$ and $\Tilde{\varepsilon}=\varepsilon^2\epsilon_n^2$, we have
\begin{equation}\label{eq:prop7.15_3}
    P_0\left(\int_{\mathcal{F}_n^c} \frac{L(\mathcal{D}_n; f_{G_{\etabf}(\zbf)})}{L_0}\pi_0(\zbf)d\zbf> e^{-\varepsilon^2 n\epsilon_n^2}\right)\leq e^{-\varepsilon^2n\epsilon_n^2} \rightarrow0.
\end{equation}
Combining \eqref{eq:prop7.15_2} and \eqref{eq:prop7.15_3}, \eqref{eq:prop7.15_1} follows.
\\Finally, to complete the proof, let $(\mathrm{I})=\log[\int_{\mathcal{V}^c_{\varepsilon\epsilon_n}}\{L(\mathcal{D}_n\mid G_{\etabf}(\zbf))/L_0\}\pi_0(\zbf)d\zbf]$.
\begin{equation*}
    \begin{aligned}
         (\mathrm{I}) &= (\mathrm{I})\times\mathds{1}((\mathrm{I})\leq \log 2-\varepsilon^2\epsilon_n^2)+(\mathrm{I})\times\mathds{1}{((\mathrm{I})\geq \log 2-\varepsilon^2\epsilon_n^2)}\\
         &\leq \log 2-\varepsilon^2\epsilon_n^2 +\underbrace{(\mathrm{I})*\mathds{1}((\mathrm{I})> \log 2-\varepsilon^2\epsilon_n^2)}_{(\mathrm{II})}\\
         &=\log 2-\varepsilon^2\epsilon_n^2 + o_{P_0}(1),
    \end{aligned}
\end{equation*}
where the last equality follows from \eqref{eq:prop7.15_1} as below
\begin{equation*}
    P_0(|(\mathrm{II})|>\nu)\leq P_0(\mathds{1}((\mathrm{I})> \log 2-\varepsilon^2\epsilon_n^2)=1)=P_0((\mathrm{I})>\log 2-\varepsilon^2\epsilon_n^2)\rightarrow 0.
\end{equation*}
\end{proof}

\begin{proposition}\label{proposition:prop7.16}
Let $\pi_0(\zbf)$ satisfy \eqref{eq:prior} with $\|\zetabf\|_{\infty}=O(n)$ and $\|\zetabf^{*}\|_{\infty}=O(1)$, $\zetabf^{*}=1/\zetabf$. Assume that the deterministic function $G_{\etabf}$ is twice differentiable at $\zbf\in\mathscr{Z}$ and the first-order derivative satisfies $\|\partial G_{\etabf}/\partial \zbf\|_F^2=o(\varepsilon\epsilon_n^2n^{1-u})$ at $\sbf$, and the second-order derivative satisfies $$\{\sum_{k=1}^{D_n}(\partial^2 G_{\etabf k}/\partial z_j^2)^2\}^{1/2}=o(\sqrt{D_n}\sum_{k=1}^{D_n}(\partial G_{\etabf k}/\partial z_j)^2)$$ at $\sbf$, for $j=1,\ldots,r_n$.
\begin{itemize}
    \item[1.] 
    If $r_n\log (n)=o(n)$, $D_n\log (n)=o(n)$ and $\|\mubf\|^2_2=o(n)$, then
    \begin{equation*}
        d_{\mathrm{KL}}(q^{*}, p_{\etabf}(\cdot\mid\mathcal{D}_n))=o_{P_0}(n).
    \end{equation*}
    \item[2.]
    If $r_n\log (n)=o(n\epsilon_n^2)$, $D_n\log (n)=o(n\epsilon_n^2)$ and $\|\mubf\|^2_2=o(n\epsilon_n^2)$ and there exists a neural network such that $\|f_0-f_{\Upsilonbf}\|_{\infty}=o(\epsilon_n^2)$, $\|\Upsilonbf\|_2^2=o(n\epsilon_n^2)$, $\|\sbf\|_2^2=o(n\epsilon_n^2)$, then
    \begin{equation*}
        d_{\mathrm{KL}}(q^{*}, p_{\etabf}(\cdot\mid\mathcal{D}_n))=o_{P_0}(n\epsilon_n^2).
    \end{equation*}
\end{itemize}
\end{proposition}

\begin{proof}
For any $q\in\mathcal{Q}$,
\begin{equation}\label{eq:prop7.16_1}
    \begin{aligned}
         d_{\mathrm{KL}}(q, p_{\etabf}(\cdot\mid\mathcal{D}_n))&=\int q(\zbf)\log \{q(\zbf)\}d\zbf-\int q(\zbf)\log \{p_{\etabf}(\zbf\mid\mathcal{D}_n)\}d\zbf\\
         &=\int q(\zbf)\log \{q(\zbf)\}d\zbf-\int q(\zbf)\log\left\{ \frac{L(\mathcal{D}_n; f_{G_{\etabf}(\zbf)})\pi_0(\zbf)}{\int L(\mathcal{D}_n; f_{G_{\etabf}(\zbf)})\pi_0(\zbf)d\zbf}\right\}d\zbf\\
         &=d_{\mathrm{KL}}(q,\pi_0)-\int\log \left\{\frac{L(\mathcal{D}_n; f_{G_{\etabf}(\zbf)})}{L_0}\right\}q(\zbf)d\zbf\\&+\log\left\{\int \frac{L(\mathcal{D}_n; f_{G_{\etabf}(\zbf)})}{L_0}\pi_0(\zbf)d\zbf\right\}\\
         & \leq d_{\mathrm{KL}}(q,\pi_0)+\left|\int\log \left\{\frac{L(\mathcal{D}_n; f_{G_{\etabf}(\zbf)})}{L_0}\right\}q(\zbf)d\zbf\right|\\
         &+\left|\log \left\{\int \frac{L(\mathcal{D}_n; f_{G_{\etabf}(\zbf)})}{L_0}\pi_0(\zbf)d\zbf\right\}\right|.
    \end{aligned}
\end{equation}
Since $q^{*}$ satisfies the minimized KL-distance to $p_{\etabf}(\cdot\mid\mathcal{D}_n)$ in the family $\mathcal{Q}$, therefore,
\begin{equation}\label{eq:prop7.16_2}
    P_0(d_{\mathrm{KL}}(q^{*}, p_{\etabf}(\cdot\mid\mathcal{D}_n))>\kappa)\leq P_0(d_{\mathrm{KL}}(q, p_{\etabf}(\cdot\mid\mathcal{D}_n))>\kappa),
\end{equation}
for any $\kappa>0$.
\\\textit{Proof of part 1}. Note that $r_n\log (n)=o(n)$, $D_n\log (n)=o(n)$, $\|\mubf\|^2_2=o(n)$, $\|\zetabf\|_{\infty}=O(n)$ and $\|\zetabf^{*}\|_{\infty}=O(1)$. We take $q(\zbf)=N(\sbf, \mathbf{I}_{r_n}/\sqrt{n})$ where $\sbf$ is defined next.
\\For $N\geq 1$, let $f_{\Upsilonbf_N}$ be a neural network that satisfies $\|f_{\Upsilonbf_N}-f_0\|_{\infty}\leq\varepsilon/4$. The existence of such a neural network is always guaranteed by \cite{hornik1989multilayer}. Define $\Upsilonbf$ as
\begin{equation*}
    a_j^{\Upsilonbf}
    =\left\{
    \begin{array}{ll}
        a_j^{\Upsilonbf_N}, &\quad j=1,\ldots,d_N\\
        0, &\quad j=d_N+1,\ldots,d
    \end{array}
    \right.
    \qquad 
    \omegabf_j^{\Upsilonbf}
    =\left\{
    \begin{array}{ll}
        \omegabf_j^{\Upsilonbf_N}, &\quad j=1,\ldots,d_N\\
        0, &\quad j=d_N+1,\ldots,d,
    \end{array}
    \right.
\end{equation*}
and let $d_N=r_n/2$, we could rewrite $\Upsilonbf$ as 
\begin{equation*}
    \Upsilonbf_{k}
    =\left\{
    \begin{array}{ll}
        (G_{\etabf}(\sbf))_k, &\quad k=1,\ldots,r_n\\
        0, &\quad k=r_n+1,\ldots,D_n,
    \end{array}
    \right.
\end{equation*}
where $(G_{\etabf}(\sbf))_k$ is the $k$th element of the vector $\Upsilonbf=G_{\etabf}(\sbf)$. This choice guarantees $\|f_{G_{\etabf}(\sbf)}-f_0\|_{\infty}=\|f_{\Upsilonbf}-f_0\|_{\infty}\leq \varepsilon/4$.
\\\textbf{Step 1 (a)}: Since $\|\Upsilonbf\|_2^2=\|\Upsilonbf_N\|_2^2$ is bounded, which implies $\|\Upsilonbf\|_2^2=o(n)$.
\\By the expression of $\Upsilonbf=G_{\etabf}(\sbf)$, $G_{\etabf}(\cdot)$ is invertible at $\sbf$ and thus $\sbf=h(\Upsilonbf)$ where $G_{\etabf}(h(\Upsilonbf))=\Upsilonbf$.
\\Denote $\Tilde{h}(\cdot)=h(\cdot)-h(\0bf)$. By the Taylor expansion,
\begin{equation*}
\begin{aligned}
    \sbf&\approx h(\0bf)+(\Upsilonbf-\0bf)^{\T}\nabla h(\0bf)\\
    &=h(\0bf)+\Upsilonbf^{\T}\{\nabla G_{\etabf}(\sbf_{0})\}^{-1},
\end{aligned}
\end{equation*}
where $\sbf_{0}=h(\0bf)$ the last equation follows by the inverse function theorem.
\\Therefore,
\begin{equation*}
    \|\sbf\|_2^2\leq \|\sbf_{0}\|_2^2+\|\Upsilonbf\|_2^2\|\{\nabla G_{\etabf}(\sbf_{0})\}^{-1}\|_F^2,
\end{equation*}
which implies $\|\sbf\|_2^2=o(n)$ since $\|\Upsilonbf\|_2^2=o(n)$ and $\|\{\nabla G_{\etabf}(\sbf_{0})\}^{-1}\|_F^2=O(1)$.
\\Using Proposition \ref{proposition:prop7.11}, with $\epsilon_n=1$, we get for any $\nu>0$,
\begin{equation*}
    d_{\mathrm{KL}}(q, \pi_0)\leq n\nu,
\end{equation*}
where the above step follows by $\|\sbf\|_2^2=o(n)$. Therefore,
\begin{equation}\label{eq:prop7.16_3}
    P_0(d_{\mathrm{KL}}(q, \pi_0)> n\nu)=0.
\end{equation}
\textbf{Step 1 (b)}: Next, note that 
\begin{equation*}
    \begin{aligned}
        d_{\mathrm{KL}}(\ell_0,\ell_{G_{\etabf}(\zbf)})&=\int _{\xbf\in[0,1]^p}\sigma(f_0(\xbf))\log \left\{\frac{\sigma(f_0(\xbf))}{\sigma(f_{G_{\etabf}(\zbf)}(\xbf))}\right\} d\xbf\\
        &+\int _{\xbf\in[0,1]^p}\{1-\sigma(f_0(\xbf))\}\log \left\{\frac{1-\sigma(f_0(\xbf))}{1-\sigma(f_{G_{\etabf}(\zbf)}(\xbf))}\right\}d\xbf\\
        &=\int _{\xbf\in[0,1]^p}\left[\sigma(f_0(\xbf))\{f_0(\xbf)-f_{G_{\etabf}(\zbf)}(\xbf)\}+\log \left\{\frac{1-\sigma(f_0(\xbf))}{1-\sigma(f_{G_{\etabf}(\zbf)}(\xbf))}\right\}\right]d\xbf.
    \end{aligned}
\end{equation*}
Since $\|f_{\Upsilonbf}-f_0\|_{\infty}\leq \varepsilon/4$, using Proposition \ref{proposition:prop7.13} with $\epsilon_n=1$ and $\varepsilon=\varepsilon$,
\begin{equation*}
    \int d_{\mathrm{KL}}(\ell_0,\ell_{G_{\etabf}(\zbf)}) q(\zbf)d\zbf \leq \varepsilon,
\end{equation*}
where the above step follows since $\|\Upsilonbf\|_2^2=\|\Upsilonbf_N\|^2_2$ is bounded, which implies $\|\Upsilonbf\|_2^2=o(n)$. 
\\Therefore, by Lemma 7.4 in \cite{bhattacharya2020variational},
\begin{equation}\label{eq:prop7.16_4}
    P_0\left(\left|\int\log \left\{\frac{L(\mathcal{D}_n; f_{G_{\etabf}(\zbf)})}{L_0}\right\}q(\zbf)d\zbf\right|>n\nu\right)\leq\frac{\varepsilon}{\nu}.
\end{equation}
\textbf{Step 1 (c)}: Since $\|f_{\Upsilonbf}-f_0\|_{\infty}\leq \varepsilon/4$, therefore, using Proposition \ref{proposition:prop7.12} with $\epsilon_n=1$ and $\nu=\varepsilon$, we get
\begin{equation*}
    \int_{\zbf\in\mathcal{L}_{\varepsilon}}\pi_0(\zbf)d\zbf\geq \exp(-n\varepsilon),
\end{equation*}
where the above step follows $\|\Upsilonbf\|_2^2=\|\Upsilonbf_N\|^2_2$ is bounded which implies $\|\Upsilonbf\|_2^2=o(n)$ and $\|\sbf\|_2^2=o(n)$.
\\Therefore, using Lemma 7.5 in \cite{bhattacharya2020variational}, we get
\begin{equation}\label{eq:prop7.16_5}
    P_0\left(\left|\log \left\{\int \frac{L(\mathcal{D}_n; f_{G_{\etabf}(\zbf)})}{L_0}\pi_0(\zbf)d\zbf\right\}\right|>n\nu\right)\leq \frac{2\varepsilon}{\nu}.
\end{equation}
\textbf{Step 1 (d)}: From \eqref{eq:prop7.16_1} and \eqref{eq:prop7.16_2}, we get
\begin{equation*}
    \begin{aligned}
        P_0(d_{\mathrm{KL}}(q^{*},p_{\etabf}(\cdot\mid\mathcal{D}_n))>3n\nu)&\leq P_0(d_{\mathrm{KL}}(q,p_{\etabf}(\cdot\mid\mathcal{D}_n))>n\nu)\\
        &+P_0\left(\left|\int\log \left\{\frac{L(\mathcal{D}_n; f_{G_{\etabf}(\zbf)})}{L_0}\right\}q(\zbf)d\zbf\right|>n\nu\right)\\
        &+P_0\left(\left|\log\left\{ \int \frac{L(\mathcal{D}_n; f_{G_{\etabf}(\zbf)})}{L_0}\pi_0(\zbf)d\zbf\right\}\right|>n\nu\right)\\
        &\leq \frac{2\varepsilon}{\nu},
    \end{aligned}
\end{equation*}
where the last inequality is a consequence of \eqref{eq:prop7.16_3}, \eqref{eq:prop7.16_4} and \eqref{eq:prop7.16_5}.
Since $\varepsilon$ is arbitrary, taking $\varepsilon\rightarrow0$ completes the proof.
\\\textit{Proof of part 2}. Note that $r_n\log (n)=o(n\epsilon_n^2)$, $D_n\log (n)=o(n\epsilon_n^2)$, $\|\mubf\|^2_2=o(n\epsilon_n^2)$, $\|\zetabf\|_{\infty}=O(n)$ and $\|\zetabf^{*}\|_{\infty}=O(1)$. We take $q(\zbf)=N(\sbf, \mathbf{I}_{r_n}/\sqrt{n})$ where $\sbf$ is defined next.
\\For $N\geq 1$, let $f_{\Upsilonbf_N}$ be a neural network that satisfies $\|f_{\Upsilonbf_N}-f_0\|_{\infty}\leq\varepsilon\epsilon_n^2/4$. The existence of such a neural network is always guaranteed by \cite{hornik1989multilayer}. Define $\Upsilonbf$ as
\begin{equation*}
    a_j^{\Upsilonbf}
    =\left\{
    \begin{array}{ll}
        a_j^{\Upsilonbf_N}, &\quad j=1,\ldots,d_N\\
        0, &\quad j=d_N+1,\ldots,d
    \end{array}
    \right.
    \qquad 
    \omegabf_j^{\Upsilonbf}
    =\left\{
    \begin{array}{ll}
        \omegabf_j^{\Upsilonbf_N}, &\quad j=1,\ldots,d_N\\
        0, &\quad j=d_N+1,\ldots,d,
    \end{array}
    \right.
\end{equation*}
and let $d_N=r_n/2$, we could rewrite $\Upsilonbf$ as 
\begin{equation*}
    \Upsilonbf_{k}
    =\left\{
    \begin{array}{ll}
        (G_{\etabf}(\sbf))_k, &\quad k=1,\ldots,r_n\\
        0, &\quad k=r_n+1,\ldots,D_n,
    \end{array}
    \right.
\end{equation*}
where $(G_{\etabf}(\sbf))_k$ is the $k$th element of the vector $\Upsilonbf=G_{\etabf}(\sbf)$. This choice guarantees $\|f_{G_{\etabf}(\sbf)}-f_0\|_{\infty}=\|f_{\Upsilonbf}-f_0\|_{\infty}\leq \varepsilon\epsilon_n^2/4$.
\\\textbf{Step 2 (a)}: Since $\|\Upsilonbf\|_2^2=\|\Upsilonbf_N\|_2^2$ is bounded, this implies $\|\Upsilonbf\|_2^2=o(n\epsilon_n^2)$.
\\By the expression of $\Upsilonbf=G_{\etabf}(\sbf)$, $G_{\etabf}(\cdot)$ is invertible at $\sbf$ and thus $\sbf_n=h(\Upsilonbf)$ where $G_{\etabf}(h(\Upsilonbf))=\Upsilonbf$.
\\Denote $\Tilde{h}(\cdot)=h(\cdot)-h(\0bf)$. By the Taylor expansion,
\begin{equation*}
\begin{aligned}
    \sbf&\approx h(\0bf)+(\Upsilonbf-\0bf)^{\T}\nabla h(\0bf)\\
    &=h(\0bf)+\Upsilonbf^{\T}\{\nabla G_{\etabf}(\sbf_{0})\}^{-1},
\end{aligned}
\end{equation*}
where $\sbf_{0}=h(\0bf)$ the last equation follows by the inverse function theorem.
\\Therefore,
\begin{equation*}
    \|\sbf\|_2^2\leq \|\sbf_{0}\|_2^2+\|\Upsilonbf\|_2^2\|\{\nabla G_{\etabf}(\sbf_{0})\}^{-1}\|_F^2,
\end{equation*}
which implies $\|\sbf\|_2^2=o(n\epsilon_n^2)$ since $\|\Upsilonbf\|_2^2=o(n\epsilon_n^2)$ and $\|\{\nabla G_{\etabf}(\sbf_{0})\}^{-1}\|_F^2=O(1)$.
\\Using Proposition \ref{proposition:prop7.11}, we get for any $\nu>0$,
\begin{equation*}
    d_{\mathrm{KL}}(q, \pi_0)\leq n\epsilon_n^2\nu,
\end{equation*}
where the above step follows by $\|\sbf\|_2^2=o(n\epsilon_n^2)$. Therefore,
\begin{equation}\label{eq:prop7.16_6}
    P_0(d_{\mathrm{KL}}(q, \pi_0)> n\epsilon_n^2\nu)=0.
\end{equation}
\textbf{Step 2 (b)}:
Since $\|f_{\Upsilonbf}-f_0\|_{\infty}\leq \varepsilon\epsilon_n^2/4$ and $\|\Upsilonbf\|_2^2=o(n\epsilon_n^2)$, by Proposition \ref{proposition:prop7.13},
\begin{equation*}
    \int d_{\mathrm{KL}}(\ell_0,\ell_{G_{\etabf}(\zbf)}) q(\zbf)d\zbf \leq \varepsilon\epsilon_n^2,
\end{equation*}
Therefore, by Lemma 7.4 in \cite{bhattacharya2020variational},
\begin{equation}\label{eq:prop7.16_7}
    P_0\left(\left|\int\log\left\{ \frac{L(\mathcal{D}_n; f_{G_{\etabf}(\zbf)})}{L_0}\right\}q(\zbf)d\zbf\right|>n\epsilon_n^2\nu\right)\leq\frac{\varepsilon}{\nu}.
\end{equation}
\textbf{Step 2 (c)}: Since $\|f_{\Upsilonbf}-f_0\|_{\infty}\leq \varepsilon\epsilon_n^2/4$, $\|\Upsilonbf\|_2^2=o(n\epsilon_n^2)$ and $\|\sbf\|_2^2=o(n\epsilon_n^2)$, by Proposition \ref{proposition:prop7.12},
\begin{equation*}
    \int_{\zbf\in\mathcal{L}_{\varepsilon}}\pi_0(\zbf)d\zbf\geq \exp(-\varepsilon n\epsilon_n^2),
\end{equation*}
Therefore, using Lemma 7.5 in \cite{bhattacharya2020variational}, we get
\begin{equation}\label{eq:prop7.16_8}
    P_0\left(\left|\log\left\{ \int \frac{L(\mathcal{D}_n; f_{G_{\etabf}(\zbf)})}{L_0}\pi_0(\zbf)d\zbf\right\}\right|>n\epsilon_n^2\nu\right)\leq \frac{2\varepsilon}{\nu}.
\end{equation}
\textbf{Step 2 (d)}: From \eqref{eq:prop7.16_1} and \eqref{eq:prop7.16_2}, we get
\begin{equation*}
    \begin{aligned}
        P_0(d_{\mathrm{KL}}(q^{*},p_{\etabf}(\cdot\mid\mathcal{D}_n))>3n\epsilon_n^2\nu)&\leq P_0(d_{\mathrm{KL}}(q,p_{\etabf}(\cdot\mid\mathcal{D}_n))>n\epsilon_n^2\nu)\\
        &+P_0\left(\left|\int\log \left\{\frac{L(\mathcal{D}_n; f_{G_{\etabf}(\zbf)})}{L_0}\right\}q(\zbf)d\zbf\right|>n\epsilon_n^2\nu\right)\\
        &+P_0\left(\left|\log\left\{ \int \frac{L(\mathcal{D}_n; f_{G_{\etabf}(\zbf)})}{L_0}\pi_0(\zbf)d\zbf\right\}\right|>n\epsilon_n^2\nu\right)\\
        &\leq \frac{2\varepsilon}{\nu},
    \end{aligned}
\end{equation*}
where the last inequality is a consequence of \eqref{eq:prop7.16_6}, \eqref{eq:prop7.16_7} and \eqref{eq:prop7.16_8}.
\\Since $\varepsilon$ is arbitrary, taking $\varepsilon\rightarrow0$ completes the proof.
\end{proof}

\bibliographystyle{apalike}
\bibliography{ref}
\end{document}